\documentclass[10pt]{article}
\usepackage{amsmath,amssymb,bm}
\usepackage{amsthm}
\usepackage{mathtools}
\usepackage[noend]{algorithmic}
\usepackage[ruled,vlined]{algorithm2e}
\usepackage{url}
\usepackage{makeidx}
\usepackage{booktabs,multirow,threeparttable}

\usepackage{graphicx,float,psfrag,epsfig,caption}
\usepackage[usenames,dvipsnames,svgnames,table]{xcolor}
\definecolor{darkgreen}{rgb}{0.0,0,0.9}
\usepackage[pagebackref,letterpaper=true,colorlinks=true,pdfpagemode=none,citecolor=OliveGreen,linkcolor=BrickRed,urlcolor=BrickRed,pdfstartview=FitH]{hyperref}
\usepackage{epstopdf}
\usepackage{color}
\usepackage{xr}
\usepackage{subfigure}
\usepackage{caption}
\usepackage{graphicx}
\usepackage[utf8]{inputenc}
\usepackage{tikz}
\usepackage[english]{babel}
\usepackage{am}

\usepackage{scalerel,stackengine}

\usepackage[mathscr]{euscript}

\DeclareSymbolFont{rsfs}{U}{rsfs}{m}{n}
\DeclareSymbolFontAlphabet{\mathscrsfs}{rsfs}

\stackMath
\newcommand\reallywidehat[1]{%
\savestack{\tmpbox}{\stretchto{%
  \scaleto{%
    \scalerel*[\widthof{\ensuremath{#1}}]{\kern.1pt\mathchar"0362\kern.1pt}%
    {\rule{0ex}{\textheight}}
  }{\textheight}%
}{2.4ex}}%
\stackon[-6.9pt]{#1}{\tmpbox}%
}
\usepackage[mathscr]{euscript}

\DeclareSymbolFont{rsfs}{U}{rsfs}{m}{n}
\DeclareSymbolFontAlphabet{\mathscrsfs}{rsfs}

\numberwithin{equation}{section}

\newtheoremstyle{myexample} 
    {\topsep}                    
    {\topsep}                    
    {\rm }                   
    {}                           
    {\bf }                   
    {.}                          
    {.5em}                       
    {}  

\makeatletter

\newcommand*{\rom}[1]{\expandafter\@slowromancap\romannumeral #1@}
\makeatother

\usetikzlibrary{arrows,shapes}
\begin{document}

\title{\bf Kernel regression in high dimensions: Refined analysis beyond double descent}

\author{Fanghui Liu\thanks{Department of Electrical Engineering
		(ESAT-STADIUS), KU Leuven} \;\;\;\;  Zhenyu Liao\thanks{Department of Statistics, University of California, Berkeley} \;\;\;\; 
Johan A.K. Suykens\footnotemark[1]}

\maketitle

\begin{abstract}
	\noindent In this paper, we provide a precise characterization of generalization properties of high dimensional kernel ridge regression across the under- and over-parameterized regimes, depending on whether the number of training data $n$ exceeds the feature dimension $d$.
	By establishing a bias-variance decomposition of the expected excess risk, we show that, while the bias is (almost) independent of $d$ and monotonically decreases with $n$, the variance depends on $n,d$ and can be unimodal or monotonically decreasing under different regularization schemes.
	Our refined analysis goes beyond the double descent theory by showing that, depending on the data eigen-profile and the level of regularization, the kernel regression risk curve can be a double-descent-like, bell-shaped, or monotonic function of $n$. Experiments on synthetic and real data are conducted to support our theoretical findings.
\end{abstract}


\section{Introduction}

Interpolation learning \cite{mei2019generalization,hastie2019surprises,bartlett2020benign} has recently attracted growing attention in the machine learning community. 
This is mainly because current state-of-the-art neural networks appear to be models of this type: they are able to interpolate the training data while still generalize well on test data, even in the presence of label noise \cite{zhang2016understanding}.
It has been empirically observed that other models including random features, decision trees, and as simple as linear regression also exhibit similar phenomenon \cite{bartlett2020benign,belkin2019reconciling,liu2020survey}.
This is somewhat striking as it goes against the conventional wisdom of \emph{bias-variance trade-off}\cite{cucker2002mathematical}: predictors that generalize well must trade off the model complexity against training data fitting.
The double descent theory \cite{belkin2019reconciling} resolves this paradox by revisiting the bias-variance trade-off and showing that the model generalization error exhibits a phase transition at the \emph{interpolation point}: moving away from this point on either side tends to reduce the generalization error.

The double descent phenomenon has recently inspired intense theoretical research \cite{mei2019generalization,gerace2020generalisation,wu2020optimal,liao2020random} and has been further extended to multiple descent \cite{chen2020multiple,liang2020multiple,adlam2020neural} on various models. 
One line of work formalized the argument that, even when no explicit regularization is imposed, \emph{implicit regularization} is encoded in the model via the choice of optimization algorithms and techniques, e.g., stochastic gradient descent (SGD) \cite{hardt2016train}, dropout \cite{srivastava2014dropout}, early stopping \cite{ali2019continuous}, and ensemble methods \cite{lejeune2020implicit}.
Different from these ``external'' schemes, the kernel interpolation estimator \cite{liang2020just,belkin2019does} directly benefits from its intrinsic kernel structure that serves as an \emph{implicit regularization} to help both interpolate and approximate. 
In fact, (strictly) positive-definite kernels can interpolate an arbitrary number of data points \cite{wendland2004scattered}, and thus kernel spaces contain (nearly) optimal interpolants \cite{ghorbani2019linearized,li2020generalization}.
Although the kernel space is rich enough to contain models that generalize well, the generalization property of kernel method, for example how it depends on the choice of kernel, its interplay with the data and the level of regularization, still remains unclear. In particular, the question whether the double descent phenomenon exists in the kernel regression models is still unanswered \cite{liang2020just,bordelon2020spectrum}.
As such, refined analyses are needed to have a thorough understanding of kernel estimators, notably in the high dimensional regime of interest. This is indeed the objective of the article.

\begin{table*}[t]
	\centering
	\caption{Trends of the variance ${\tt{V}}$ with respect to $n$ in the $n<d$ case. The notation $\nearrow$ means ${\tt{V}}$ increases with $n$; $\rightarrow$ for ${\tt{V}}$ stays unchanged; and $\searrow$ for ${\tt{V}}$ decreases with $n$, see Figure~\ref{tableten1}; and $r_* := \operatorname{rank}(\bm X \bm X^{\!\top})$. From left to the right column, the regularization $\lambda$ increases, and a large  $\lambda$ leads to a small value of peak point $n_* := n_*(\lambda)$, which may even disappear. Note that $n_*$ is different for three eigenvalue decays of $\bm X \bm X^{\!\top}/d$. See Section~\ref{sec:nleqd} for details. }\label{tablevar}
	\smallskip
	\begin{threeparttable}
		\begin{tabular}{c|c|c|c|c|c|ccccc}
			\toprule[1.5pt]
			eigenvalue decay & $\lambda = 0$                                      & \multicolumn{4}{c}{$\lambda := \bar{c} n^{-\vartheta}$ (KRR)}                         \\ \midrule[1pt]
			\multirow{3}{*}{\emph{harmonic decay}} &  \multicolumn{1}{c|}{\multirow{3}{*}{ $\nearrow$ $\rightarrow$}}                                  & $1 \geq \vartheta \geq \frac{1}{2(2-\bar{c})}$            & \multicolumn{4}{c}{$\vartheta < \frac{1}{2(2-\bar{c})}$} \\ \cline{3-7} 
			&  & \multirow{2}{*}{$\nearrow$ $\rightarrow$} & $r_*<d \leq n_*$     & $r_* \leq n_* \leq d$     & $n_* \leq r_* < d$   & $n_* \leq c < r_* < d$~\tnote{1}    \\ \cline{4-7} 
			& \multicolumn{1}{c|}{}                   &                    & $\nearrow$ $\rightarrow$      & $\nearrow$ $\rightarrow$       & $\nearrow$ $\searrow$ $\rightarrow$  & $\searrow$ $\rightarrow$    \\ \midrule[1pt]
			
			\multirow{3}{*}{\emph{polynomial decay}} & \multicolumn{1}{c|}{\multirow{3}{*}{ $\nearrow$ $\rightarrow$}}                                   & $1 \geq \vartheta \geq \frac{1}{1+\frac{1}{2a}}$            & \multicolumn{4}{c}{$\vartheta < \frac{1}{1+\frac{1}{2a}}$} \\ \cline{3-7} 
			&  & \multirow{2}{*}{$\nearrow$ $\rightarrow$} & $r_*<d \leq n_*$     & $r_* \leq n_* \leq d$     & $n_* \leq r_* < d$   & $n_* \leq c < r_* < d$   \\ \cline{4-7} 
			& \multicolumn{1}{c|}{}                   &                    & $\nearrow$ $\rightarrow$      & $\nearrow$ $\rightarrow$       & $\nearrow$ $\searrow$ $\rightarrow$  & $\searrow$ $\rightarrow$   \\ \midrule[1pt]   
			
			\multirow{2}{*}{\emph{exponential decay}} & \multicolumn{1}{c|}{\multirow{2}{*}{$\nearrow$ $\rightarrow$}} & & $r_*<d \leq n_*$     & $r_* \leq n_* \leq d$     & $n_* \leq r_* < d$   & $n_* \leq c < r_* < d$  \\  \cline{4-7} 
			& \multicolumn{1}{c|}{}     &                          & $\nearrow$ $\rightarrow$      & $\nearrow$ $\rightarrow$     & $\nearrow$ $\searrow$ $\rightarrow$  & $\searrow$  $\rightarrow$   \\ \bottomrule[1.5pt]      
		\end{tabular}
		\begin{tablenotes}
			\footnotesize
			\item[1] Here $c$ is some constant such that $n > c$ always holds as $n$ is required to be large in theory and practice.  
		\end{tablenotes}
	\end{threeparttable}
\end{table*}

\begin{figure*}[!htb]
	\centering
	\subfigure[trends of variance]{\label{tableten1}
		\includegraphics[width=0.23\textwidth]{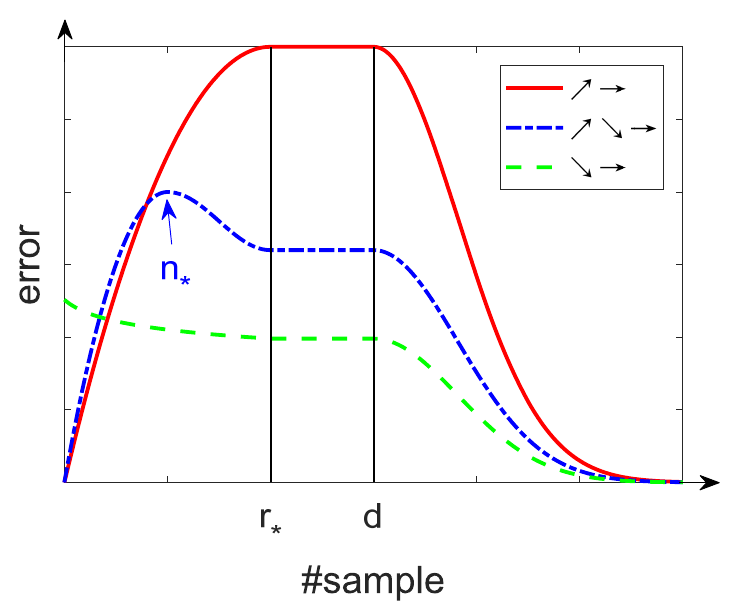}}
	\subfigure[double descent]{\label{tabledoub}
		\includegraphics[width=0.23\textwidth]{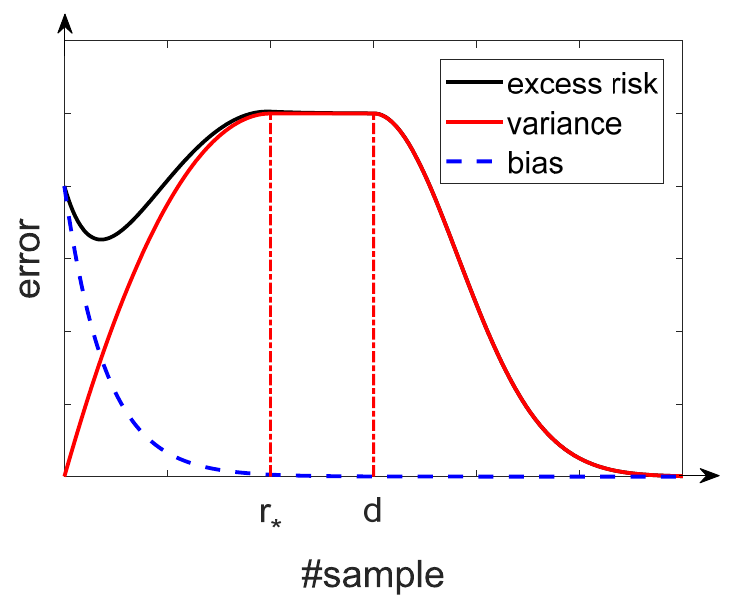}}
	\subfigure[bell-shaped]{\label{tableuni}
		\includegraphics[width=0.23\textwidth]{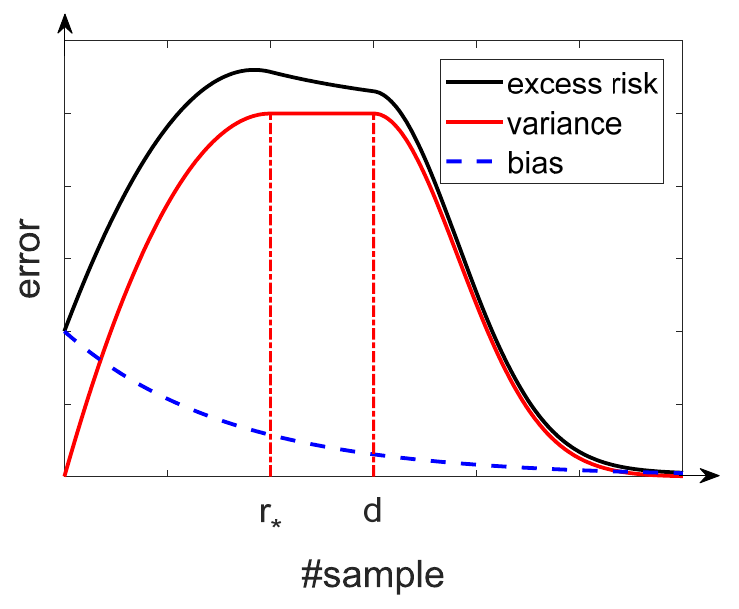}}
	\subfigure[monotonically decreasing]{\label{tabledec}
		\includegraphics[width=0.23\textwidth]{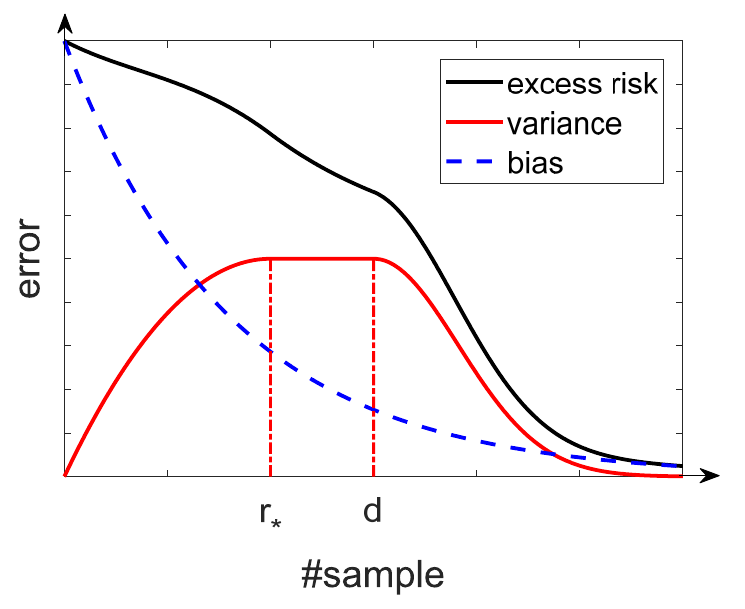}}
	\caption{(a) Trends of variance under different regularization schemes corresponding to Table~\ref{tablevar}. (b-d) Trends of the risk curve under various bias and variance can be double descent, bell-shaped, and monotonically decreasing.}\label{fig-tendency}
\end{figure*}

Here, we consider the kernel ridge regression (KRR) estimator \cite{cucker2002mathematical,suykens2002least,liu2020analysis} in a high dimensional setting with data dimension $d$ and size $n$ both large, and treat the kernel interpolation as a special case of KRR by taking the explicit regularization to be zero.
More precisely, by virtue of the linearization of kernel matrices in high dimensions \cite{liang2020just,el2010spectrum,elkhalil2020risk,liao2018spectrum,liao2019lssvm}, we disentangle the \emph{implicit} regularization of kernel interpolation estimators in an \emph{explicit} manner.
As a result, both implicit and explicit regularization schemes can be systematically studied within the proposed framework.
Mathematically, KRR aims to solve the following empirical risk minimization problem on a training set $\bm z := \{ (\bm x_i, y_i) \}_{i=1}^n$ with data $\bm x_i \in \mathbb{R}^{d}$ and responses $ y_i \in \mathbb{R}$:
\begin{equation}\label{1fzrs}
f_{\bm{z},\lambda} \!:=\! \argmin_{f \in \mathcal{H}} \left\{ \frac{1}{n} \sum_{i=1}^{n} \big(f(\bm x_i) \!-\! y_i \big)^2 \!+\! \lambda \langle f,f\rangle_{\mathcal{H}} \!\right\}\,,
\end{equation}
where an explicit Tikhonov regularization term induced by a reproducing kernel Hilbert space (RKHS) $\mathcal H$ is added to the least-squares objective. 
In statistical learning theory \cite{cucker2007learning}, the regularization parameter $\lambda > 0$ is generally taken to depend on the sample size $n$ in such a way that $\lim_{n \rightarrow \infty} \lambda(n) = 0$. Here we assume that $\lambda := \bar{c} n^{-\vartheta}$ with some $\vartheta \geq 0$ and $0 \leq \bar{c} \leq 1$ to cover the interpolation case.

In this paper, we propose a novel bias-variance decomposition of the KRR expected excess risk, and derive non-asymptotic bounds for both bias and variance. This precise assessment leads to fruitful discussions as a function of different data eigenvalue decays and regularization schemes.
Our main findings include:
\begin{itemize}
	\item  We demonstrate that, for data dimension $d$ large, the kernel matrix admits the same eigenvalue decay as $\bm X \bm X^{\!\top}/d$, where $\bm X = [\bm x_1, \dots, \bm x_n]^{\!\top} \in \mathbb{R}^{n \times d}$ is the data matrix. So in high dimensions, the eigenvalue decay of $\bm K$ is almost determined by the data, as reflected in our error bound for the bias. 
	\item The explicit regularization $\lambda := \bar{c}n^{-\vartheta}$ largely affects the peak point of the variance: a large $\lambda$ decreases the model complexity, and thus corresponds to a small value of interpolation point $n_* \equiv n_*(\lambda)$. 
	Table~\ref{tablevar} shows that, under a small (or zero) regularization so that $r_* \leq n_*$ with $ r_*:= \operatorname{rank}(\bm X \bm X^{\!\top}/d)$: the error bound for variance ${\tt{V}}$ monotonically increases with $n$ until $n := r_*$, as in the red curve of Figure~\ref{tableten1}.
	Under a moderate regularization with $n_* \leq r_*$: ${\tt{V}}$ first increases with $n$ until $n := n_*$ and then decreases. In this case, the peak point will move to the left due to $n_* < d$, see the blue curve in Figure~\ref{tableten1}. Under a large regularization with $n_* \leq c$ for some constant $c$, ${\tt{V}}$ monotonically decreases with $n$, as in the green curve of Figure~\ref{tableten1}. 
	\item Our error bounds for the bias and the variance exhibit different characteristics.  More specifically, the bias bound is (almost) independent of the data/feature dimension $d$ and monotonically decreases with $n$ at a certain $\mathcal{O}(\lambda)$ (learning) rate as in the classical learning theory \cite{cucker2007learning,Wang2011Optimal,steinwart2007fast}.
	Besides, the variance bound depends on $n$ and $d$, and exhibits monotonic decreasing or unimodal with $n$ under different regularizations. Hence, the expected excess risk, as the sum of bias and variance, can be double descent (Figure~\ref{tabledoub}), bell-shaped (Figure~\ref{tableuni}), or monotonic decreasing (Figure~\ref{tabledec}), depending on the level of \emph{implicit} and \emph{explicit} regularizations. 
	This is in agreement with empirical findings in neural networks \cite{yang2020rethinking}.
	\item Our non-asymptotic results show that, for large but fixed $d$, both the variance and bias tends to zero as $n \rightarrow \infty$ under $\lambda := \bar{c} n^{-\vartheta}$, implying that the excess risk approaches zero. Based on this, in the double descent case particularly, the minimum of the expected error in the over-parameterized $n>d$ regime is \emph{lower} than that in the $n<d$ regime. This claim cannot be obtained from \cite{liang2020just}. 
\end{itemize}
The rest of the paper is organized as follows.
We briefly introduce problem settings in Section~\ref{sec:pre}.
In Section~\ref{sec:mainres}, we present our main results on the generalization property of KRR in high dimensions and briefly sketch the main ideas of the proof.
Discussions on the derived error bounds are given in Section~\ref{sec:discusserr}.
In Section~\ref{sec:exp}, we report numerical experiments to support our theoretical results and the conclusion is drawn in Section~\ref{sec:conclusion}.

\section{Problem Settings and Preliminaries}
\label{sec:pre}

We work in the high dimensional regime for some large $d, n$ with $ c \leq d/n \leq C$ for some constants $c,C >0$.
For notational simplicity, we denote by $a(n) \lesssim b(n)$: there exists a constant $\widetilde{C}$ independent of $n$ such that $a(n) \leq \widetilde{C} b(n)$, and analogously for $\asymp$ and $ \gtrsim$.

\subsection{Kernel Ridge Regression Estimator}

Let $X \subseteq \mathbb{R}^d$ be a metric space and $Y \subseteq \mathbb{R}$, the instances $(\bm x_i, y_i)$ in the training set $\bm z = \{  (\bm x_i, y_i) \}_{i=1}^n \in Z^n $ are assumed to be independently drawn from a non-degenerate Borel probability measure $\rho$ on $X \times Y$. The \emph{target function} of $\rho$ is defined by
\begin{equation}\label{eq:def-f-rho}
f_{\rho}(\bm x) = \int_Y y \,\mathrm{d} \rho(y \mid \bm x),~ \bm x \in X\,,
\end{equation}
where $\rho(\cdot\mid\bm x)$ is the conditional distribution of $\rho$ at $\bm x \in X$.
Define the response vector $\bm y = [y_1, y_2, \cdots, y_n]^{\!\top} \in \mathbb R^n$ and the kernel matrix $\bm K = \{ k(\bm x_i, \bm x_j) \}_{i,j=1}^n$ induced by a positive definite kernel $k(\cdot,\cdot)$, KRR aims to find a hypothesis $f: X \rightarrow Y$ such that $f(\bm x)$ is a good approximation of the response $y \in Y$ corresponding to a new instance $\bm x \in X$.
This is actually an empirical risk minimization in problem~\eqref{1fzrs}.
By denoting $k(\bm x, \bm X) = [k(\bm x, \bm x_1), k(\bm x, \bm x_2), \cdots, k(\bm x, \bm x_n)]^{\!\top} \in \mathbb{R}^n$, the closed-form of KRR estimator in Eq.~\eqref{1fzrs} is
\begin{equation}\label{krrclo}
f_{\bm z, \lambda}(\bm x) = k(\bm x, \bm X)^{\!\top} (\bm K+ n \lambda \bm I)^{-1} \bm y\,.
\end{equation}
We consider two popular positive definite kernel classes of (i) the inner-product kernel of the form $k(\bm x_i, \bm x_j) = h\left( \langle \bm x_i, \bm x_j \rangle/d \right)$ and (ii) the \emph{radial} kernel function $k(\bm x_i, \bm x_j) = h\left( \| \bm x_i - \bm x_j \|_2^2/d \right)$. Here $h(\cdot): \mathbb{R} \rightarrow \mathbb{R}$ is a nonlinear function that is assumed to be (locally) smooth, as in \cite{el2010spectrum,liang2020just}.
Examples include commonly used kernels such as linear kernels, polynomial kernels, Sigmoid kernels, exponential kernels, and Gaussian kernels, to name a few.

The expected (quadratic) risk is defined as $\mathcal{E}(f) = \int_Z (f(\bm x) - y)^2 \mathrm{d} \rho$ and the empirical risk functional is defined on the training set $\bm z$, i.e., $ \mathcal{E}_{\bm z}(f) = \frac{1}{n} \sum_{i=1}^{n} \big(f( \bm x_i) - y_i \big)^2$.
To measure the estimation quality of $f_{\bm z, \lambda}$, one natural way is the \emph{expected excess risk}: 
$\mathbb{E}_{y|\bm x}[\mathcal{E}(f_{\bm z, \lambda}) - \mathcal{E}(f_{\rho})]$.
Specifically, in KRR, the expected excess risk admits $\mathbb{E}_{y|\bm x}[\mathcal{E}(f_{\bm z, \lambda}) - \mathcal{E}(f_{\rho})] = \mathbb{E}_{y|\bm x}\| f_{\bm{z},\lambda}  - f_{\rho} \|^2_{\mathcal{L}_{\rho_X}^{2}}$,
which is exactly in the weighted $\mathcal{L}^{2}$-space with the norm $\|f\|^2_{\mathcal{L}^{2}_{\rho_{{X}}}} =  \int_{{X}} |f(\bm x)|^{2} \mathrm{d} \rho_{X}(\bm x) $.

\subsection{Background on RKHS}

Now we characterize the integral operators defined by a kernel. 
Given a kernel $k$, its integral operator $L_{K}: \mathcal{L}_{\rho_X}^2 \rightarrow \mathcal{L}_{\rho_X}^2$ admits
\begin{equation}
(L_{K} f)(\cdot)=\int_{X} k(\cdot, \bm x) f(\bm x) d \rho_{X}(\bm x), \quad \forall f \in \mathcal{L}_{\rho_{X}}^{2} \,.
\end{equation}
Since $L_K$ is compact, positive definite and self-adjoint, by the spectral theorem (see, Theorem A.5.13 in \cite{Steinwart2008SVM}), there exists countable pairs of eigenvalues and eigenfunctions $\{\mu_i, \psi_i\}_{i =1}^{\infty}$ of $L_K$ such that $L_K \psi_i = \mu_i \psi_i$, 
where $\{ \psi \}_{i=1}^{\infty}$ are orthogonal basis of $\mathcal{L}_{\rho_X}^2(X)$ and $\mu_1 \geq \mu_2 \cdots > 0$ with $\lim\limits_{i \rightarrow \infty} \mu_i = 0$.
Accordingly, by Mercer's theorem, we have
$ k(\bm x, \bm x') = \sum_{i=1}^{\infty} \mu_i \psi_i(\bm x) \psi_i(\bm x')$, 
and there exists a constant $\kappa \geq 1$ such that 
$\sup_{\bm x \in X} \sum_{i=1}^{\infty} \mu_i \psi_i^2(\bm x) \leq \kappa^2
$.
It holds by $\kappa:= \max\{ 1, \sup_{\bm x \in X} \sqrt{k(\bm x, \bm x)} \}$.
Based on the data matrix $\bm X$ and the integral operator $L_K$, the empirical integral operator is given by
$L_{K, \bm X} = \frac{1}{n} \sum_{i=1}^n k(\cdot, \bm x_i) \otimes k(\cdot, \bm x_i)$,
which converges to the data-free limit $L_K$ at an $\mathcal{O}(1/\sqrt{n})$ rate \cite{de2009elastic}.

\section{Main Results}
\label{sec:mainres}

In this section, we state our main result under some basic/technical assumptions, compare it with existing results, and sketch the main ideas of our proof.

\subsection{Basic results}
\label{sec:basic}

To illustrate our analysis, we need the following three standard assumptions.
\begin{assumption}\label{assexist}
	(Existence of $f_{\rho}$)
	We assume $f_{\rho} \in \mathcal{H}$.
\end{assumption}
This is a standard assumption in learning theory and assumes that the target function $f_\rho$ defined in Eq.~\eqref{eq:def-f-rho} is indeed realizable, see also \cite{richards2019optimal,Rudi2017Generalization,cucker2007learning,steinwart2007fast}.

\begin{assumption}\label{assber}
	(Noise condition \cite{liang2020just,dobriban2018high})
	There exists $\sigma$ such that
	$\mathbb{E}[(f_{\rho}(\bm x) - y)^2 \mid \bm x ] \leq \sigma^2, $ almost surely.
\end{assumption}
This is a broad model for the noise in the output $y$, containing uniformly bounded or sub-Gaussian noise; and is in fact weaker than the standard Bernstein condition, e.g., in \cite{blanchard2010optimal}.

\begin{assumption}\label{assum8m} ((8+$m$)-moments \cite{liang2020just,liu2020ridge})
	Let $\bm x_i = \bm \Sigma_d^{1/2} \bm t_i$, where $\bm t_i \in \mathbb{R}^d$ has i.i.d. entries with zero mean, unit variance, and a finite (8+$m$)-moments, i.e., its entry $\bm t_i(j)$, $1 \leq j \leq d$, satisfies $\mathbb{E}[\bm t_i(j)] = 0$, $\mathbb{V}[\bm t_i(j)] = 1$, and $\mathbb{E}(|\bm t_i(j)|) \leq C d^{\frac{2}{8+m}}$ such that $\mathbb E [\bm x_i \bm x_i^{\!\top}] = \bm \Sigma_d$ with a bounded spectral norm $\| \bm \Sigma_d \|_2$, for some $m > 0$.
\end{assumption}
This is a standard setting in high-dimensional statistics and random matrix theory \cite{el2010spectrum,dobriban2018high,liang2020just,hastie2019surprises,elkhalil2020risk} that assumes that the data are drawn from some not-too-heavy-tailed distribution, with possibly (involved) structure between the entries.

\begin{table}[tp]
	\centering
	\fontsize{7}{8}\selectfont
	\begin{threeparttable}
		\caption{Parameters of the linearized kernel $\widetilde{\bm K^{\operatorname{lin}}}$ in \cite{el2010spectrum}.}
		\label{tabparam}
		\begin{tabular}{cccccccccccccccccccc}
			\toprule
			parameters &inner-product kernels & radial kernels  \cr
			\midrule
			$\alpha$  &$h(0)+h^{\prime \prime}(0) \frac{\operatorname{tr}\left(\bm \Sigma_{d}^{2}\right)}{2d^{2}}$ & $h(2\tau)+2h^{\prime \prime}(2\tau) \frac{\operatorname{tr}\left(\bm \Sigma_{d}^{2}\right)}{d^{2}}$ \cr
			\midrule
			$\beta$  &$h^{\prime}(0)$ & $-2h^{\prime}(2\tau)$ \cr
			\midrule
			$\gamma$  &$h(\tau)-h(0)-\tau h^{\prime}(0)$ & $h(0)+ 2\tau h^{\prime}(2\tau) - h(2\tau)$ \cr
			\midrule
			$\bm E$ & $\bm 0_{n \times n}$ & $h^{\prime}(2\tau) \bm A + \frac{1}{2}h^{\prime \prime}(2\tau) \bm A \odot \bm A$ \tnote{1} \cr
			\bottomrule
		\end{tabular}
		\begin{tablenotes}
			\footnotesize
			\item[1] $\bm A := \bm 1 \bm \psi^{\!\top} + \bm \psi \bm 1^{\!\top}$, where $\bm \psi \in \mathbb{R}^n$ with $\psi_i := \| \bm x_i \|^2_2/d - \tau$ and $\tau := \operatorname{tr}(\bm \Sigma_d)/d$.
		\end{tablenotes}
	\end{threeparttable}\vspace{-0.1cm}
\end{table}

To aid our proof, we need some extra results.
In \cite{el2010spectrum}, it has been shown that the kernel matrix $\bm K$ in high dimensions can be well approximated by $\widetilde{\bm K^{\operatorname{lin}}}$ in spectral norm, i.e., $\| \bm K - \widetilde{\bm K^{\operatorname{lin}}} \|_2 \rightarrow 0$ as $n,d \to \infty$\vspace{-0.1cm}
\begin{equation}\label{eqlinear}
\widetilde{\bm K^{\operatorname{lin}}} := \alpha \bm 1\bm 1^{\!\top}+\beta \frac{\bm X \bm X^{\!\top}}{d} + \gamma \bm I + \bm E \,,
\end{equation}
with non-negative parameters $\alpha$, $\beta$, $\gamma$, and the additional matrix $\bm E$ given in Table~\ref{tabparam}, see some typical examples in Appendix~\ref{sec:examples}.
Here $\gamma$ is the \emph{implicit} regularization parameter in kernel estimator that depends on the nonlinear function $h$ in the kernel $k$ and the data structure $\bm \Sigma_d$.
According to Eq.~\eqref{eqlinear}, denote the shortcut $\widetilde{\bm X} := \beta {\bm X \bm X^{\!\top}}/{d}+\alpha \bm 1\bm 1^{\!\top}$, we show in high dimensions that, $\bm K$ admits the same eigenvalue decay as $\widetilde{\bm X}$ and $\bm X \bm X^{\!\top}/d$ (see details in Appendix~\ref{sec:proofeig}).
Subsequently, we introduce the following quantity function
\begin{equation}\label{defNbx}
\mathcal{N}^b_{\widetilde{\bm X}}: = \operatorname{tr}\left[(\widetilde{\bm X} + b \bm I_n)^{-2} \widetilde{\bm X}\right] = \sum_{i=1}^n \frac{\lambda_{i}(\widetilde{\bm X} )}{\left[b +\lambda_{i}(\widetilde{\bm X})\right]^{2}}\,,
\end{equation} 
which is associated with various quantity functions in \cite{ali2019continuous,dobriban2018high,liang2020just,jacot2020kernel,nakkiran2020optimal} and, as we shall see, plays an important role in determining the variance behavior.
We will discuss at length $\mathcal{N}^b_{\widetilde{\bm X}}$ based on different data eigenvalue decays in Section~\ref{sec:discusserr}.

Formally, our main results of KRR in a high-dimensional regime are stated as follows.
\begin{theorem} (Basic result) \label{promainba}
	Under Assumptions~\ref{assexist}-\ref{assum8m}, let $0 < \delta <1/2$, $\theta = \frac{1}{2} - \frac{2}{8+m}$, $d$ large enough, taking the regularization parameter $\lambda := \bar{c} n^{-\vartheta}$ with $0 \leq \vartheta \leq 1/2$, for any given $\varepsilon > 0$, it holds with probability at least $1 - 2\delta - d^{-2}$ with respect to the draw of ${\bm X}$ that
	\begin{equation}\label{resmainba}
	\mathbb{E}_{y|\bm x}\big\| f_{\bm{z},\lambda} \! - \! f_{\rho} \big\|^2_{\mathcal{L}_{\rho_X}^{2}}  \!\lesssim \! n^{- \vartheta}\! \log^4\! \Big(\frac{2}{\delta}\Big)  + {\tt{V}}_1 + {\tt{V}}_2 \,,
	\end{equation}
	with ${\tt{V}}_1 := \frac{\sigma^2\beta}{d} \mathcal{N}^{n\lambda + \gamma}_{\widetilde{\bm X}}$ and the residual term ${\tt{V}}_2$
	\begin{equation*}
	{\tt{V}}_2 := \left\{
	\begin{array}{rcl}
	\begin{split}
	&\frac{\sigma^2 \log ^{2+4\varepsilon} d}{(n\lambda+\gamma)^2 d^{4 \theta-1}} ,~\mbox{for inner-product kernels} \\
	& \frac{\sigma^2}{(n\lambda+\gamma)^2} d^{-2 \theta} \log ^{1+\varepsilon} d,~\mbox{for radial kernels}\,.
	\end{split}
	\end{array} \right.
	\end{equation*}
\end{theorem}
\noindent{\bf Remark:} The first term in Eq.~\eqref{resmainba} is the bound of the bias, which is independent of $d$ and monotonically decreases with $n$. The sum ${\tt{V}}_1 + {\tt{V}}_2$ is the bound of the variance that depends on both $n$ and $d$. Note that ${\tt{V}}_2$ monotonically decreases with $n$, and approaches to zero for a large $n$. Therefore, the error bound for ${\tt{V}}_1 \asymp \frac{1}{d} \mathcal{N}^{n\lambda + \gamma}_{\widetilde{\bm X}}$ is the key part of estimates for the variance and will be discussed in in Section~\ref{sec:discusserr}, where $n \lambda$ corresponds to the \emph{explicit} regularization and $\gamma$ the \emph{implicit} regularization. 
We will demonstrate that ${\tt{V}}_1$ can be monotonically decreasing or unimodal under different regularization schemes.
Such monotonic bias and unimodal variance can lead to various behaviors of the excess risk, including monotonically decreasing, double descent, and bell-shaped risk curve, as illustrated in Figure~\ref{fig-tendency} of introduction.

\subsection{Refined result}
Based on the basic result, if we consider two additional assumptions, i.e., extending Assumption~\ref{assexist} by considering the regularity of $f_{\rho}$ and studying spectral decay of $k$ via complexity of $\mathcal{H}$, we can obtain a refined result.
\begin{assumption}\label{sourcecon} (Source condition \cite{cucker2007learning}) 
	For some  $0 < r \leq 1$, there exists $g_{\rho} \in \mathcal{L}_{\rho_X}^2$ satisfying $\| g_{\rho} \|_{\mathcal{L}_{\rho_X}^2} \leq R$ such that $f_{\rho} = L_K^r g_{\rho}$.
\end{assumption}
It has been widely used in the literature of learning theory to assess the regularity of $f_{\rho}$ \cite{cucker2007learning,zhang2013divide,Rudi2017Generalization}, which indicates $f_{\rho}$ belongs to the range space of $L_K^r$.
Assumption~\ref{assexist} is the worst case of Assumption~\ref{sourcecon} by choosing $r=1/2$ since $\| f \|_{\mathcal{L}_{\rho_X}^2} = \| L_K^{1/2} f \|_{\mathcal{H}},~\forall f \in {\mathcal{L}_{\rho_X}^2}$.
\begin{assumption}
	(Capacity condition \cite{cucker2007learning})
	\label{effectass}
	For any $\lambda > 0$, there exist $Q > 0$ and $\eta \in [0,1]$ such that
	\begin{equation*}\label{nlambda}
	\begin{split}
	\mathcal{N}(\lambda) &:= \operatorname{tr}\left((L_K+\lambda I)^{-1} L_K \right) \leq Q^2 \lambda^{-\eta}\,.
	\end{split}
	\end{equation*}
\end{assumption}
The notation $\mathcal{N}(\lambda)$ denotes the ``effective dimension" and can be regarded as a ``measure of size" of the RKHS.
This is a natural and widely used assumption in the literature \cite{cucker2007learning,zhang2013divide,Rudi2017Generalization}.
Assumption~\ref{effectass} always holds for $\eta=1$ and $Q=\kappa$ where $\kappa:= \max\{ 1, \sup_{\bm x \in X} \sqrt{k(\bm x, \bm x)} \}$ as $L_K$ is a trace class operator.
Its kernel matrix form is
$d_{\bm K}^{\lambda} := \operatorname{tr} \left( (\bm K + \lambda \bm I_n)^{-1} \bm K \right) = \sum_{i=1}^n \frac{\lambda_i(\bm K)}{\lambda_i(\bm K) + \lambda}$ \cite{avron2017random,li2019towards}.
While Assumption~\ref{effectass} can be further refined to obtain a bound that depends on $d$ \cite{pagliana2020interpolation}, here we focus on the eigenvalue decay of $\bm K$, see Section \ref{sec:discusserr} for details.

Based on the above discussion, we obtain a refined result of Theorem~\ref{promainba} as below.
\begin{theorem}(Refined result)\label{promain}
	Under Assumptions~\ref{assber}-\ref{effectass}, let $0 < \delta <1/2$, $\theta = \frac{1}{2} - \frac{2}{8+m}$, and $d$ large enough, taking $\lambda := \bar{c} n^{-\vartheta}$ with $0 \leq \vartheta \leq \frac{1}{1+\eta}$, then for any given $\varepsilon > 0$, it holds with probability at least $1 - 2\delta - d^{-2}$
	\begin{equation}\label{resmain}
	\mathbb{E}_{y|\bm x}\big\| f_{\bm{z},\lambda} \! - \! f_{\rho} \big\|^2_{\mathcal{L}_{\rho_X}^{2}}  \!\!\leq\!  n^{-2 \vartheta r}\! \log^4\! \Big(\frac{2}{\delta}\Big)  + {\tt{V}}_1 + {\tt{V}}_2 \,,
	\end{equation}
	where ${\tt{V}}_1 $ and ${\tt{V}}_2 $ are the same as in Theorem~\ref{promainba}.
\end{theorem}
\noindent{\bf Remark:} Compared to classical learning theory results \cite{fischer2017sobolev} achieving $\mathcal{O}(n^{-\frac{2r+1}{2r+1+\eta}})$ learning rates, the parameter $\eta$ in our results only effects the selection range of $\lambda$, which is nearly independent of the learning rates to some extent. That means, the spectral decay of a kernel function $k$ in high dimensions is almost irrelevant to its kernel type.
In fact, the eigenvalue decay of the kernel matrix in our model largely depends on the data, which is in essence different from classical learning theory results. 
Therefore, our result reflects a certain ``universality" on the kernel function in high dimensional problems, which shows consistency to \cite{el2010spectrum}.

\subsection{Related work}
We provide non-asymptotic results that systematically analyze both implicit and explicit regularization schemes within a unified framework. 

{\bf Implicit regularization in kernel/linear interpolation:}
Implicit regularization can be induced by minimum norm solutions in linear interpolation \cite{derezinski2019exact,kobak2020optimal}, or the curvature of the kernel function in kernel interpolation \cite{liang2020just}.
Compared to the risk curve in \cite{liang2020just} that converges to a non-zero constant, the risk curve in our results tends to zero when $n \gg d$.
Hence our result demonstrates that, in the double descent case, the minimum of the expected risk in the second descent is lower than the first descent; while the same claim cannot be obtained from \cite{liang2020just}.
Besides, under the basic $f_{\rho} \in \mathcal{H}$ case, our bias bound is based on the eigen-decay (trends) of the kernel matrix $\bm K$ and thus can be (almost) independent of $d$, achieving an optimal learning rate $\mathcal{O}(\lambda)$ in a minimax case. This is different from \cite{liang2020just} that corresponds to the sum of tailed eigenvalues of $\bm K$.
Specifically, if we directly set $\lambda$ to zero, our result for the bias still holds, which can be bounded by $ \| L_{K,\bm X} - L_K \|_{\mathcal{L}_{\rho_X}^{2}} \lesssim \mathcal{O}(1/\sqrt{n})$.

{\bf Explicit regularization in kernel/linear regression:}
We provide non-asymptotic results that refine a series of asymptotic analyses, e.g., the Stieltjes transform approach in \cite{hastie2019surprises,elkhalil2020risk,yang2020rethinking,jacot2020implicit} and the statistical mechanic approach in \cite{canatar2020statistical}. 
In fact, by considering the limiting eigenvalue distribution of $\bm X \bm X^{\!\top}/d$ via its Stieltjes transform $\frac{1}{n}\mathcal{N}^b_{\bm X \bm X^{\!\top}/d} \!\approx\!   m(-b) - b m'(-b) $, for $m(b)$ the solution to the popular Mar{\u c}enko–Pastur equation \cite{marvcenko1967distribution}, our error bound recovers \cite[Theorem 5]{hastie2019surprises} with $b:=\lambda$ and isotropic features $\bm \Sigma_d = \bm I_d$.
Finite sample analyses are often based on a finer control of the Stieltjes transform \cite{jacot2020kernel} or the effective rank \cite{bartlett2020benign,chinot2020benign}.
However, the aforementioned results are generally limited to Gaussian \cite{jacot2020kernel,nakkiran2020optimal} and sub-Gaussian data \cite{bartlett2020benign,chinot2020benign,caron2020finite}, or Gaussian covariates \cite{richards2020asymptotics}. Here we consider a much broader family of distributions. 
Besides, under some specific situations, the regularization parameter $\lambda$ in (generalized) linear regression can be negative \cite{kobak2020optimal} or optimal tuned \cite{nakkiran2020optimal,wu2020optimal} so as to generalize well.
Recent research \cite{ghorbani2019linearized,liang2020multiple,bordelon2020spectrum} on kernel regression in $n:= \mathcal{O}(d^c)$ shows different trends.

\subsection{Proof framework}
\label{sec:proofframework}
The proof of our results is fairly technical and lengthy, and we briefly sketch some main ideas of Theorem~\ref{promain} here. 
Note that, Theorem~\ref{promainba} is a special case of Theorem~\ref{promain} by taking $r=1/2$ and $\eta = 1$.
The modified error decomposition, the error bounds of variance for radial kernels, and estimates for bias are the main elements of novelty in the proof.

In order to estimate the error $\mathbb{E}_{y\mid \bm x}\| f_{\bm{z},\lambda}  - f_{\rho} \|$ in the $\mathcal{L}_{\rho_X}^{2}$ space, we
need the following intermediate functions.
Define $f_{\lambda} = (L_K + \lambda I)^{-1} L_K f_{\rho}$, where $I$ is the identity operator, then $f_{\lambda}$ is actually the minimizer of the following problem
$
f_\lambda = \argmin_{f \in \mathcal{H}} \Big\{ \| f - f_{\rho} \|^2_{\mathcal{L}_{\rho_X}^{2}} + \lambda \| f \|_{{\mathcal{H}}}^2 \Big\} $.
Besides, by defining \vspace{-0.1cm}
\begin{equation*}
f_{\bm X, \lambda}(\bm x) = k(\bm x, \bm X)^{\!\top} (\bm K+ n \lambda \bm I)^{-1} f_{\rho}(\bm x)\,,
\end{equation*}
we have $f_{\bm X, \lambda} = (L_{K,\bm X} + \lambda I)^{-1} L_{K,\bm X} f_{\rho}$.
Accordingly, the variance-bias decomposition is stated in the following lemma, with proof deferred to Appendix~\ref{sec:prooflemma1}.
\begin{lemma}\label{lemerr}
	Let $f_{\bm{z},\lambda}$ be the minimizer of problem~\eqref{1fzrs}, $\mathbb{E}_{y|\bm x}\| f_{\bm{z},\lambda}  - f_{\rho} \|^2_{\mathcal{L}_{\rho_X}^{2}}$ can be bounded by
	\begin{equation*}
	\begin{split}
	& \mathbb{E}_{y \mid \bm x}\big\| f_{\bm{z},\lambda}  - f_{\rho} \big\|^2_{\mathcal{L}_{\rho_X}^{2}} = {\tt{B}} + {\tt{V}} \\
	& \leq 2\left( \| f_{\bm X, \lambda} - f_{\lambda} \|^2_{\mathcal{L}^2_{\rho_X}} + \| f_{\lambda} - f_{\rho} \|^2_{\mathcal{L}^2_{\rho_X}} \right) +  {\tt{V}} 
	\end{split}
	\end{equation*}
	where the bias $\tt{B}$ is defined as
	\begin{equation}\label{biaslemma1}
	{\tt{B}} := \mathbb{E}_{\bm x} \big\| k(\bm x, \cdot)^{\!\top} (\bm K + n\lambda \bm I)^{-1} f_{\rho}(\bm X) - f_{\rho} \big\|^2_{\mathcal{L}_{\rho_X}^{2}} \,,
	\end{equation}
	where $f_{\rho}(\bm X) = [f_{\rho}(\bm x_1), f_{\rho}(\bm x_1), \cdots, f_{\rho}(\bm x_n)]^{\!\top} \in \mathbb{R}^n$ and the variance ${\tt{V}}$ is defined as 
	\begin{equation}\label{defvariance}
	{\tt{V}} := \mathbb{E}_{\bm x,y} \left\| k(\bm x, \cdot)^{\!\top} (\bm K + n\lambda \bm I)^{-1} \bm \epsilon \right\|^2_{\mathcal{L}_{\rho_X}^{2}}\,,
	\end{equation}
	where $\bm \epsilon:= \bm y - f_{\rho}(\bm X)$ satisfying $\mathbb{E}_{y | \bm x} [\bm \epsilon] = 0$.
\end{lemma}
It is clear that, the variance term does not depend on the target function $f_{\rho}$, and the bias is independent of the residual error $\bm \epsilon$. 
Proof for the bias $	{\tt{B}}  \lesssim n^{-2\vartheta r} \log^4 \big(\frac{2}{\delta}\big) $ can be found in Appendix~\ref{sec:proofbiass}.
Proof for the variance ${\tt{V}} \lesssim {\tt{V}}_1 + {\tt{V}}_2$ refers to Appendix~\ref{app:provar}.

\section{Discussion on Error Bounds}
\label{sec:discusserr}

In this section, we discuss our Theorem~\ref{promain} for different eigenvalue profiles of $\widetilde{\bm X}$ in the two regimes of $n < d$ and $n > d$.
Since $\bm K$ shares the same eigenvalue decay as $\bm X \bm X^{\!\top}/d$ and $\widetilde{\bm X}$ (see Proposition~\ref{sec:proeig} in Appendix~\ref{sec:proofeig}), we do not distinguish the eigen-decay of these two data matrices in the subsequent discussions.
We first focus on the variance ${\tt{V}}$ that can be unimodal or monotonically decreasing with $n$ under different regularization schemes. Subsequently, we investigate the total risk curve as the sum of bias and variance.
Note that $\widetilde{\bm X}$ has different numbers of non-zero eigenvalues under the two regimes, we denote $r_* := \operatorname{rank}(\widetilde{\bm X}) \leq \min\{ n,d \}$, which, as we shall see, plays a significant role in characterizing the different cases of our bounds.

\subsection{Variance trend for $n < d$}
\label{sec:nleqd}

We consider here three eigenvalue decays of $\widetilde{\bm X}$: \emph{harmonic}, \emph{polynomial}, and \emph{exponential decay} \cite{bach2013sharp,li2019towards}.

\begin{table}[t]
	\centering
	\caption{Three eigenvalue decays of $\widetilde{\bm X}$.  }\label{tableeig}
	\begin{threeparttable}
		\begin{tabular}{c|c|ccccccc}
			\toprule[1.2pt]
			\multirow{2}{*}{eigenvalue decay}    & \multicolumn{2}{c}{$\lambda_{i}(\widetilde{\bm X})$  }                         \\ 
			\cline{2-3}
			& $i \leq r_*$ & $i>r_*$ \\
			\midrule[0.8pt]
			\emph{harmonic decay} &  $n/i$               & \multirow{3}{*}{0} \\  
			\emph{polynomial decay} &  $n i^{-2a}$ with $a>1/2$ & \\
			\emph{exponential decay} &  $ne^{-ai}$ with $a > 0$ &\\
			\bottomrule[1.2pt]      
		\end{tabular}
	\end{threeparttable}
\end{table}

\begin{proposition}\label{pro:decay}
	Under the three eigenvalue decays in Table~\ref{tableeig}, denote $r_* = \operatorname{rank}(\widetilde{\bm X})$, then the quantity function $\mathcal{N}^{b}_{\widetilde{\bm X}}$ with $b := n\lambda + \gamma$ can be bounded by\\
	1) harmonic decay: 
	$\mathcal{N}^{b}_{\widetilde{\bm X}} \leq \frac{n}{b^2} \ln \frac{n+(r_*+1)b}{n+b} = \mathcal{O}(\frac{n}{b^2})$.\\
	2) polynomial decay: 
	$\mathcal{N}^{b}_{\widetilde{\bm X}} \leq \frac{\widetilde{C}}{2ab} \left( \frac{n}{b} \right)^{\frac{1}{2a}}$, where $\widetilde{C}$ is some constant.\\
	3) exponential decay: $\mathcal{N}^{b}_{\widetilde{\bm X}} \!\leq\!  \frac{1}{a} \left( \frac{1}{b+ne^{-a(r_*+1)}} \!-\! \frac{1}{b+ne^{-a}} \right)$.
\end{proposition}
\begin{proof}
	The proof can be found in Appendix~\ref{sec:proofdecay}.
\end{proof}

According to Proposition~\ref{pro:decay}, we summarize our results in Table~\ref{tablevar} and discuss them as follows:

{\bf \emph{Harmonic decay}}:
${\tt{V}}_1 \leqslant \mathcal{O}(\frac{n}{b^2d})$.

For $\lambda=0$, i.e., the ridgeless case, we have $b = \gamma = \mathcal{O}(1)$, and ${\tt{V}}_1 \leq \mathcal{O}(\frac{n}{d})$, which indicates ${\tt{V}}_1$ increases with $n$ in the $n<d$ regime.
For $\lambda \neq 0$, taking $\lambda := \bar{c} n^{-\vartheta}$, we have ${\tt{V}}_1 \leqslant \mathcal{O}(\frac{n}{d ( \bar{c}n^{1-\vartheta} + \gamma)^2} )$. 
To investigate the monotonicity of $g(n):= \frac{n}{d (\bar{c}n^{1-\vartheta} + \gamma)^2}$, 
define
$	n_* := \left( \frac{\gamma}{2 - 2 \vartheta - \bar{c} } \right)^{\frac{1}{1-\vartheta}}$,
we find that, a large $\lambda$ leads to  a small $n_*$. 
According to the relationship between $r_*$, $n_*$, and $d$, we can conclude that (see Table~\ref{tablevar} and the red curve in Figure~\ref{tableten1}):

When $\vartheta \geq \frac{1}{2(2-\bar{c})}$, ${\tt{V}}_1$ will increase with $n$ until $n:= r_*$ and then remain unchanged when $r_*<n<d$. \\
When $\vartheta < \frac{1}{2(2-\bar{c})}$, there are various trends as follows:\\
1) if $d < n_*$, this is the same as the $\vartheta \geq \frac{1}{2(2-\bar{c})}$ case;\\
2) if $r_* <n_*<d$, ${\tt{V}}_1$ will increase with $n$ until $n:= r_*$, and then remain unchanged when $r_* < n < d$;\\
3) if $n_* < r_* < d $, ${\tt{V}}_1$ will increase with $n$ until $n:= n_*$ and then decrease with $n$ until $n:= r_*$, and stay unchanged on $r_*<n<d$;\\
4) If $n_* < c$ such that $n > c$ always holds for some constant $c$, we have ${\tt{V}}_1$ increases with $n$ until $n:= r_*$, and then stays unchanged on $r_* < n < d$.
Remark that, for $\gamma < 2 -2 \vartheta - \bar{c} $, we have $n_* < 1$ and thus $n > n_*$, so that ${\tt{V}}_1$ always decreases with $n$ until $n:= r_*$.

{\bf \emph{Polynomial decay:}} ${\tt{V}}_1 \leqslant \mathcal{O}(\frac{1}{bd}(\frac{n}{b})^{\frac{1}{2a}})$. 

Similar to above, define
$n_* = \left( \frac{\gamma}{2a\bar{c}[1 - (1+\frac{1}{2a}) \vartheta]} \right)^{\frac{1}{1-\vartheta}}$, 
we obtain results similar to the case of \emph{harmonic decay}, but with different thresholds: $\vartheta \geq (1+\frac{1}{2a})^{-1}$ and $\vartheta <  (1+\frac{1}{2a})^{-1}$, see Table~\ref{tablevar} for details.

{\bf \emph{Exponential decay:}} ${\tt{V}}_1 \leq \frac{\widetilde{C} \beta}{ad} \big( \frac{1}{b+ne^{-a(r_*+1)}} \!-\! \frac{1}{b+ne^{-a}} \big)$.

Here we consider the monotonicity of the function $G(n) := \big( \frac{1}{b+ne^{-a(r_*+1)}} \!-\! \frac{1}{b+ne^{-a}} \big)$ with $b:=n\lambda + \gamma$ to study the trend of ${\tt{V}}_1$ regarding to $n$.
Let $n_*$ be the solution of the equation $G'(n)=0$, then we have the similar conclusion with that of \emph{harmonic decay} and \emph{polynomial decay} by the relationship between $n_*$, $r_*$, and $d$, see Table~\ref{tablevar} for details.
More specifically, under some certain conditions, ${\tt{V}}_1$ is able to monotonically decrease with $n$, refer to  Appendix~\ref{sec:appnleqd} for details.

\subsection{Variance trends for $n > d$ and total risk}
Different from the above $n<d$ case, the current $n>d$ regime admits that $\widetilde{\bm X}$ has at most $d$ non-zero eigenvalues.
In this under-parameterized regime, we are particularly interested in the behavior as $n \to \infty$.
In Appendix~\ref{sec:ngeqd}, we prove that ${\tt{V}}_1$ approaches to zero as $n \to \infty$ under the above three eigenvalue decays.

Based on the above discussions in the $n > d$ and $n<d$ regimes, we conclude that, the variance can be unimodal (small regularization) or decreasing (large regularization) as $n$ grows,
which, together with the fact that the bias is monotonically decreasing with $n$, leads to the following three configurations for the total risk:
(i) if the bias dominates at small $n$ and then decays fast (i.e., with a small regularization), we observe a double descent curve as in Figure~\ref{tabledoub}; (ii) if the bias dominates but decays slowly (with a large regularization), the risk curve will be monotonic decreasing as in Figure~\ref{tabledec};
(iii) if the variance dominates, a bell-shaped risk curve as in  Figure~\ref{tableuni} will be observed.

\begin{figure}[t]
	\centering
	\subfigure[\emph{poly kernel with order 3}]{
		\includegraphics[width=0.33\textwidth]{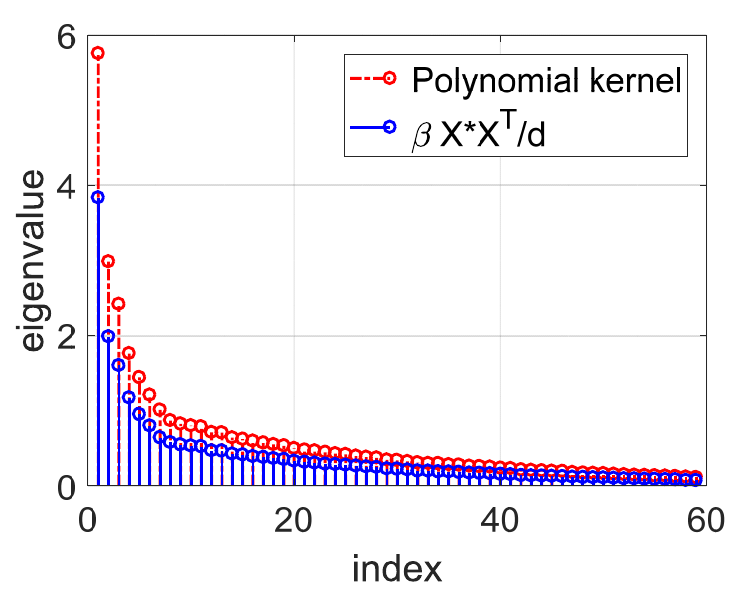}}
	\subfigure[\emph{Gaussian kernel}]{
		\includegraphics[width=0.33\textwidth]{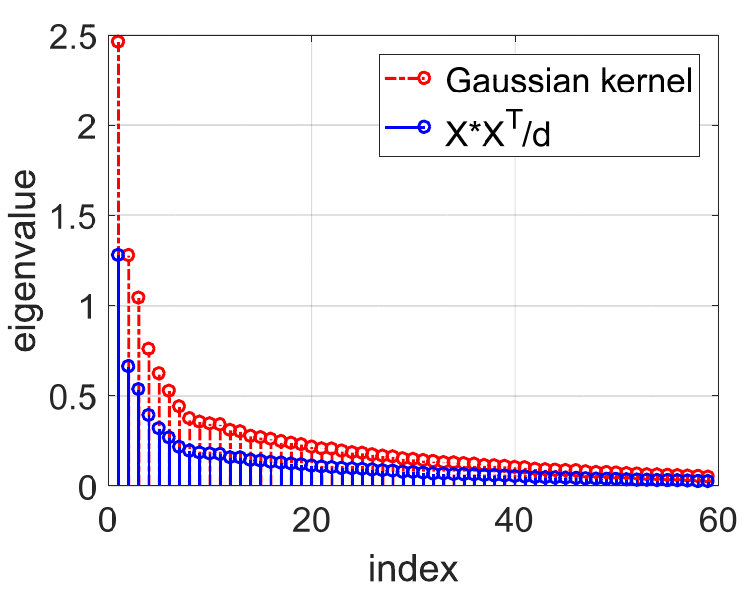}}
	\caption{Top 60 eigenvalues of two kernel matrices and their linearizations on the subset of the \emph{YearPredictionMSD} dataset. Note that the largest eigenvalue $\lambda_1$ is not plotted for better display.
	}
	\label{fig-year-eig}\vspace{-0.25cm}
\end{figure}

\begin{figure*}[!htb]
	\centering
	\subfigure[$\vartheta=2/3$]{
		\includegraphics[width=0.23\textwidth]{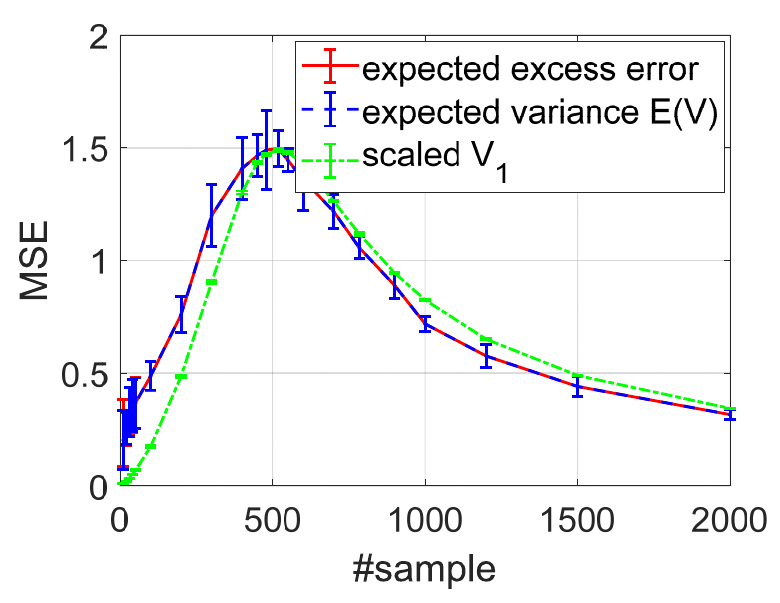}}
	\subfigure[$\vartheta=2/3$]{
		\includegraphics[width=0.23\textwidth]{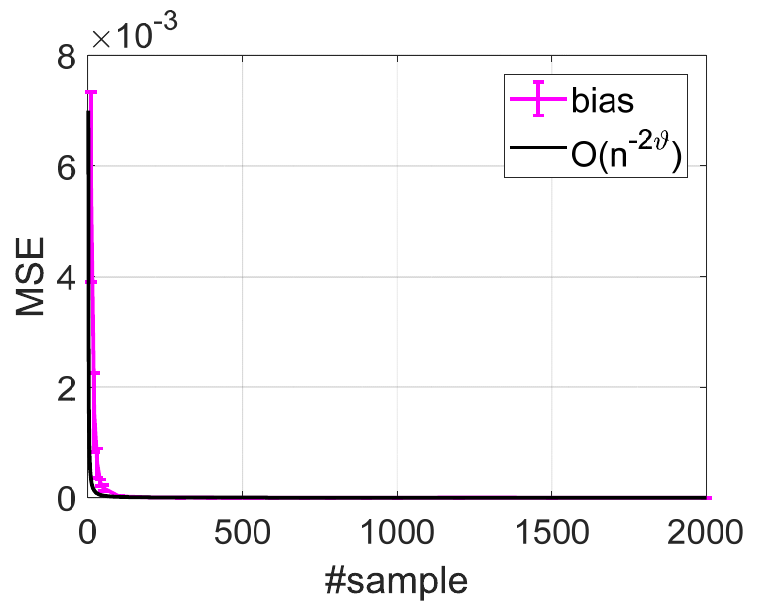}}
	\subfigure[$\vartheta=1/3$]{
		\includegraphics[width=0.23\textwidth]{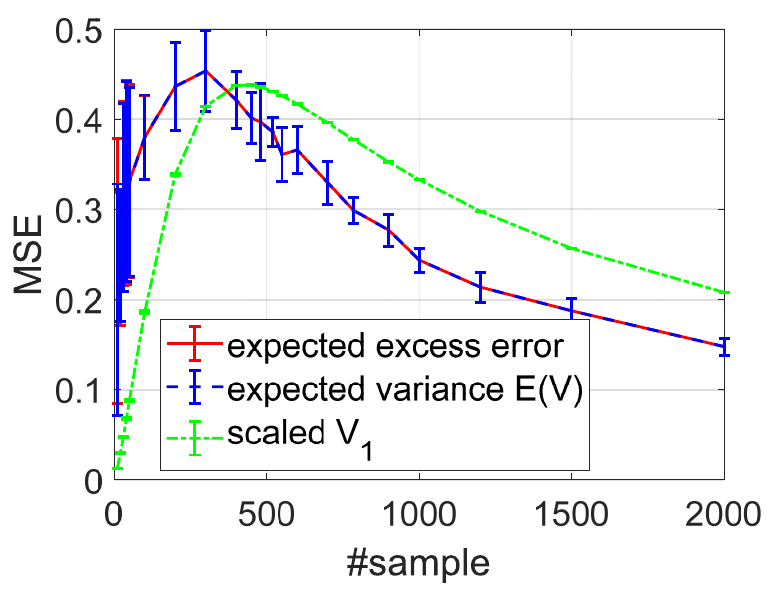}}
	\subfigure[$\vartheta=1/3$]{
		\includegraphics[width=0.23\textwidth]{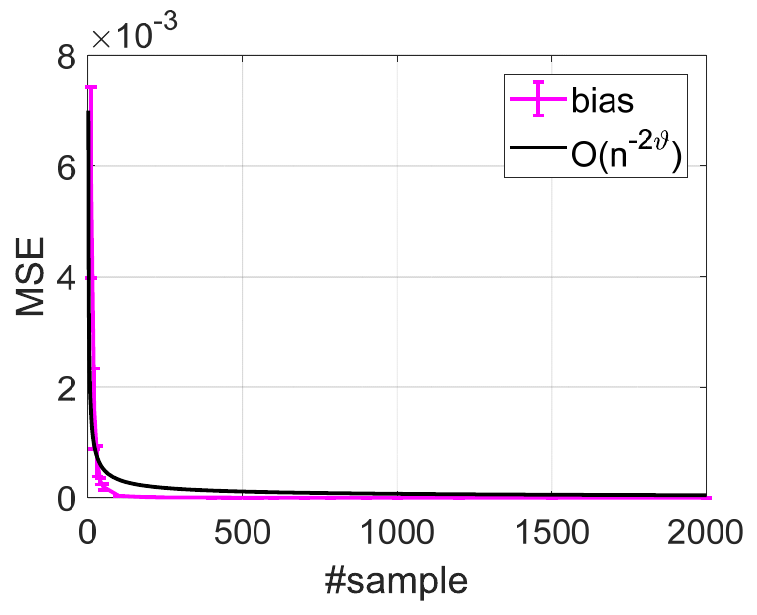}}
	\caption{Harmonic decay of $\widetilde{\bm X}$ with polynomial kernel: MSE of the expected excess risk, the variance in Eq.~\eqref{defvariance}, our derived ${\tt V_1}$, the bias in Eq.~\eqref{biaslemma1}, and our derived convergence rate $\mathcal{O}(n^{-2\vartheta r})$ with $r=1$ for different $\vartheta$.}\label{fig-polykernel} 
\end{figure*}

\begin{figure*}[!htb]
	\centering
	\subfigure[$\vartheta=2/3$]{
		\includegraphics[width=0.23\textwidth]{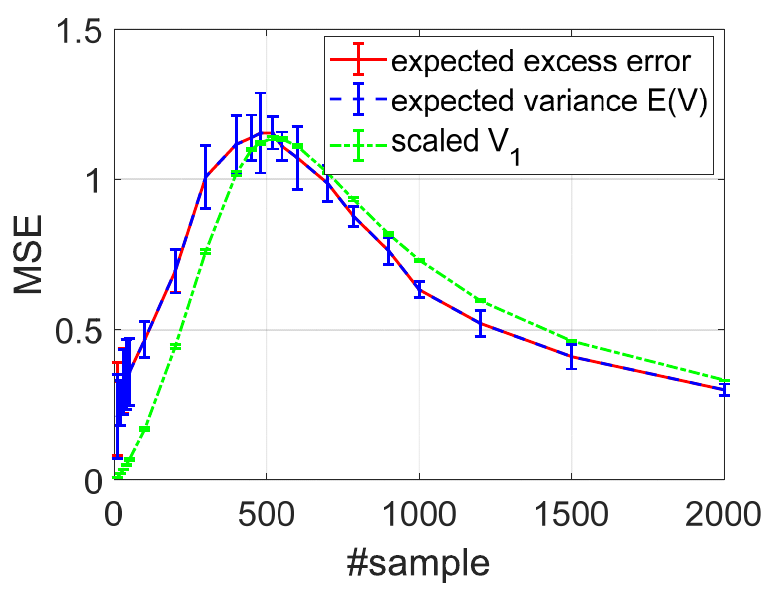}}
	\subfigure[$\vartheta=2/3$]{
		\includegraphics[width=0.23\textwidth]{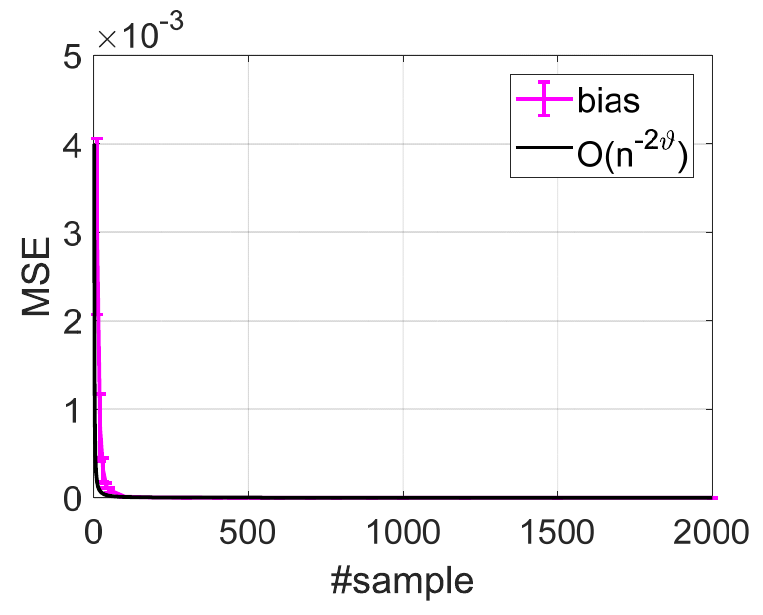}}
	\subfigure[$\vartheta=1/3$]{
		\includegraphics[width=0.23\textwidth]{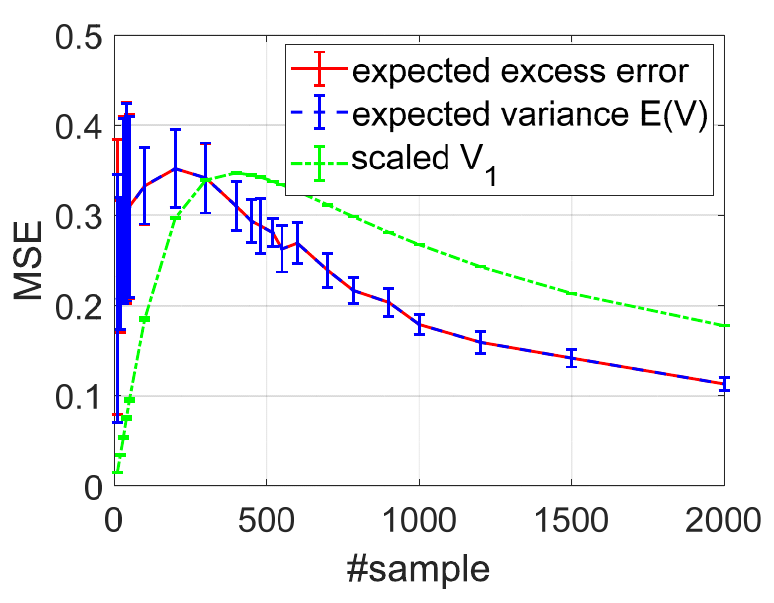}}
	\subfigure[$\vartheta=1/3$]{
		\includegraphics[width=0.23\textwidth]{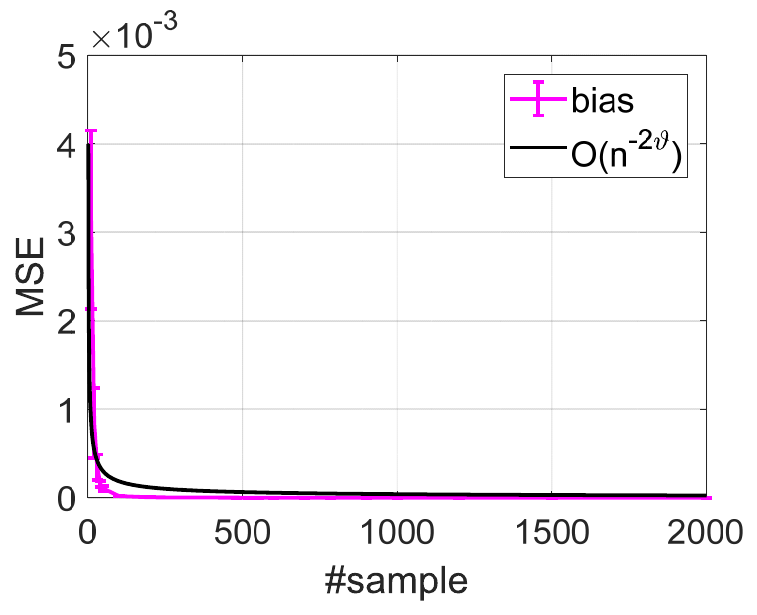}}
	\caption{Harmonic decay of $\widetilde{\bm X}$ in the Gaussian kernel case. The legend is the same as Figure~\ref{fig-polykernel}.}\label{fig-gaussiankernel}
	\vspace{-0.2cm}
\end{figure*}

\vspace{-0.1cm}
\section{Numerical Results}
\label{sec:exp}
\vspace{-0.1cm}
In this section, experiments are conducted to validate our theoretical results\footnote{The source code of our implementation can be found in \url{http://www.lfhsgre.org}.}.
Polynomial kernel of degree $3$ and Gaussian kernel are evaluated on 1) a synthetic dataset that satisfies our technical assumptions and 2) a subset of the \emph{YearPredictionMSD} dataset \cite{chang2008libsvm} with 1,000 data samples and $d=90$, to study our derived error bounds for the bias and variance. More experimental results can be found in Appendix~\ref{sec:appexp}.

{\bf Eigenvalue decay equivalence:} 
Here we study the eigenvalue decay of the original polynomial/Gaussian kernel matrices and their linearization $\bm X \bm X^{\!\top}/d$ on the subset of \emph{YearPredictionMSD} dataset. 
Note that, polynomial kernels $k(\bm x, \bm x') := \left(1 +  \langle \bm x, \bm x' \rangle/d \right)^p$ admit $\beta:= p$ independent of $\bm \Sigma_d$ (see in Table~\ref{tabkernels}), so we use the linearization $\beta \bm X \bm X^{\!\top}/d$ for this kernel.
Results in Figure~\ref{fig-year-eig} demonstrate that, the original nonlinear kernels admit the same eigenvalue decay as $\bm X \bm X^{\!\top}/d$. 
More experimental results on various dataset can be found in Appendix~\ref{sec:appexpeig}.

{\bf Risk curves on synthetic dataset:} 
To quantitatively assess our derived error bounds for the bias and variance, we generate a synthetic dataset under a known $f_{\rho}$, with \emph{harmonic decay} for the data as an illustrating example. More experimental results on different eigenvalue decays refer to Appendix~\ref{sec:appexpsyn}.
To be specific, we assume $y_i = f_{\rho} (\bm x_i) + \varepsilon$ with target function $f_{\rho}(\bm x) = \sin (\| \bm x \|^2_2)$ and Gaussian noise $\varepsilon$ having zero-mean and unit-variance. The feature dimension $d$ is set to 500. The samples are generated from $\bm x_i = \bm \Sigma_d^{1/2} \bm t_i$ (and thus $\bm X^{\!\top} \bm X = \bm T^{\!\top} \bm \Sigma_d \bm T$ with $\bm T =[\bm t_1, \bm t_2, \cdots, \bm t_n]^{\!\top} $) by the following steps:
(i) take $\bm \Sigma_d$ as a diagonal matrix with its diagonal entries following with \emph{harmonic decay}, i.e., $(\bm \Sigma_d)_{ii} \propto n/i$.
(ii) take $\bm T$ as a random orthogonal matrix\footnote{We generate a random Gaussian matrix and use the QR decomposition to obtain an orthogonal matrix \cite{Yu2016Orthogonal}.} such that $\bm T^{\!\top} \bm \Sigma_d \bm T$ also has a harmonic eigen-decay with $\bm T$ having almost i.i.d entries.

Accordingly, the above generation process satisfies Assumption~\ref{assum8m}, and also 
$\bm X \bm X^{\!\top}/d$ admits the same eigenvalue decay as $\bm \Sigma_d$, which can be used to validate our discussion in Section~\ref{sec:discusserr}.
In this setting, the expected excess risk, the bias, and the variance can be directly computed to validate our derived error bounds.
The experimental results are validated across 10 trials.
Specifically, to disentangle the \emph{implicit regularization} effect of KRR on the final result, we apply the linearization of the polynomial/Gaussian kernel by setting $\gamma = 0$ in Eq.~\eqref{eqlinear}.
In this case, the explicit $\lambda := \bar{c} n^{-\vartheta}$ is the only regularization in KRR.
In our model, $\bar{c}$ is empirically set to $0.01$ to avoid a large $\lambda$ when $n$ is small.

Figures~\ref{fig-polykernel} and~\ref{fig-gaussiankernel} show results under the \emph{harmonic decay} setting for the linearization of the polynomial/Gaussian kernel, respectively. 
We observe that:
1) our error bound ${\tt{V}}_1 \asymp \frac{1}{d} \mathcal{N}^{n\lambda}_{\widetilde{\bm X}}$ exhibits the same trend as the true variance;
2) in this case, the variance dominates and we thus obtain a bell-shaped risk curve that first increases and then decreases;
3) as $\vartheta$ decreases, $\lambda$ increases and the peak point of the variance occurs at smaller and smaller $n$;
4) the bias monotonically decreases with $n$, which corresponds to our error bound for the bias at a certain $\mathcal{O}(n^{-2\vartheta r})$ rate in Theorem~\ref{promain} by taking $r=1$ as the used $f_{\rho}$ is smooth enough to achieve a good approximation error;
5) in our high-dimensional regimes, different kernels lead to the same convergence rates of the bias, which verifies our results but is different from those in classical learning theory.

\begin{figure}[t]
	\centering
	\subfigure[\emph{YearPredictionMSD}]{\label{fig-dd-year}
		\includegraphics[width=0.33\textwidth]{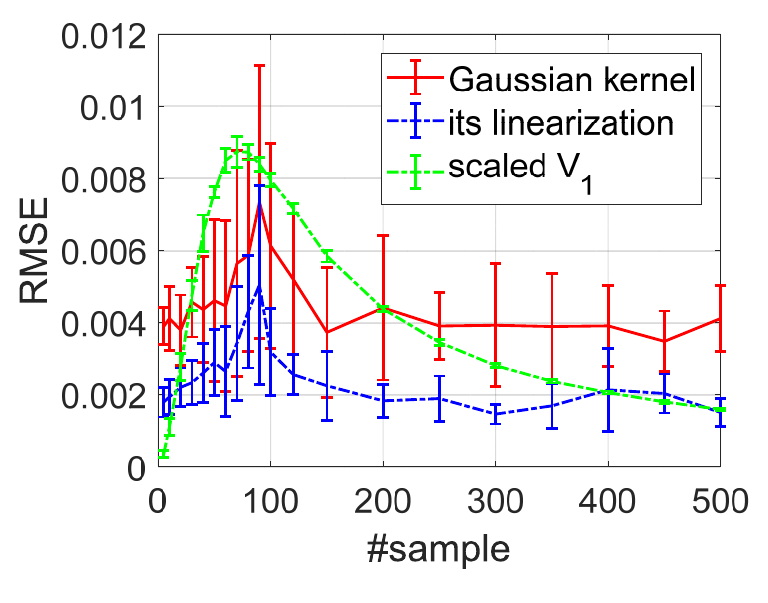}}
	\subfigure[\emph{MNIST} (digits 3 vs. 7)]{\label{fig-dd-mnist}
		\includegraphics[width=0.33\textwidth]{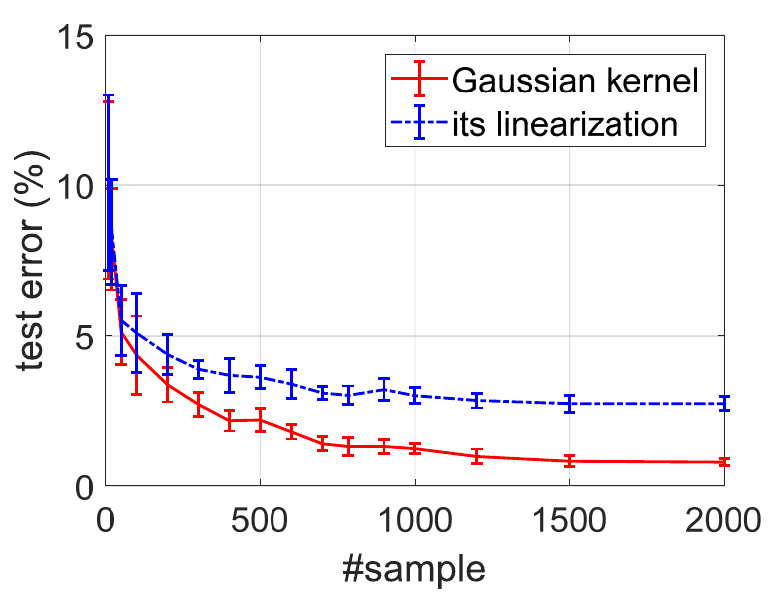}}
	\caption{The test performance of the kernel interpolation estimator and its linearization one.}\label{fig-dd}
\end{figure}

{\bf Risk curves on the real-world datasets:} Figure~\ref{fig-dd-year} shows the relative mean squared error (RMSE) of kernel ridgeless regression and its linearization in Eq.~\eqref{eqlinear} on a subset (1,000 examples) of the \emph{YearPredictionMSD} dataset averaged on 10 trials.
Figure~\ref{fig-dd-mnist} shows the classification accuracy of such two methods on the \emph{MNIST} dataset \cite{L1998Gradient}.
To evaluate the effectiveness of our error bounds, we plot the re-scaled ${\tt{V}}_1 \asymp \frac{1}{d} \mathcal{N}^{\gamma}_{\widetilde{\bm X}}$ with $\lambda = 0$.
It can be found that, kernel interpolation estimator generalizes well due to the \emph{implicit regularization}, i.e., $\gamma \neq 0$, which also exhibits a bell-shaped risk curve as our theoretical results suggest.
However, in Figure~\ref{fig-dd-mnist}, the risk curve monotonically decreases with $n$ on the \emph{MNIST} dataset \cite{L1998Gradient}, and at the same time kernel interpolation estimator and its linearization appear to generalize well.
This observation may due to the \emph{implicit regularization} parameter $\gamma$ in Eq.~\eqref{eqlinear} (of $10^{-3}$ order on this dataset) that plays a fundamental role of ``self-regularization''.  
Accordingly, the proposed analysis provides access to the high-dimensional classification problem that may establish more involved behavior than double descent, despite a clear mismatch between real-world data and the technical Assumption~\ref{assum8m}, thereby conveying a strong practical motivation for the present analysis.

\section{Conclusion}
\label{sec:conclusion}

We derived non-asymptotic expressions for the expected excess risk of kernel ridge regression estimators in the under- and over-determined regimes.
The used linearization technique of nonlinear smooth kernel allows us to discuss the impact of \emph{implicit} and \emph{explicit} regularization in a systematic manner.
Our refined analysis demonstrates that the monotonic bias and unimodal variance are able to exhibit various trends of risk curves. 
Since it is enough to require that the kernel function is differentiable in a neighborhood, our results further extend to the case of Laplace kernels \cite{rakhlin2019consistency}.

\section*{Acknowledgements}
The research leading to these results has received funding from the European Research Council under the European Union's Horizon 2020 research and innovation program / ERC Advanced Grant E-DUALITY (787960). This paper reflects only the authors' views and the Union is not liable for any use that may be made of the contained information.
This work was supported in part by Research Council KU Leuven: Optimization frameworks for deep kernel machines C14/18/068; Flemish Government: FWO projects: GOA4917N (Deep Restricted Kernel Machines: Methods and Foundations), PhD/Postdoc grant. This research received funding from the Flemish Government (AI Research Program). 
This work was supported in part by Ford KU Leuven Research Alliance Project KUL0076 (Stability analysis and performance improvement of deep reinforcement learning algorithms), EU H2020 ICT-48 Network TAILOR (Foundations of Trustworthy AI - Integrating Reasoning, Learning and Optimization), Leuven.AI Institute.


\newcommand{\etalchar}[1]{$^{#1}$}
\providecommand{\bysame}{\leavevmode\hbox to3em{\hrulefill}\thinspace}
\providecommand{\MR}{\relax\ifhmode\unskip\space\fi MR }
\providecommand{\MRhref}[2]{%
	\href{http://www.ams.org/mathscinet-getitem?mr=#1}{#2}
}
\providecommand{\href}[2]{#2}

\clearpage

\appendix
\onecolumn

Appendix is organized as follows.
\begin{itemize}
	\item Section~\ref{sec:examples} provides high dimensional linearizations of some typical smooth kernels as concrete examples of Table~\ref{tabparam}.
	\item In Section~\ref{sec:proofeig}, we demonstrate that, a kernel matrix in high dimensions admits the same eigenvalue decay as $\widetilde{\bm X}$ and ${\bm X \bm X^{\!\top}}/{d}$.
	\item Our proof framework includes the error decomposition in Section~\ref{sec:prooflemma1}, the error bound for the bias in Section~\ref{sec:proofbiass} and for the variance in Section~\ref{app:provar}, respectively.
	\item Section~\ref{sec:proofdecay} discusses the quantity function $\mathcal{N}^{n\lambda + \gamma}_{\widetilde{\bm X}} $ based on three eigenvalue decays: \emph{harmonic decay}, \emph{polynomial decay}, and \emph{exponential decay} in the $n < d$ and $n > d$ regimes.
	\item Some additional experiments are presented in Section~\ref{sec:appexp} to further validate our theoretical results.
\end{itemize} 

\section{Examples of kernels and their linearizations}
\label{sec:examples}
In this section, we present linearization of some typical kernels by Eq.~\eqref{eqlinear}.
Here we assume that $\alpha, \beta, \gamma \geq 0$ to ensure the positive definiteness of the approximated kernel matrix $\widetilde{\bm K}^{\text{lin}}$.
Table~\ref{tabkernels} reports the results of three inner-product kernels including polynomial kernel, linear kernel, exponential kernel; as well as a radial kernel: the common-used Gaussian kernel.
We can find that $\alpha, \gamma \geq 0$. Specifically, $\beta > 0$ avoids a trivial solution.
\begin{table*}[h]
	\centering
	\fontsize{7}{8}\selectfont
	\begin{threeparttable}
		\caption{Linearizations of typical kernels in high dimensions.}
		\label{tabkernels}
		\begin{tabular}{cccccccccccccccccccc}
			\toprule
			kernel &formulation & $\alpha$  & $\beta$ & $\gamma$  \cr
			\midrule
			polynomial kernels  &$k(\bm x, \bm x') := \left(1 + \frac{1}{d} \langle \bm x, \bm x' \rangle \right)^p$ & $1 + p(p-1)\frac{\operatorname{tr}\left(\bm \Sigma_{d}^{2}\right)}{2d^{2}}$ & $p$ & $(1+
			\tau)^p - 1 - p\tau$ \cr
			\midrule
			linear kernel  &$k(\bm x, \bm x') = \frac{1}{d} \langle \bm x, \bm x' \rangle$ & $0$ & $1$ & $0$ \cr
			\midrule
			exponential kernel  &$k(\bm x, \bm x') = \exp(\frac{2}{d} \langle \bm x, \bm x' \rangle)$ & $ 1 + 2\frac{\operatorname{tr}\left(\bm \Sigma_{d}^{2}\right)}{d^{2}} $ & $2$ & $\exp(2\tau) - 1 - 2\tau$ \cr
			\midrule
			Gaussian kernel & $k(\bm x, \bm x') = \exp\left( -\frac{1}{d} \| \bm x - \bm x' \|_2^2 \right)$ & $\exp(-2\tau)\left[1+2\frac{\operatorname{tr}\left(\bm \Sigma_{d}^{2}\right)}{d^{2}} \right]$ & $2\exp(-2\tau)$ & $1-2\tau \exp(-2\tau) - \exp(-2\tau)$ \cr
			\bottomrule
		\end{tabular}
	\end{threeparttable}\vspace{-0.3cm}
\end{table*}

\section{Eigenvalue decay equivalence}
\label{sec:proofeig}

In this section, we demonstrate that, in high dimensions, a kernel matrix induced by inner-product kernels or radial kernels admits the same eigenvalue decay as $\widetilde{\bm X}= \beta {\bm X \bm X^{\!\top}}/{d}+\alpha \bm 1\bm 1^{\!\top}$ and ${\bm X \bm X^{\!\top}}/{d}$.

For notational simplicity, denote the inner-product kernel matrix $\bm K_{\operatorname{inner}}$ and its linearization $\widetilde{{\bm K}^{\operatorname{lin}}_{\operatorname{inner}}}$; the radial kernel matrix $\bm K_{\operatorname{radial}}$ and its linearization $\widetilde{{\bm K}^{\operatorname{lin}}_{\operatorname{radial}}}$.
\begin{proposition}\label{sec:proeig}
	The inner-product kernel matrix $\bm K_{\operatorname{inner}}$ admits the same eigenvalue decay as $\widetilde{\bm X}$ and ${\bm X \bm X^{\!\top}}/{d}$.
\end{proposition}
\begin{proof}
	According to Theorem~2.1 in \cite{el2010spectrum}, the inner-product kernel matrix $\bm K_{\operatorname{inner}}$ can be well approximated by $\widetilde{{\bm K}^{\operatorname{lin}}_{\operatorname{inner}}}$ with 
	\begin{equation*}
	\widetilde{{\bm K}^{\operatorname{lin}}_{\operatorname{inner}}} := \beta \frac{\bm X \bm X^{\!\top}}{d} + \gamma \bm I + \alpha \bm 1\bm 1^{\!\top} \,,
	\end{equation*}
	in a spectral norm sense, where $\alpha$, $\beta$, $\gamma$ are given in Table~\ref{tabparam}.
	As a result, with high probability, the inner-product kernel matrix $\bm K_{\operatorname{inner}}$ and its linearization $\widetilde{{\bm K}^{\operatorname{lin}}_{\operatorname{inner}}}$ has the same eigenvalue.
	That means, $\bm K_{\operatorname{inner}}$ admits the same eigenvalue decay as $\widetilde{\bm X} := \beta {\bm X \bm X^{\!\top}}/{d} + \alpha \bm 1\bm 1^{\!\top} $ via a constant shift $\gamma$.
	
	Next, we shall demonstrate that $\bm K_{\operatorname{inner}}$ admits the same eigenvalue decay as $\bm X \bm X^{\!\top}/d$.
	Since $\bm 1\bm 1^{\!\top}$ is a rank-one matrix with $\lambda_1 (\bm 1\bm 1^{\!\top}) = n$, with Weyl's inequality and $\lambda_n \le \lambda_{n-1} \le \ldots \leq \lambda_1$, we have
	\begin{equation*}
	\beta \lambda_1 \left(\frac{\bm X \bm X^{\!\top}}{d}\right) + \gamma \leq \lambda_1(\widetilde{{\bm K}^{\operatorname{lin}}_{\operatorname{inner}}}) \leq \beta \lambda_1 \left(\frac{\bm X \bm X^{\!\top}}{d}\right) + \gamma + \alpha n,
	\end{equation*}
	and
	\begin{equation*}
	\beta \lambda_i \left(\frac{\bm X \bm X^{\!\top}}{d}\right) + \gamma \leq \lambda_i(\widetilde{{\bm K}^{\operatorname{lin}}_{\operatorname{inner}}}) \leq \beta \lambda_{i-1} \left(\frac{\bm X \bm X^{\!\top}}{d}\right) + \gamma\,, \quad i = 2,3, \ldots n \,,
	\end{equation*}
	so that the eigenvalue of $\widetilde{{\bm K}^{\operatorname{lin}}_{\operatorname{inner}}}$ \emph{interlaced} with those of $\beta {\bm X \bm X^{\!\top}}/{d} + \gamma \bm I$. We can thus conclude that the eigenvalue decay of $\widetilde{{\bm K}^{\operatorname{lin}}_{\operatorname{inner}}}$ is the same as that of ${\bm X \bm X^{\!\top}}/{d}$ with a constant shift and scaling, which do not effect the trend of eigenvalue decay.
	Accordingly, the inner-product-type kernel matrix ${\bm K}_{\operatorname{inner}}$ and its linearization $\widetilde{{\bm K}^{\operatorname{lin}}_{\operatorname{inner}}}$, $\widetilde{\bm X}$ admit the same eigenvalue decay as ${\bm X \bm X^{\!\top}}/{d}$, which concludes the proof.
\end{proof}
Proposition~\ref{sec:proeig} also provides a justification to study the eigenvalue decay of a radial kernel matrix.
According to Theorem~2.2 in \cite{el2010spectrum}, the radial kernel matrix $\bm K_{\operatorname{radial}}$ can be well approximated by $\widetilde{{\bm K}_{\operatorname{radial}}^{\operatorname{lin}}}$ with 
\begin{equation*}
\widetilde{{\bm K}_{\operatorname{radial}}^{\operatorname{lin}}} := \beta \frac{\bm X \bm X^{\!\top}}{d} + \gamma \bm I + \alpha \bm 1\bm 1^{\!\top} + h^{\prime}(2\tau) \bm A + \frac{1}{2}h^{\prime \prime}(2\tau) \bm A \odot \bm A \,,
\end{equation*}
in a spectral norm sense, where $\alpha$, $\beta$, $\gamma$ are given in Table~\ref{tabparam}.
Recall $\bm A := \bm 1 \bm \psi^{\!\top} + \bm \psi \bm 1^{\!\top}$, where $\bm \psi \in \mathbb{R}^n$ with $\psi_i := \| \bm x_i \|^2_2/d - \tau$, we find that $\bm A$ is a rank 2 matrix with its eigenvalues $\lambda(\bm A) = \bm 1^{\!\top} \bm \psi \pm \sqrt{n} \| \bm \psi \|_2$, and thus we have $\operatorname{rank}( \bm A \odot \bm A ) = 3$.\footnote{This can be proved using rank-one decomposition of $\bm A$.}
Hence, by virtue of Proposition~\ref{sec:proeig}, apart from the top 5 eigenvalues of the radial kernel matrix ${\bm K}_{\operatorname{radial}}$, its remaining eigenvalues follow with
\begin{equation*}\label{eigen2}
\beta \lambda_i \left(\frac{\bm X \bm X^{\!\top}}{d}\right) + \gamma \leq \lambda_i(\widetilde{{\bm K}_{\operatorname{radial}}^{\operatorname{lin}}}) \leq \beta \lambda_{i-1} \left(\frac{\bm X \bm X^{\!\top}}{d}\right) + \gamma\,, \quad i = 6,7, \ldots n \,.
\end{equation*}
Accordingly, ${\bm K}_{\operatorname{radial}}$ admits the same eigenvalue decay as ${\bm X \bm X^{\!\top}}/{d}$.

\section{Proof of Lemma~\ref{lemerr}}
\label{sec:prooflemma1}
\begin{proof}
	By virtue of the closed form of the KRR estimator in Eq.~\eqref{krrclo} and $\bm \epsilon:= \bm y - f_{\rho}(\bm X)$, we have
	\begin{equation*}
	f_{\bm{z},\lambda}(\bm x)  - f_{\rho}(\bm x) = k(\bm x, \bm X)^{\!\top} (\bm K + n\lambda \bm I)^{-1} \bm \epsilon + k(\bm x, \bm X)^{\!\top}(\bm K + n\lambda \bm I)^{-1} f_{\rho}(\bm X) - f_{\rho}(\bm x)\,,
	\end{equation*}  
	where $f_{\rho}(\bm X) = [f_{\rho}(\bm x_1), f_{\rho}(\bm x_1), \cdots, f_{\rho}(\bm x_n)]^{\!\top} \in \mathbb{R}^n$.
	According to $\mathbb{E}_{y \mid \bm x} [\bm \epsilon] = 0$, we then have
	\begin{equation*}
	\begin{split}
	\mathbb{E}_{y \mid \bm x}\big\| f_{\bm{z},\lambda}  - f_{\rho} \big\|^2_{\mathcal{L}_{\rho_X}^{2}} & =  \mathbb{E}_{\bm x} \big\| k(\bm x, \cdot)^{\!\top} (\bm K + n\lambda \bm I)^{-1} f_{\rho}(\bm X) - f_{\rho} \big\|^2_{\mathcal{L}_{\rho_X}^{2}} + \mathbb{E}_{y,\bm x} \big\| k(\bm x, \cdot)^{\!\top} (\bm K + n\lambda \bm I)^{-1} \bm \epsilon \big\|^2_{\mathcal{L}_{\rho_X}^{2}}\,.
	\end{split}
	\end{equation*}
	Based on the definition of ${\tt{B}}$, we decompose ${\tt{B}}$ as
	\begin{equation*}
	\begin{split}
	{\tt{B}} &:= \mathbb{E}_{\bm x} \big\| k(\bm x, \cdot)^{\!\top} (\bm K + n\lambda \bm I)^{-1} f_{\rho}(\bm X) - f_{\rho} \big\|^2_{\mathcal{L}_{\rho_X}^{2}} =  \| f_{\bm X, \lambda} - f_{\rho} \|^2_{\mathcal{L}^2_{\rho_X}} \\
	& \leq 2 \| f_{\bm X, \lambda} - f_{\lambda} \|^2_{\mathcal{L}^2_{\rho_X}} + 2\| f_{\lambda} - f_{\rho} \|^2_{\mathcal{L}^2_{\rho_X}} \,,
	\end{split}
	\end{equation*}
	which concludes our proof.
\end{proof}

\section{Proof for the bias}
\label{sec:proofbiass}

The error bound for the bias is given by the following theorem.
\begin{theorem} (Bias)\label{thm:biass}
	Under Assumption~\ref{sourcecon} (source condition with $0 < r \leq 1$), Assumption~\ref{effectass} (capacity condition with $0\leq \eta \leq 1$), let $0 < \delta <1/2$, taking the regularization parameter $\lambda := \bar{c}n^{-\vartheta}$ with $0 \leq \vartheta \leq \frac{1}{1+\eta}$, there holds with probability at least $1 - 2\delta$, we have
	\begin{equation*}
	{\tt{B}} \leq 2\left( \| f_{\bm X, \lambda} - f_{\lambda} \|^2_{\mathcal{L}^2_{\rho_X}} + \| f_{\lambda} - f_{\rho} \|^2_{\mathcal{L}^2_{\rho_X}} \right)  \lesssim n^{-2\vartheta r} \log^4 \big(\frac{2}{\delta}\big)\,.
	\end{equation*}
	
\end{theorem}
In our error decomposition, $\| f_{\lambda} - f_{\rho} \|^2_{\mathcal{L}^2_{\rho_X}}$ is independent of data ${\bm X}$ that corresponds to the approximation error in learning theory \cite{cucker2007learning}; while the first term $\| f_{\bm X, \lambda} - f_{\lambda} \|^2_{\mathcal{L}^2_{\rho_X}}$ depends on ${\bm X}$, termed as bias-sample error.
To prove Theorem~\ref{thm:biass}, we need to bound the approximation error and the bias-sample error as follows.

\subsection{Bound approximation error}
In learning theory, the approximation error $\| f_{\lambda} - f_{\rho} \|_{\mathcal{L}^2_{\rho_X}}$ can be estimated by the source condition in Assumption~\ref{sourcecon}.
\begin{lemma} (Lemma 3 in \cite{smale2007learning})
	Under the source condition in Assumption~\ref{sourcecon} with $0 < r \leq 1$, the approximation error can be given by
	\begin{equation*}
	\| f_{\lambda} - f_{\rho} \|_{\mathcal{L}^2_{\rho_X}} = \| (L_K + \lambda I)^{-1} L_K f_{\rho} - f_{\rho} \|_{\mathcal{L}^2_{\rho_X}} \leq \lambda^{r} \| L_{K}^{-r} f_{\rho} \|_{\mathcal{L}^2_{\rho_X}} \leq R \lambda^{r} \,.
	\end{equation*}
\end{lemma}

\subsection{Bound bias-sample error}
To bound the bias-sample error $\| f_{\bm X, \lambda} - f_{\lambda} \|_{\mathcal{L}^2_{\rho_X}}$, we need the following lemma.

\begin{lemma}\label{LKfirst} (Lemma 17 in \cite{lin2017distributed})
	For any $0 < \delta < 1$, it holds with probability at least $1 - \delta$ that
	\begin{equation*}
	\| (L_{K} + \lambda I)^{-1/2} (L_K - L_{K, \bm X}) \| \leq \frac{2 \kappa}{\sqrt{n}} \left\{ \frac{\kappa}{\sqrt{n \lambda}} + \sqrt{\mathcal{N}(\lambda)} \right\} \log \left(\frac{2}{\delta} \right)\,,
	\end{equation*}
	where $\kappa:= \max\{ 1, \sup_{\bm x \in X} \sqrt{k(\bm x, \bm x)} \}$.
\end{lemma}

Then the bias-sample error can be decomposed into several parts.
\begin{lemma}\label{decombias}
	Under Assumption~\ref{sourcecon}, we have 
	\begin{equation*}
	\begin{split}
	\| f_{\bm X, \lambda} - f_{\lambda} \| &\leq R \lambda^{1/2} \| (L_{K,\bm X}+ \lambda I)^{-1/2} (L_{K}+ \lambda I)^{1/2} \|
	\| (L_{K}+ \lambda I)^{-1/2}(L_K - L_{K,\bm X}) \|^r \\
	&\qquad \| (L_{K}+ \lambda I)^{-1/2}(L_K - L_{K,\bm X}) (L_{K}+ \lambda I)^{-1} \|^{1-r} \,.
	\end{split}
	\end{equation*}
\end{lemma}
\begin{proof}[Proof of Lemma~\ref{decombias}]
	According to the definition of $f_{\bm X, \lambda}$ and $f_{\lambda}$, we have
	\begin{equation*}
	f_{\bm X, \lambda} - f_{\lambda} = (L_{K,\bm X} + \lambda I)^{-1} L_{K,\bm X} f_{\rho} - (L_K + \lambda I)^{-1} L_K f_{\rho}\,.
	\end{equation*}
	Due to $(A+\lambda I)^{-1}A = I - \lambda (A+ \lambda I)^{-1}$ for any bounded positive operator $A$, we have
	\begin{equation*}
	(L_K + \lambda I)^{-1} L_K f_{\rho} - (L_{K,\bm X} + \lambda I)^{-1} L_{K,\bm X} f_{\rho}  = \lambda \left[ (L_{K,\bm X} + \lambda I)^{-1} - (L_{K} + \lambda I)^{-1} \right] f_{\rho}\,.
	\end{equation*}
	Further, by virtue of the first order decomposition of operator difference:  $A^{-1} - B^{-1} = A^{-1}(B-A)B^{-1}$ for any invertible bounded operator and using the source condition in Assumption~\ref{sourcecon}, the above equation can be further expressed as
	\begin{equation*}
	\begin{split}
	&	(L_K + \lambda I)^{-1} L_K f_{\rho} - (L_{K,\bm X} + \lambda I)^{-1} L_{K,\bm X} f_{\rho} = \lambda (L_{K,\bm X} + \lambda I)^{-1} (L_{K} - L_{K,\bm X}) (L_{K}+ \lambda I)^{-1} L_{K}^r g_{\rho}\\
	& \qquad = \lambda^{1/2} \left( \lambda^{1/2} (L_{K,\bm X}+ \lambda I)^{-1/2} \right) \left( (L_{K,\bm X}+ \lambda I)^{-1/2} (L_{K}+ \lambda I)^{1/2} \right) \\
	& \qquad \left( (L_{K} + \lambda I)^{-1/2} (L_K - L_{K, \bm X}) (L_{K} + \lambda I)^{-(1-r)} \right) \left( (L_{K} + \lambda I)^{-r} L_K^r \right)  g_{\rho}\,.
	\end{split}
	\end{equation*} 
	Besides, using $\| A B ^{t} \| \leq \| A \|^{1-t} \| A B \|^{t}$ with $t \in [0,1]$ for any bounded linear operator $A$ and positive semi-definite operator $B$ in Proposition 9 in \cite{Rudi2017Generalization}, we have
	\begin{equation*}
	\begin{split}
	\| (L_{K} + \lambda I)^{-1/2} (L_K - L_{K, \bm X}) (L_{K} + \lambda I)^{-(1-r)} \| & \leq \| (L_{K} + \lambda I)^{-1/2} (L_K - L_{K, \bm X}) \|^r \\
	&  \| (L_{K} + \lambda I)^{-1/2} (L_K - L_{K, \bm X}) (L_{K} + \lambda I)^{-1} \|^{1-r}\,,
	\end{split}
	\end{equation*}
	where we choose $A:= (L_{K} + \lambda I)^{-1/2} (L_K - L_{K, \bm X})$, $B:= (L_{K} + \lambda I)^{-1}$, and $t:= 1 - r \in [0,1)$.
	Accordingly, we can conclude our proof due to $\| (L_{K, \bm X} + \lambda I)^{-1/2} \| \leq 1/\sqrt{\lambda} $ and $\| (L_{K} + \lambda I)^{-r} L_K^r \| \leq 1 $.
\end{proof}
{\bf Remark:} The proof framework of Lemma~\ref{decombias} is similar to Lemma 4 in \cite{Rudi2017Generalization} but we consider a more general case $0 < r \leq 1$ than $1/2 \leq r \leq 1$ in \cite{Rudi2017Generalization}.
Although $0 < r < 1/2$ appears to be unattainable as claimed in \cite{Rudi2017Generalization}, we follow with \cite{lin2017distributed,guo2017learning} on a quite general case with $r>0$.

To prove Theorem~\ref{thm:biass}, we also need the following two lemmas.

\begin{lemma}\label{lemmaterm2} (Proposition 6 in \cite{Rudi2017Generalization})
	Let $\delta \in (0,1/2]$, it holds with probability at least $1 - 2\delta$ that
	\begin{equation*}
	\begin{split}
	&\| (L_{K}+ \lambda I)^{-1/2}(L_K - L_{K,\bm X}) (L_{K}+ \lambda I)^{-1} \| \\
	& \leq  \| (L_{K}+ \lambda I)^{-1/2}(L_K - L_{K,\bm X}) (L_{K}+ \lambda I)^{-1/2} \|  \| (L_{K}+ \lambda I)^{-1/2} \| 
	\leq \left( \frac{\kappa^2}{3n \lambda} + \sqrt{\frac{\kappa^2}{n \lambda}} \right) \frac{1}{\sqrt{\lambda}} \,.
	\end{split}
	\end{equation*}
\end{lemma}

\begin{lemma}\label{lemmaterm1}
	For any $0 < \delta < 1$, with probability at least $1 - \delta$, we have
	\begin{equation*}
	\begin{split}
	&	\| (L_{K,\bm X}+ \lambda I)^{-1/2} (L_{K}+ \lambda I)^{1/2} \|  \leq   1+\frac{2 \kappa}{\sqrt{n \lambda}} \left\{ \frac{\kappa}{\sqrt{n \lambda}} + \sqrt{\mathcal{N}(\lambda)} \right\} \log \left(\frac{2}{\delta}\right) \,.
	\end{split}
	\end{equation*}
\end{lemma}
\begin{proof}[Proof of Lemma~\ref{lemmaterm1}]
	By virtue of a second order decomposition of operator difference in Lemma 16 \cite{lin2017distributed}, we have
	\begin{equation*}
	A^{-1} - B^{-1} = B^{-1}(B-A) A^{-1}(B-A) B^{-1}+B^{-1}(B-A) B^{-1} \,,
	\end{equation*}
	which leads to 
	\begin{equation}\label{bainv}
	A^{-1}B = I + B^{-1}(B-A) + B^{-1}(B-A) A^{-1}(B-A) \,.
	\end{equation}
	Accordingly, denote $A:= L_{K,\bm X}+ \lambda I$ and $B:= L_{K}+ \lambda I$, we can derive that 
	\begin{equation*}
	\begin{split}
	&	\| (L_{K,\bm X}+ \lambda I)^{-1/2} (L_{K}+ \lambda I)^{1/2} \|  \leq \| (L_{K,\bm X}+ \lambda I)^{-1} (L_{K}+ \lambda I) \|^{1/2} \\
	& \leq \sqrt{1 + \lambda^{-1/2}\| (L_{K}+ \lambda I)^{-1/2} (L_K - L_{K, \bm X}) \| + \lambda^{-1} \| (L_{K}+ \lambda I)^{-1/2} (L_K - L_{K, \bm X}) \|^2   } \\
	& \leq \sqrt{ 1 + \mathcal{A} + \mathcal{A}^2 } \leq 1 + \mathcal{A}\,,
	\end{split}
	\end{equation*}
	where $\mathcal{A} := \frac{2 \kappa}{\sqrt{n \lambda}} \left\{ \frac{\kappa}{\sqrt{n \lambda}} + \sqrt{\mathcal{N}(\lambda)} \right\} \log(2/\delta)$ by Lemma~\ref{LKfirst}. The first inequality holds by $\| A^s B^s \| \leq \| A B \|^s$ with $0 \leq s \leq 1$ for positive operators $A$ and $B$ on Hilbert spaces \cite{blanchard2010optimal}. The second inequality can be derived by Eq.~\eqref{bainv}, $\| (L_{K,\bm X}+ \lambda I)^{-1} \| \leq 1/\lambda $ and $\| (L_{K}+ \lambda I)^{-1/2} \| \leq 1/\sqrt{\lambda} $.
\end{proof}
{\bf Remark:} Lemma~7.2 in \cite{rudi2013sample} gives $\| (L_{K,\bm X}+ \lambda I)^{-1/2} (L_{K}+ \lambda I)^{1/2} \| \leq \sqrt{2}$ by assuming $\lambda > \frac{9}{n} \log \frac{n}{\delta}$; whereas our result does not require extra conditions on $\lambda$.

Based on the above lemmas, we are ready to prove Theorem~\ref{thm:biass}.

\begin{proof}[Proof of Theorem~\ref{thm:biass}]
	We first estimate $\| (L_{K,\bm X}+ \lambda I)^{-1/2} (L_{K}+ \lambda I)^{1/2} \|$ in Lemma~\ref{lemmaterm1} by taking $\lambda := \bar{c}n^{-\vartheta}$ and the capacity condition in Assumption~\ref{effectass}: $\mathcal{N}(\lambda) \leq Q^2 \lambda^{-\eta}$ with $\eta \in [0,1]$. Accordingly, we have
	\begin{equation*}
	\begin{split}
	\| (L_{K,\bm X}+ \lambda I)^{-1/2} (L_{K}+ \lambda I)^{1/2} \| & \leq 1+\frac{2 \kappa}{\sqrt{n \lambda}} \left\{ \frac{\kappa}{\sqrt{n \lambda}} + \sqrt{\mathcal{N}(\lambda)} \right\} \log \left(\frac{2}{\delta}\right) \\
	& \leq 1+ \left( \frac{2 \kappa^2}{\bar{c}} n^{-(1-\vartheta)} + {2 \kappa}\bar{c}^{-(\frac{1}{2} + \frac{\eta}{2})} Q n^{-\frac{1 - \vartheta - \vartheta \eta}{2}} \right) \log \left(\frac{2}{\delta}\right) \\
	& \leq \left( 1+ \frac{2\kappa(\kappa + Q)}{\bar{c}}n^{-\frac{1 - \vartheta - \vartheta \eta}{2}} \right) \log \left(\frac{2}{\delta}\right)\,,
	\end{split}
	\end{equation*}
	where we use $\log^r(2/\delta) \leq \log(2/\delta)$ due to $\log(2/\delta) > 1$  in the last inequality. Since $\| (L_{K,\bm X}+ \lambda I)^{-1/2} (L_{K}+ \lambda I)^{1/2} \|$ converges to zero when $n$ is large enough, we require $\vartheta < \frac{1}{1+\eta}$ to ensure a positive convergence rate, which implies $\vartheta \leq 1$.
	Then we bound $\| (L_{K} + \lambda I)^{-1/2} (L_K - L_{K, \bm X}) \|^r$ by Lemma~\ref{LKfirst}.
	By virtue of $(a+b)^r \leq a^r + b^r$ for any $r \in (0,1]$ and $a,b \geq 0$,  we have
	\begin{equation*}
	\begin{split}
	\| (L_{K} + \lambda I)^{-1/2} (L_K - L_{K, \bm X}) \|^r & \leq \left( \frac{2 \kappa}{\sqrt{n}} \right)^r \left\{ \frac{\kappa}{(n \lambda)^{\frac{r}{2}}} + [\mathcal{N}(\lambda)]^{\frac{r}{2}} \right\} \log \left(\frac{2}{\delta} \right)\\
	& \leq (2 \kappa)^r (n\bar{c})^{-\frac{r}{2}} \left[ \kappa n^{-\frac{r(1-\vartheta)}{2}} + \frac{Q}{{\bar{c}}^{\frac{\eta r}{2}}} n^{\frac{\vartheta \eta r}{2}}  \right]  \log \left(\frac{2}{\delta} \right) \\
	& \leq \frac{2\kappa (Q + \kappa)}{\bar{c}} n^{-\frac{(1 - \vartheta \eta) r}{2}} \log \left(\frac{2}{\delta} \right)\,,
	\end{split}
	\end{equation*}
	where the second one admits by the capacity condition in Assumption~\ref{effectass}.
	Similarly, to bound $\| (L_{K}+ \lambda I)^{-1/2}(L_K - L_{K,\bm X}) (L_{K}+ \lambda I)^{-1} \|^{1-r}$ by Lemma~\ref{lemmaterm2}, we can derive that 
	\begin{equation*}
	\begin{split}
	\| (L_{K}+ \lambda I)^{-1/2}(L_K - L_{K,\bm X}) (L_{K}+ \lambda I)^{-1} \|^{1-r} & \leq \lambda^{-\frac{1-r}{2}} \left( {\kappa^2}{(n \lambda)^{-(1-r)}} + {\kappa}{(n \lambda)^{-\frac{1-r}{2}}} \right) \\
	& \leq \frac{\kappa^2}{\bar{c}} \left( n^{(\frac{3}{2}\vartheta - 1)(1 -r)} + n^{(\vartheta -\frac{1}{2})(1-r)} \right) \\
	& \leq \frac{\kappa^2}{\bar{c}} n^{(\vartheta -\frac{1}{2})(1-r)}\,.
	\end{split}
	\end{equation*}
	Combining the above three inequalities, we have  
	\begin{equation*}
	\begin{split}
	\| f_{\bm X, \lambda} - f_{\lambda} \| & \leq \frac{4R\kappa^3(Q+\kappa)^2}{\bar{c}^3} n^{-\frac{(1-\vartheta \eta)r+\vartheta}{2}} n^{(\vartheta -\frac{1}{2})(1-r)} \log^2 \left(\frac{2}{\delta}\right) \\
	& \leq \widetilde{C_{R, Q, \kappa, \bar{c}}} n^{-\frac{1- \vartheta(\eta r + 1 - 2r)}{2}} \log^2 \bigg(\frac{2}{\delta}\bigg)\,,
	\end{split}
	\end{equation*}	
	where $\widetilde{C_{R, Q, \kappa, \bar{c}}} := 4R\kappa^3(Q+\kappa)^2/\bar{c}^3$ is independent of $n$ and $d$.
	
	Finally, the bias can be bounded by
	\begin{equation*}
	\begin{split}
	{\tt{B}} 
	& \leq 2 \| f_{\bm X, \lambda} - f_{\lambda} \|^2_{\mathcal{L}^2_{\rho_X}} + 2\| f_{\lambda} - f_{\rho} \|^2_{\mathcal{L}^2_{\rho_X}} \\
	& \leq 2R^2 n^{-2\vartheta r} + \widetilde{C_{1}} n^{-\left[{1- \vartheta(\eta r + 1 - 2r)}\right]} \log^4 \bigg(\frac{2}{\delta}\bigg)\\
	& \leq \widetilde{C} n^{-2\vartheta r} \log^4 \bigg(\frac{2}{\delta}\bigg) \,,
	\end{split}
	\end{equation*}
	where the third inequality holds by $2 \vartheta r \leq 1- \vartheta(\eta r + 1 - 2r)$ due to $\vartheta \leq \frac{1}{1+ \eta}$, and $\widetilde{C}, \widetilde{C_1}$ are some constants independent of $n$ and $d$.
	Accordingly, we can conclude the proof.
\end{proof}

\section{Proof for the variance}
\label{app:provar}


Formally, we have the following theorem to bound the variance.
\begin{theorem}\label{provar} (Variance)
	Under Assumptions~\ref{assber},~\ref{assum8m}, then for $0 < \delta < 1$ with probability $1 - \delta - d^{-2}$, $\theta = \frac{1}{2} - \frac{2}{8+m}$, and $d$ large enough, for any given $\varepsilon > 0$, we have 
	\begin{equation*}
	{\tt{V}} \lesssim {\tt{V}}_1 + {\tt{V}}_2\,,
	\end{equation*}
	where ${\tt{V}}_1 := \frac{\sigma^2\beta}{d} \mathcal{N}^{n\lambda + \gamma}_{\widetilde{\bm X}}$ and ${\tt{V}}_2$ is the residual term with 
	\begin{equation*}
	{\tt{V}}_2 := \left\{
	\begin{array}{rcl}
	\begin{split}
	&\frac{\sigma^2 \log ^{2+4\varepsilon} d}{(n\lambda+\gamma)^2 d^{4 \theta-1}} ,~\mbox{inner-product kernels} \\
	& \frac{\sigma^2}{(n\lambda+\gamma)^2} d^{-2 \theta} \log ^{1+\varepsilon} d,~\mbox{radial kernels}\,.
	\end{split}
	\end{array} \right.
	\end{equation*}
\end{theorem}

For inner-product kernels, our proof framework follows \cite{liang2020just}, and is briefly discussed in Section~\ref{subsec:proof-inner-prod-V}.
Nevertheless, error bound on radial kernels has not been investigated in \cite{liang2020just} and is more subtle to handle (than that of inner-product kernels) due to the additionally introduced $\bm A$ and $\bm A \odot \bm A$ in Table~\ref{tabparam}. 
Accordingly, we mainly focus on proofs for radial kernels.

\subsection{Inner-product kernel matrices}\label{subsec:proof-inner-prod-V}
In this subsection, we consider the inner-product kernel case with $
k(\bm x, \bm x') = h\left( \langle \bm x, \bm x' \rangle / d \right)$.
We briefly introduce our results that can be derived from proofs of Theorem 2 in \cite{liang2020just} for completeness.

To prove Theorem~\ref{provar}, define 
\begin{equation}\label{klinmatrix}
\widetilde{\bm K^{\operatorname{lin}}}(\bm X, \bm X) :=(n \lambda + \gamma) \bm I+\alpha \bm 1\bm 1^{\!\top}+\beta \frac{\bm X \bm X^{\!\top}}{d} \in \mathbb{R}^{n \times n} \,,
\ k^{\operatorname{lin}}(\bm x, \bm X) := h(0) \bm 1 + \beta \frac{\bm X \bm x^{\!\top}}{d} \in \mathbb{R}^{n \times 1}\,,
\end{equation}
and $k^{\operatorname{lin}}(\bm X, \bm x)$ is the transpose of $k^{\operatorname{lin}}(\bm x, \bm X)$.
Note that $\gamma$ in $\widetilde{\bm K^{\operatorname{lin}}}$ corresponds to the \emph{implicit} regularization and $n \lambda$ corresponds to the \emph{explicit} regularization.
Now we prove Theorem~\ref{provar} for inner-product kernels.
\begin{proof}[Proof of Theorem~\ref{provar} for inner-product kernels]
	According to the definition of {\tt{V}}, we have
	\begin{equation}\label{varinitial}
	\begin{split}
	{\tt{V}} &= \mathbb{E}_{\bm x,y} \mathrm{tr}\left[ k(\bm x, \bm X)^{\!\top} (\bm K + n\lambda \bm I)^{-1} \bm \epsilon \bm \epsilon^{\!\top} (\bm K + n\lambda \bm I)^{-1} k(\bm x, \bm X) \right] 
	= \mathbb{E}_{\bm x}  \| (\bm K + n\lambda \bm I)^{-1} k(\bm x, \bm X) \|_2^2~ \mathbb{E}_{y | \bm x} \| \bm \epsilon \|_2^2  \\
	& \leq \sigma^2 \mathbb{E}_{\bm x} \| (\bm K + n\lambda \bm I)^{-1} k(\bm x, \bm X) \|_2^2 \\
	& \leq  \sigma^2 \| (\bm K + n\lambda \bm I)^{-1} \widetilde{\bm K^{\operatorname{lin}}}\|_2^2~ \mathbb{E}_{\bm x} \| [\widetilde{\bm K^{\operatorname{lin}}}]^{-1} k^{\operatorname{lin}}(\bm x, \bm X)   \|_2^2  + \sigma^2 \| (\bm K + n\lambda \bm I)^{-1} \|_2^2~ \mathbb{E}_{\bm x} \|  k(\bm x, \bm X) -  k^{\operatorname{lin}}(\bm x, \bm X) \|_2^2\,,
	\end{split}
	\end{equation}
	where the first inequality comes from Assumption~\ref{assber}.
	To bound the terms in Eq.~\eqref{varinitial}, we need
	\begin{equation}
	\begin{split}\label{klambdn}
	&\mathbb{E}_{\bm x} \| [\widetilde{\bm K^{\operatorname{lin}}}]^{-1} k^{\operatorname{lin}}(\bm x, \bm X)   \|_2^2
	=\mathbb{E}_{\bm x} \operatorname{tr}\left[ \left[\widetilde{\bm K^{\operatorname{lin}}} \right]^{-1} \Big(\beta \frac{\bm X \bm{x}}{d} + h(0) \bm 1 \Big) \Big( \beta \frac{\bm{x}^{\!\top} \bm X^{\!\top}}{d} + h(0) \bm 1^{\!\top} \Big) \left[\widetilde{\bm K^{\operatorname{lin}}}\right]^{-1}\right] \\
	&\leq \frac{1}{d}\left\|\bm \Sigma_{d}\right\|_2 \operatorname{tr}\left(\left[\widetilde{\bm K^{\operatorname{lin}}}\right]^{-1} \beta^{2} \frac{\bm X \bm X^{\!\top}}{d}\left[\widetilde{\bm K^{\operatorname{lin}}}\right]^{-1}\right) +  \frac{1}{d} \operatorname{tr}\left(\left[\widetilde{\bm K^{\operatorname{lin}}}\right]^{-1} h(0)^2 \bm 1 \bm 1^{\!\top} \left[\widetilde{\bm K^{\operatorname{lin}}}\right]^{-1}\right) \\
	&\leq \frac{\beta}{d}\left\|\bm \Sigma_{d}\right\|_2 \sum_{j=1}^n \frac{\lambda_{j}\left(\widetilde{\bm X}\right)}{\left[n\lambda + \gamma+\lambda_{j}\left(\widetilde{\bm X}\right)\right]^{2}} + \frac{1}{d} \frac{h(0)^2 n}{\left[n\lambda + \gamma+\lambda_{1}\left(\widetilde{\bm X} \right)\right]^{2}}\\
	& \asymp \frac{\beta}{d}   \mathcal{N}_{\widetilde{\bm X}}^{n \lambda + \gamma} + \mathcal{O}\left( \frac{1}{nd} \right) \,.
	\end{split}
	\end{equation}
	To bound the remaining terms in Eq.~\eqref{varinitial}, we also need the following results that can be obtained from \cite{liang2020just}:\\
	(i) By Proposition A.2 in \cite{liang2020just}, with probability at least $1-\delta - d^{-2}$, for $\theta=\frac{1}{2} - \frac{2}{8+m}$ and any given $\varepsilon > 0$, we have\\
	$\left\|\bm K + n\lambda \bm I -\widetilde{\bm K^{\operatorname{lin}}}\right\|_2 \leq d^{-\theta}\left(\delta^{-1 / 2}+\log ^{0.5+\varepsilon} d\right)$ and
	$\mathbb{E}_{\bm x} \left\|k(\bm x, \bm X)-k^{\operatorname{lin}}(\bm x, \bm X)\right\|_2^2 \leq \widetilde{C_1} d^{-(4 \theta-1)} \log ^{2+4 \varepsilon} d$.\\
	(ii) $\left\| (\bm K+ n\lambda \bm I)^{-1}\right\|_2 \leq \frac{2}{n\lambda + \gamma}$ and $\left\|(\bm K + n\lambda \bm I)^{-1} \widetilde{\bm K^{\operatorname{lin}}}\right\|_2 \leq 2$ provided $d$ is large enough such that $d^{-\theta}\left(\delta^{-1 / 2}+\log ^{0.5+\varepsilon} d\right) \leq \gamma/2$.
	
	Combining the above results, with probability at least $1-\delta - d^{-2}$, for any given $\varepsilon > 0$, The error bound for the variance in Eq.~\eqref{varinitial} can be further given by
	\begin{equation*}
	\begin{split}
	{\tt{V}}
	& \leq  \sigma^2 \mathbb{E}_{\bm x} \| (\bm K + n\lambda \bm I)^{-1} k(\bm x, \bm X) \|_2^2 \\
	& \leq 2\sigma^2 \left\|(\bm K + n\lambda \bm I)^{-1}\widetilde{\bm K^{\operatorname{lin}}}\right\|_2^2~ \mathbb{E}_{\bm x} \| [\widetilde{\bm K^{\operatorname{lin}}}]^{-1} k^{\operatorname{lin}}(\bm X, \bm x)   \|_2^2 + 2\sigma^2 \| \bm K^{-1} \|^2_2~ \mathbb{E}_{\bm x}\! \left\|k(\bm x, \bm X)-k^{\operatorname{lin}}(\bm x, \bm X)\right\|_2^2 \\
	& \asymp \frac{8\sigma^2 \beta}{d} \left\|\bm \Sigma_{d}\right\|_2 \sum_{j=1}^n \frac{\lambda_{j}(\widetilde{\bm X})}{\left[n\lambda + \gamma+\lambda_{j}(\widetilde{\bm X})\right]^{2}} + \frac{8\sigma^2}{(n\lambda+\gamma)^2} \widetilde{C_1} d^{-(4 \theta-1)} \log ^{2+4 \varepsilon} d \\
	& \asymp  \frac{\sigma^2 \beta}{d} \mathcal{N}_{\widetilde{\bm X}}^{n \lambda + \gamma} + \frac{\sigma^2}{(n\lambda+\gamma)^2} d^{-(4 \theta-1)} \log ^{2+4 \varepsilon} d\,,
	\end{split}
	\end{equation*}
	which concludes the proof.
\end{proof}

\subsection{Radial kernel matrices}
In this subsection, we consider the radial kernel case with $
k(\bm x, \bm x') = h\left( \frac{1}{d} \| \bm x - \bm x'\|_2^2 \right)$.
Since the linearization of radial kernel matrices incurs in two additionally terms $\bm A$ and $\bm A \odot \bm A$,
estimation for radial kernels is more technical than that of inner-product kernels.
Accordingly, to prove Theorem~\ref{provar} for radial kernels, we need to introduce the following notations and auxiliary results.  

\subsubsection{Auxiliary results}
Recall $\tau := \operatorname{tr}(\bm \Sigma_d)/d$, define
\begin{equation}\label{klinmatrixrbf}
\begin{split}
&\widetilde{\bm K^{\operatorname{lin}}}(\bm X, \bm X) :=(\gamma+ n\lambda) \bm I+\alpha \bm 1\bm 1^{\!\top}+\beta \frac{\bm X \bm X^{\!\top}}{d} + h^{\prime}(2 \tau ) \bm A + \frac{1}{2}h^{\prime \prime}(2 \tau) \bm A \odot \bm A \\
&
\ k^{\operatorname{lin}}(\bm x, \bm X) := h(2 \tau)\bm 1 + \beta \frac{\bm X \bm x^{\!\top}}{d} - \frac{\beta}{2}\bm A(\bm x, \bm X)  \in \mathbb{R}^{n \times 1}\,,
\end{split}
\end{equation}
where $\bm A(\bm x, \bm X) := \psi_{\bm x} + [\psi_1, \psi_2, \cdots, \psi_n]^{\!\top}$ with $\psi_{\bm x} = \| \bm x \|^2_2/d - \tau$ and $\psi_i = \| \bm x_i \|^2_2/d - \tau$ for $i=1,2,\ldots, n$.
As discussed in Appendix~\ref{sec:proofeig}, we conclude that $\widetilde{\bm K^{\operatorname{lin}}}$ admits the same eigenvalue decay as $\widetilde{\bm X}$ since $\bm A$ is a rank-2 matrix.
Accordingly, we have the following results.

\begin{proposition}\label{propexa}
	Given $\bm A(\bm x, \bm X)$ in Eq.~\eqref{klinmatrixrbf}, we have $\mathbb{E}_{\bm x} [\bm X \bm x \bm A(\bm X, \bm x)] = \mu_3 \bm X \bm \Sigma_d^{1/2} \diag(\bm \Sigma_d) \bm 1_n^{\!\top} $, where $\mu_3 := \mathbb{E}[\bm t(j)^3]$ does not depend on $j$ because each entry in $\bm t$ are independent for $j=1,2,\dots,d$. Further, $\mathbb{E}_{\bm x} [\bm X \bm x \bm A(\bm X, \bm x)]$ is a rank-one matrix with its eigenvalue $\lambda_1( \mathbb{E}_{\bm x} [\bm X \bm x \bm A(\bm X, \bm x)] ) = \mathcal{O}(\sqrt{n/d})$.
\end{proposition}
\begin{proof}[Proof of Proposition~\ref{propexa}]
	According to the definition in Assumption~\ref{assum8m}, $\bm x_i = \bm \Sigma_d^{1/2} \bm t_i$ with $\mathbb{E}[\bm t_i(j)] = 0$ and $\mathbb{V}[\bm t_i(j)] = 1$, 
	we have the following expression
	\begin{equation*}
	\begin{split}
	\mathbb{E}_{\bm t} [\bm t \bm t^{\!\top} \bm \Sigma_d \bm t] = \mathbb{E}_{\bm t} \left[\bm t \sum_{i,j=1}^d \bm t(i) (\bm \Sigma_d)_{ij} \bm t(j) \right]= \mu_3[ (\bm \Sigma_d)_{11},  (\bm \Sigma_d)_{22}, \cdots, (\bm \Sigma_d)_{dd}]^{\!\top} \,,
	\end{split}
	\end{equation*}
	where $\mu_3 := \mathbb{E}(t_i^3)$.
	Accordingly, $\mathbb{E}_{\bm x} [\bm X \bm x \bm A(\bm X, \bm x)]$ can be computed by
	\begin{equation*}
	\begin{split}
	\mathbb{E}_{\bm x} [\bm X \bm x \bm A(\bm X, \bm x)] & = \mathbb{E}_{\bm x} [\bm X \bm x(\psi_1 + \psi_{\bm x}), \bm X \bm x(\psi_2 + \psi_{\bm x}) \cdots, \bm X \bm x(\psi_n + \psi_{\bm x})] \\
	& = \mathbb{E}_{\bm x} [\bm X \bm x \psi_{\bm x}, \bm X \bm x \psi_{\bm x}, \cdots, \bm X \bm x \psi_{\bm x} ] \\
	& =  \bm X \bm \Sigma_d^{1/2} \left[ \frac{ \mathbb{E}_{\bm t} [\bm t \bm t^{\!\top} \bm \Sigma_d \bm t] }{d}, \frac{ \mathbb{E}_{\bm t} [\bm t \bm t^{\!\top} \bm \Sigma_d \bm t] }{d}, \cdots,  \frac{ \mathbb{E}_{\bm t} [\bm t \bm t^{\!\top} \bm \Sigma_d \bm t] }{d} \right] \\
	& =  \mu_3 \bm X \bm \Sigma_d^{1/2} \diag(\bm \Sigma_d) \bm 1_n^{\!\top} \,.
	\end{split}
	\end{equation*}
	Note that, the matrix $\diag(\bm \Sigma_d) \bm 1_n^{\!\top}$ is a rank-one matrix, which implies $\operatorname{rank}(\bm X \bm \Sigma_d^{1/2} \diag(\bm \Sigma_d) \bm 1_n^{\!\top}) \leq 1$. Accordingly, its non-zero eigenvalue $\lambda_1(\bm X \bm \Sigma_d^{1/2} \diag(\bm \Sigma_d) \bm 1_n^{\!\top} )$ admits
	\begin{equation*}
	\begin{split}
	\frac{1}{d}\lambda_1(\bm X \bm \Sigma_d^{1/2} \diag(\bm \Sigma_d) \bm 1_n^{\!\top} ) = \frac{1}{d}\sum_{i=1}^n \bm x_i^{\!\top} \bm \Sigma_d^{\frac{1}{2}} \diag(\bm \Sigma_d) = \frac{1}{d}\sum_{i=1}^n \bm t_i^{\!\top} \bm \Sigma_d \diag(\bm \Sigma_d) \,.
	\end{split}
	\end{equation*}
	Due to $\mathbb{E}[\bm t_i^{\!\top} \bm \Sigma_d \diag(\bm \Sigma_d)] = 0$ and $\mathbb{V}[\bm t_i^{\!\top} \bm \Sigma_d \diag(\bm \Sigma_d)] = \| \bm \Sigma_d \diag(\bm \Sigma_d) \|_2^2 $, which, with a central limit theorem argument,
	implies $\sum_{i=1}^n \bm t_i^{\!\top} \bm \Sigma_d \diag(\bm \Sigma_d) = \mathcal{O}(\sqrt{nd}) $ due to $\| \bm \Sigma_d \diag(\bm \Sigma_d)\|_2 \leq \| \bm \Sigma_d \|_2 \| \diag(\bm \Sigma_d) \|_2 \leq \widetilde{C} \| \diag(\bm \Sigma_d) \|_2$.
	Accordingly, we can conclude that 
	$\frac{1}{d}\lambda_1(\bm X \bm \Sigma_d^{1/2} \diag(\bm \Sigma_d) \bm 1_n^{\!\top} ) = \mathcal{O}(\sqrt{n/d}) $.
\end{proof}

\begin{proposition}\label{propeaa}
	Given $\bm A(\bm x, \bm X)$ in Eq.~\eqref{klinmatrixrbf}, we have $\mathbb{E}_{\bm x} [\bm A(\bm x, \bm X) \bm A(\bm X, \bm x)] = \bm \psi \bm \psi^{\!\top} + \mathcal{O}(1/d)$. Further, it has only one non-zero eigenvalue that admits $\lambda_1(\mathbb{E}_{\bm x} [\bm A(\bm x, \bm X) \bm A(\bm X, \bm x)]) = \mathcal{O}(n)$.
\end{proposition}
\begin{proof}[Proof of Proposition~\ref{propeaa}]
	By virtue of the following results \cite{el2010spectrum}
	\begin{equation*}
	\begin{split}
	\frac{1}{d}\mathbb{E}_{\bm x} \| \bm x \|_2^2 &= \frac{1}{d} \mathbb{E}_{\bm t}[\bm t^{\!\top} \bm \Sigma_d \bm t] = \tau \\
	\mathbb{V}_{\bm x} \left[\frac{\| \bm x \|_2^2}{d} \right] & =\frac{1}{d^2} \bigg( (\mu_4 - 3) \sum_{i=1}^d ( (\bm \Sigma_d)_{ii})^2 + 2 \operatorname{tr}(\bm \Sigma_d^2) \bigg) = \mathcal{O}\left(\frac{1}{d} \right) \,,
	\end{split}
	\end{equation*}
	where $\mu_4 := \mathbb{E}[\bm t(i)^4]$ does not depend on $i$. Accordingly, each entry in $\mathbb{E}_{\bm x} [\bm A(\bm x, \bm X) \bm A(\bm X, \bm x)]$ can be computed as
	\begin{equation*}
	\begin{split}
	\mathbb{E}_{\bm x} [\bm A(\bm x, \bm X) \bm A(\bm X, \bm x)]_{ij} &= \mathbb{E}_{\bm x}[(\psi_{i} + \psi_{\bm x}) (\psi_{j} + \psi_{\bm x})] \\
	& = \psi_{i} \psi_{j} + (\psi_{i} + \psi_{j})\mathbb{E}_{\bm x} \psi_{\bm x} + \mathbb{E}_{\bm x} [\psi_{\bm x}^2] \\
	& = \psi_{i} \psi_{j} + \mathbb{V}_{\bm x} \left[\frac{\| \bm x \|_2^2}{d} \right] \\
	& = \psi_{i} \psi_{j}  + \frac{\mu_4 - 3}{d^2} \operatorname{tr}(\bm \Sigma_d \odot \bm \Sigma_d) + \frac{2\operatorname{tr}(\bm \Sigma_d^2)}{d^2}\,.
	\end{split}
	\end{equation*}
	Then we have
	\begin{equation*}
	\mathbb{E}_{\bm x} [\bm A(\bm x, \bm X) \bm A(\bm X, \bm x)] = \bm \psi \bm \psi^{\!\top} + \mathcal{O}(1/d)\,.
	\end{equation*}
	Therefore, $\bm \psi \bm \psi^{\!\top}$ is a rank-one matrix with $\lambda_1(\bm \psi \bm \psi^{\!\top}) = \| \bm \psi \|_2^2 = \mathcal{O}(n)$.
	Then $\lambda_1(\mathbb{E}_{\bm x} [\bm A(\bm x, \bm X) \bm A(\bm X, \bm x)])$ can be estimated by
	\begin{equation*}
	\| \bm \psi \|_2^2 	\leq \lambda_1(\mathbb{E}_{\bm x} [\bm A(\bm x, \bm X) \bm A(\bm X, \bm x)]) \leq \| \bm \psi \|_2^2 + n \underbrace{\left[\frac{\mu_4 - 3}{d^2} \operatorname{tr}(\bm \Sigma_d \odot \bm \Sigma_d) + \frac{2\operatorname{tr}(\bm \Sigma_d^2)}{d^2} \right]}_{= \mathcal{O}(1/d)}\,,
	\end{equation*}
	which implies $\lambda_1(\mathbb{E}_{\bm x} [\bm A(\bm x, \bm X) \bm A(\bm X, \bm x)]) = \mathcal{O}(n)$.
\end{proof}

\begin{lemma}\label{lemma1rbf}
	Given a radial kernel, under Assumption~\ref{assum8m}, for $\theta = \frac{1}{2} - \frac{2}{8+m}$, we have with probability at least $1-d^{-2}$ with respect to the draw of $\bm X$, for $d$ large enough, for any given $\varepsilon > 0$, we have
	\begin{equation*}
	\mathbb{E}_{\bm x} \left\|k(\bm x, \bm X)-k^{\operatorname{lin}}(\bm x, \bm X)\right\|_2^2 \leq \widetilde{C_1} d^{-2 \theta} \log ^{1+\varepsilon} d\,,
	\end{equation*}
	where $\widetilde{C_1}$ is some constant independent of $n$ and $d$.
\end{lemma}
\noindent{\bf Remark:} In fact, we only need the $(5+m)$-moment in Assumption~\ref{assum8m} but we still follow with it for simplicity.

\begin{proof}[Proof of Lemma~\ref{lemma1rbf}]
	We start with the entry-wise Taylor expansion for the smooth kernel at $2 \tau$ with $\tau := \operatorname{tr}(\bm \Sigma_d)/d$
	\begin{equation*}
	\begin{split}
	k(\bm x, \bm x_j) = h(\frac{1}{d}\| \bm x - \bm x_j \|_2^2) &= h(2 \tau) + h'(2 \tau) (\frac{1}{d}\| \bm x - \bm x_j \|^2 - 2 \tau) + \frac{h^{\prime \prime}(2 \tau)}{2} \left(\frac{1}{d}\| \bm x - \bm x_j \|^2 - 2 \tau \right)^2 + \mathcal{O}(d^{-3/2}) \\
	& = h(2 \tau) + h'(2 \tau)(\psi_{\bm x} + \psi_j - \frac{2\bm x^{\!\top} \bm x_j}{d}) + \frac{h^{\prime \prime}(2 \tau)}{2} \left(\psi_{\bm x} + \psi_j - \frac{2\bm x^{\!\top} \bm x_j}{d} \right)^2  + \mathcal{O}(d^{-3/2})\,,
	\end{split}
	\end{equation*}
	where $\psi_{j} = \| \bm x_j \|^2_2/d - \tau$ for $j = 1,2,\dots,n$ as defined before.
	Accordingly, by virtue of $k^{\operatorname{lin}}(\bm x, \bm x_j) = \frac{\beta \bm x^{\!\top} \bm x_j}{d} - \frac{\beta}{2}(\psi_{\bm x} + \psi_j)$ and Corollary 2 in \cite{el2010spectrum}, with probability at least $1 - d^{-2}$, for any $\varepsilon > 0$, we have
	\begin{equation*}
	\begin{split}
	k(\bm x, \bm x_j) - k^{\operatorname{lin}}(\bm x, \bm x_j) & = \frac{h^{\prime \prime}(2\tau)}{2} \left(\frac{1}{d}\| \bm x - \bm x_j \|_2^2 - 2\tau \right)^2  \leq \widetilde{C} d^{-1+\frac{4}{m}} (\log d)^{\frac{1+\varepsilon}{2}} \,,
	\end{split}
	\end{equation*}
	where we only need $(5+m)$-moment.
	Therefore, with probability at least $1-d^{-2}$, for any given $\varepsilon > 0$, we have
	\begin{equation*}
	\left\|k(\bm x, \bm X)-k^{\operatorname{lin}}(\bm x, \bm X)\right\|_2 \leq C_1 d^{-1/2+\frac{4}{m}} (\log d)^{\frac{1+\varepsilon}{2}} \leq \widetilde{C_1} d^{-\theta} (\log d)^{\frac{1+\varepsilon}{2}}\,,
	\end{equation*}
	which implies 
	\begin{equation*}
	\mathbb{E}_{\bm x} \left\|k(\bm x, \bm X)-k^{\operatorname{lin}}(\bm x, \bm X)\right\|_2^2 \leq \widetilde{C_2} d^{-2 \theta} \log ^{1+\varepsilon} d\,,
	\end{equation*}
	where $\widetilde{C_1}$ and $\widetilde{C_2}$ are some constant independent of $n$ and $d$.
\end{proof}

\subsubsection{Proofs of Theorem~\ref{provar} for radial kernels}
Now we are ready to prove Theorem~\ref{provar} for radial kernels.
\begin{proof}[Proof of Theorem~\ref{provar} for radial kernels]
	Similar to Eq.~\eqref{varinitial}, to estimate ${\tt{V}} \leq \sigma^2 \mathbb{E}_{\bm x} \| (\bm K + n\lambda \bm I)^{-1} k(\bm x, \bm X) \|_2^2 $, we need to bound subsequently the following terms: $\left\|\bm K + n\lambda \bm I -\widetilde{\bm K^{\operatorname{lin}}}(\bm X,\bm X)\right\|_2$, $\left\| (\bm K+ n\lambda \bm I)^{-1}\right\|_2$, $\left\|(\bm K + n\lambda \bm I)^{-1} \widetilde{\bm K^{\operatorname{lin}}}(\bm X, \bm X)\right\|_2$, $\mathbb{E}_{\bm x}\| [\widetilde{\bm K^{\operatorname{lin}}}(\bm X, \bm X)]^{-1} k^{\operatorname{lin}}(\bm x, \bm X)   \|_2^2$, and $\mathbb{E}_{\bm x} \|  k(\bm x, \bm X) -  k^{\operatorname{lin}}(\bm x, \bm X) \|_2^2$.
	
	In \cite{el2010spectrum}, the approximation error between radial kernel matrices and their linearization can be decomposed into three parts: the first-order term $A_1$, the second-order term $A_2$, and the third-order term $A_3$
	\begin{equation*}
	\left\|\bm K + n\lambda \bm I -\widetilde{\bm K^{\operatorname{lin}}}(\bm X,\bm X)\right\|_2 := A_1 + A_2 + A_3\,,
	\end{equation*}
	where $A_1$ and $A_3$ admit $\| A_1 \|_2 \leq d^{-\theta} \log^{2+4\varepsilon}  d$ and $\| A_3 \|_2 \leq d^{-\theta} \log^{2+4\varepsilon} d$. The second-order term $A_2$ admits $\operatorname{Pr}(\| A_2 \| \leq d^{-\theta} \delta^{-1/2}) \leq \delta$ by Proposition A.2 in \cite{liang2020just} and \cite{el2010spectrum}. 
	Accordingly, with probability at least $1-\delta - d^{-2}$, for $\theta=\frac{1}{2} - \frac{2}{8+m}$ and any given $\varepsilon > 0$, we have
	\begin{equation*}
	\left\|\bm K + n\lambda \bm I -\widetilde{\bm K^{\operatorname{lin}}}(\bm X,\bm X)\right\|_2 \leq d^{-\theta}\left(\delta^{-1 / 2}+\log ^{2+4\varepsilon} d\right) \,.
	\end{equation*}

	According to Proposition~\ref{propexa} and \ref{propeaa}, we have
	\begin{equation}\label{vargauklin1}
	\begin{split}
	&\mathbb{E}_{\bm x} \| [\widetilde{\bm K^{\operatorname{lin}}}(\bm X, \bm X)]^{-1} k^{\operatorname{lin}}(\bm X, \bm x)   \|_2^2 \\
	&=\beta^2 \mathbb{E}_{\bm x} \operatorname{tr}\left[ [\widetilde{\bm K^{\operatorname{lin}}}]^{-1} \Big( \frac{\bm X \bm x \bm x^{\!\top} \bm X^{\!\top}}{d^2} - \frac{\bm X \bm x \bm A(\bm X, \bm x)}{d} + \frac{1}{4} \bm A(\bm x, \bm X) \bm A(\bm X, \bm x) + h(2\tau)^2 \bm 1 \bm 1^{\!\top} \Big)  [\widetilde{\bm K^{\operatorname{lin}}}]^{-1}\right] \\
	& \leq \beta^2 \operatorname{tr}\left([\widetilde{\bm K^{\operatorname{lin}}}]^{-1} \Big( \frac{\bm X \bm X^{\!\top} \| \bm \Sigma_d \|_2}{d^2} - \frac{\mu_3 \bm X \bm \Sigma_d^{1/2} \diag(\bm \Sigma_d) \bm 1_n^{\!\top}}{d} + \frac{1}{4} \bm A(\bm x, \bm X) \bm A(\bm X, \bm x) + h(2\tau)^2 \bm 1 \bm 1^{\!\top} \Big)  [\widetilde{\bm K^{\operatorname{lin}}}]^{-1}\right) \\
	& = \frac{\beta^2 \| \bm \Sigma_d \|_2}{d} \sum_{i=1}^n \frac{\lambda_i(\bm X \bm X^{\!\top}/d)}{[\lambda_i(\widetilde{\bm K^{\operatorname{lin}}})]^2} - \frac{\beta^2 \mu_3}{d} \frac{\lambda_1(\bm X \bm \Sigma_d^{1/2} \diag(\bm \Sigma_d) \bm 1_n^{\!\top} )}{[\lambda_1(\widetilde{\bm K^{\operatorname{lin}}})]^2} + \beta^2   \frac{4\lambda_1(\mathbb{E}_{\bm x}[\bm A(\bm x, \bm X) \bm A(\bm X, \bm x)]) + h(2\tau)^2 n}{[\lambda_1(\widetilde{\bm K^{\operatorname{lin}}})]^2} \\
	& \asymp \frac{\beta^2 \left\|\bm \Sigma_{d}\right\|_2 }{d} \sum_{i=1}^n \frac{\lambda_i(\bm X \bm X^{\!\top}/d)}{[\lambda_1(\widetilde{\bm K^{\operatorname{lin}}})]^2} + \frac{\mathcal{O}(\sqrt{n/d})}{[\lambda_1(\widetilde{\bm K^{\operatorname{lin}}})]^2} +  \frac{ \mathcal{O}(n) }{[\lambda_1(\widetilde{\bm K^{\operatorname{lin}}})]^2} +  \frac{ \mathcal{O}(n) }{[\lambda_1(\widetilde{\bm K^{\operatorname{lin}}})]^2} \\
	& \asymp \frac{\beta}{d} \mathcal{N}_{\widetilde{\bm X}}^{n \lambda + \gamma} + \mathcal{O}\left(\frac{1}{n}\right) \,.
	\end{split}
	\end{equation}
	It can be found that, the above error bounds are the same as that of inner-product kernels, except two additional terms due to the considered $\bm A$ and $\bm A \odot \bm A$ in the linearization, which can be shown small in the large $n,d$ regime.
	
	By virtue of $\left\| (\bm K+ n\lambda \bm I)^{-1}\right\|_2 \leq \frac{2}{n\lambda + \gamma} $ and $\left\|(\bm K + n\lambda \bm I)^{-1} \widetilde{\bm K^{\operatorname{lin}}}\right\|_2 \leq 2$ in \cite{liang2020just}, Lemma~\ref{lemma1rbf}, and the above equations, with probability at least $1-\delta - d^{-2}$, for any given $\varepsilon > 0$, we have
	\begin{equation}\label{variancerbf}
	\begin{split}
	{\tt{V}}
	& \leq  \sigma^2 \mathbb{E}_{\bm x} \| (\bm K + n\lambda \bm I)^{-1} k(\bm x, \bm X) \|_2^2 \\
	& \leq 2\sigma^2 \left\|(\bm K + n\lambda \bm I)^{-1} \widetilde{\bm K^{\operatorname{lin}}}\right\|_2^2~ \mathbb{E}_{\bm x} \| [\widetilde{\bm K^{\operatorname{lin}}}]^{-1} k^{\operatorname{lin}}(\bm X, \bm x)   \|_2^2 + 2\sigma^2 \| \bm K^{-1} \|^2_2~ \mathbb{E}_{\bm x}\! \left\|k(\bm x, \bm X)-k^{\operatorname{lin}}(\bm x, \bm X)\right\|_2^2 \\
	& \leq 8 \sigma^2 \mathbb{E}_{\bm x} \| [\widetilde{\bm K^{\operatorname{lin}}}]^{-1} k^{\operatorname{lin}}(\bm X, \bm x)   \|_2^2 + \frac{8\sigma^2}{(n\lambda+\gamma)^2} \widetilde{C_1} d^{-2 \theta} \log ^{1+\varepsilon} d \\
	& \asymp  \frac{\sigma^2 \beta}{d} \mathcal{N}_{\widetilde{\bm X}}^{n \lambda + \gamma}  + \frac{\sigma^2}{(n\lambda+\gamma)^2} d^{-2 \theta} \log ^{1+\varepsilon} d \,,
	\end{split}
	\end{equation}
	where the second inequality admits by Lemma~\ref{lemma1rbf}, and the last inequality follows by Eq.~\eqref{vargauklin1}.
	Finally, we conclude the proof.
\end{proof}

\section{Proof of Proposition~\ref{pro:decay}}
\label{sec:proofdecay}

In this section, we discuss $\mathcal{N}^{n\lambda + \gamma}_{\widetilde{\bm X}} $ based on three eigenvalue decays: \emph{harmonic decay}, \emph{polynomial decay}, and \emph{exponential decay} under two regimes $n < d$ and $n > d$.

\subsection{$n < d$ case}
\label{sec:appnleqd}
Recall $b:= n\lambda + \gamma > 0$, and $\mathcal{N}^b_{\widetilde{\bm X}} := \sum_{i=1}^n \frac{\lambda_{i}(\widetilde{\bm X} )}{\left[b +\lambda_{i}(\widetilde{\bm X})\right]^{2}}$, define 
$F(\lambda_i) := \frac{\lambda_i}{(b+\lambda_i)^2} $ where $\lambda_i$ is short for $\lambda_{i}(\widetilde{\bm X})$.
We notice that, when $\lambda_i \leq b$, $F(\lambda_i)$ is an increasing function of $\lambda_i$, and thus a decreasing function of $i$ when the above three eigenvalue decays are considered.
Likewise, when $\lambda_i \geq b$, $F(\lambda_i)$ is a decreasing function of $\lambda_i$, and thus an increasing function of $i$.
Without loss of generality, we assume that the first $q$ eigenvalues satisfy $\lambda_i \geq b$ with $i=1,2,\cdots,q$
and the remaining $n-q$ eigenvalues satisfy $\lambda_i \leq b$ with $i=m+1, m+2 \cdots, n$.
Clearly, the integer $q$ can be chosen from $0$ to $n$.
Accordingly, denote $r_* := \operatorname{rank}(\widetilde{\bm X})$ which includes the rank-deficient case, $\mathcal{N}^b_{\widetilde{\bm X}}$ can be upper bounded by the Riemann sum as follows.


{\bf Harmonic decay} $\lambda_i( \widetilde{\bm X}) \propto n/i$ for $i \in \{ 1, 2, \dots, r_* \}$ and $\lambda_i( \widetilde{\bm X}) = 0$ for $i \in \{r_*+1, \dots, n  \}$
\begin{equation*}
\begin{split}
\frac{1}{d} \mathcal{N}^b_{\widetilde{\bm X}} = \frac{1}{d}\sum_{i=1}^{r_*} \frac{n/i}{\left(b+n/i \right)^{2}} &= \frac{1}{d}\sum_{i=1}^{q} \frac{n/i}{\left(b+n/i \right)^{2}}  + \frac{1}{d}\sum_{i=q+1}^{r_*} \frac{n/i}{\left(b+n/i \right)^{2}} \\
& \leqslant \frac{1}{nd}\int_{1}^{q+1} \frac{t}{\left(1+\frac{bt}{n} \right)^{2}} \mathrm{d} t + \frac{1}{nd}\int_{q+1}^{r_*+1} \frac{t}{\left(1+\frac{bt}{n} \right)^{2}} \mathrm{d} t \\
& = \frac{n}{b^2d}  \int_{\frac{b}{n}}^{\frac{(r_*+1)b}{n}} \frac{u}{\left(1+u \right)^{2}} \mathrm{d} u~\mbox{with the change of variable $u=tb/n$} \\
& = \frac{n}{b^2d} \left[ \ln \frac{n+(r_*+1)b}{n+b} + \frac{n}{n+b+r_*b} - \frac{n}{n+b}  \right] \\
& \leqslant \frac{n}{b^2d} \ln \frac{n+(r_*+1)b}{n+b} = \mathcal{O}(\frac{n}{b^2d})\,.
\end{split}
\end{equation*} 


{\bf Polynomial decay:} $\lambda_i( \widetilde{\bm X}) \propto ni^{-2a}$ with $a > 1/2$ for $i \in \{ 1, 2, \dots, r_* \}$ and $\lambda_i( \widetilde{\bm X}) = 0$ for $i \in \{r_*+1, \dots, n  \}$.
Hence, we actually aim to bound
\begin{equation*}
\begin{split}
\frac{1}{d} \mathcal{N}^b_{\widetilde{\bm X}} = \frac{1}{d}\sum_{i=1}^{r_*} \frac{ni^{-2 a}}{\left(b+ni^{-2 a} \right)^{2}} & = \frac{1}{d}\sum_{i=1}^{q} \frac{ni^{-2 a}}{\left(b+ni^{-2 a} \right)^{2}} + \frac{1}{d}\sum_{i=q+1}^{r_*+1} \frac{ni^{-2 a}}{\left(b+ni^{-2 a} \right)^{2}} \\
& \leqslant \frac{1}{nd}\int_{1}^{r_*+1} \frac{t^{2a}}{\left(1+ \frac{t^{2 a}b}{n} \right)^{2}} \mathrm{d} t \\
& = \frac{1}{2ab d} \left( \frac{n}{b} \right)^{\frac{1}{2a}} \int_{b/n}^{(r_*+1)^{2a}b/n} \frac{u^{\frac{1}{2a}}}{\left(1+u \right)^{2}} \mathrm{d} u~\mbox{with the change of variable $u=t^{2a}b/n$} \\
& \leqslant \widetilde{C} \frac{1}{2ab d} \left( \frac{n}{b} \right)^{\frac{1}{2a}}  ~\quad \mbox{since the integral is finite due to $2a > 1$} 
\end{split}
\end{equation*}

{\bf Exponential decay:} $\lambda_i( \widetilde{\bm X}) \propto ne^{-ai}$ with $a > 0$ for $i \in \{ 1, 2, \dots, r_* \}$ and $\lambda_i( \widetilde{\bm X}) = 0$ for $i \in \{r_*+1, \dots, n  \}$.

We aim to bound the sum as
\begin{equation*}
\begin{split}
\frac{1}{d} \mathcal{N}^b_{\widetilde{\bm X}} = \frac{1}{d}\sum_{i=1}^{r_*} \frac{ne^{-ai}}{\left(b+ne^{-ai} \right)^{2}} & = \frac{1}{d}\sum_{i=1}^{q} \frac{ne^{-ai}}{\left(b+ne^{-ai} \right)^{2}} + \frac{1}{d}\sum_{i=q+1}^{r_*} \frac{ne^{-ai}}{\left(b+ne^{-ai} \right)^{2}} \\
& \leqslant \frac{1}{d}\int_{1}^{r_*+1} \frac{ne^{-at }}{\left(b+ne^{-at} \right)^{2}} \mathrm{d} t \\
& = \frac{1}{ad} \int_{ne^{-a(r_*+1)}}^{{ne^{-a}}} \frac{1}{\left(b+u \right)^{2}} \mathrm{d} u~\mbox{with the change of variable $u=ne^{-at}$} \\
& = \frac{1}{ad} \left( \frac{1}{b+ne^{-a(r_*+1)}} - \frac{1}{b+ne^{-a}} \right) \,.
\end{split}
\end{equation*}
Note that, the monotonicity of $\mathcal{N}^b_{\widetilde{\bm X}}$ (also ${\tt{V}}_1$) with respect to $n$ is relatively clear for \emph{harmonic decay} and \emph{polynomial decay} but is unclear in the case of \emph{exponential decay}.
Here we study the monotonicity in the \emph{exponential decay}.
Denote the function $G(n) := \big( \frac{1}{b+ne^{-a(r_*+1)}} \!-\! \frac{1}{b+ne^{-a}} \big)$ with $b:=n\lambda + \gamma$, taking $\lambda := \bar{c} n^{-\vartheta}$, its derivation is
\begin{equation}\label{Gderin}
G'(n) = \frac{-\bar{c}(1-\vartheta) n^{-\vartheta} - e^{-a(r_*+1)}}{\left[ cn^{1-\vartheta} + \gamma + n e^{-a(r_*+1)} \right]^2} + \frac{\bar{c}(1-\vartheta) n^{-\vartheta} + e^{-a}}{\left[ cn^{(1-\vartheta)} + \gamma + ne^{-a} \right]^2} \,,
\end{equation}
which can be rewritten as
\begin{equation*}
G'(n) = \frac{\bar{c}(1-\vartheta) n^{-\vartheta} + e^{-a}}{\left[ cn^{1-\vartheta} + \gamma + n e^{-a(r_*+1)} \right]^2} \left( \underbrace{ \frac{\left[ \bar{c}n^{1-\vartheta} + \gamma + n e^{-a(r_*+1)} \right]^2}{\left[ \bar{c}n^{(1-\vartheta)} + \gamma + ne^{-a} \right]^2} }_{\triangleq H_1(n)} - \underbrace{ \frac{\bar{c}(1-\vartheta) n^{-\vartheta} + e^{-a(r_*+1)}}{\bar{c}(1-\vartheta) n^{-\vartheta} + e^{-a}} }_{\triangleq H_2(n)} \right)\,.
\end{equation*}
It can be found that both $H_1(n)$ and $H_2(n)$ are decreasing functions with $n$. More specifically, their maximum and minimum can be achieved with
\begin{equation*}
\max_n H_1(n) = H_1(1) = \left( \frac{\bar{c}+ \gamma + e^{-a(r_*+1)}}{\bar{c}+\gamma + e^{-a}} \right)^2,\quad \min_n H_1(n)  = \lim\limits_{n \rightarrow \infty} H_1(n) = \left( \frac{e^{-a(r_*+1)}}{e^{-a}} \right)^2\,,
\end{equation*}
and 
\begin{equation*}
\max_n H_2(n) = H_2(1) = \frac{\bar{c}(1-\vartheta) + e^{-a(r_*+1)}}{\bar{c}(1-\vartheta) + e^{-a}} ,\quad \min_n H_2(n)  = \lim\limits_{n \rightarrow \infty} H_2(n) =  \frac{e^{-a(r_*+1)}}{e^{-a}} \,.
\end{equation*}

Accordingly, if $H_1(1) < H_2(1)$, we obtain a decreasing function $G(n)$ of $n$, which implies that $\mathcal{N}^b_{\widetilde{\bm X}}$ will decrease with $n$.
Here the condition $H_1(1) < H_2(1)$ indicates
\begin{equation*}
\left( \frac{\bar{c}+ \gamma + e^{-a(r_*+1)}}{\bar{c}+\gamma + e^{-a}} \right)^2 \leq \frac{\bar{c}(1-\vartheta) + e^{-a(r_*+1)}}{\bar{c}(1-\vartheta) + e^{-a}}\,,
\end{equation*}
which is equivalent to
\begin{equation}\label{dericondition}
(\vartheta \bar{c} + \gamma)^2 \leq \left[ e^{-a} + (1-\vartheta) \bar{c} \right] \left[ e^{-a(r_*+1)} + (1-\vartheta) \bar{c} \right]\,.
\end{equation}
Accordingly, if the above inequality holds,  $\mathcal{N}^b_{\widetilde{\bm X}}$ will decrease with $n$.
In Section~\ref{sec:appexpsyn}, we will experimentally check whether this condition holds or not. 

\subsection{$n > d$ case and the large $n$ limit}
\label{sec:ngeqd}
In this section, we consider the $n > d$ case, and further study the trend of ${\tt{V}}_1$ as $n \rightarrow \infty$.
Note that, in this case, ${\bm X \bm X^{\!\top}}/{d}$ has at most $r_* \leq d$ non-zero eigenvalues.
Accordingly, the Riemann sum is counted to $r_*$ instead of $n$. 
Similar to the above description, we also consider the following three eigenvalue decays.


{\bf Harmonic decay} $\lambda_i( \widetilde{\bm X}) \propto n/i$, $i \in \{ 1, 2, \cdots, d \}$
\begin{equation*}
\begin{split}
\frac{1}{d} \mathcal{N}^b_{\widetilde{\bm X}} = \frac{1}{d}\sum_{i=1}^{r_*} \frac{n/i}{\left(b+n/i \right)^{2}} &= \frac{1}{d}\sum_{i=1}^{q} \frac{n/i}{\left(b+n/i \right)^{2}}  + \frac{1}{d}\sum_{i=q+1}^{r_*} \frac{n/i}{\left(b+n/i \right)^{2}} \\
& \leqslant \frac{n}{b^2d}  \int_{\frac{b}{n}}^{\frac{(r_*+1)b}{n}} \frac{u}{\left(1+u \right)^{2}} \mathrm{d} u \\
& = \frac{n}{b^2d} \left[ \ln \frac{n+(r_*+1)b}{n+b} + \frac{n}{n+b+r_*b} - \frac{n}{n+b}  \right] \,.
\end{split}
\end{equation*} 
In particular, taking the limit of $n \rightarrow \infty$, we have
\begin{equation*}
\begin{split}
\lim\limits_{n \rightarrow \infty} \frac{1}{d} \mathcal{N}^b_{\widetilde{\bm X}} & = 	\lim\limits_{n \rightarrow \infty} \frac{n}{b^2d} \left[ \ln \frac{n+(r_*+1)b}{n+b} + \frac{n}{n+b+r_*b} - \frac{n}{n+b}  \right] \\
& = 	\lim\limits_{n \rightarrow \infty} \frac{n}{b^2d} \ln \frac{n+(r_*+1)b}{n+b} + \lim\limits_{n \rightarrow \infty} \frac{n}{b^2d} \left( \frac{n}{n+b+r_*b} - \frac{n}{n+b}  \right) \\
& = \frac{r_*}{d} \left( \lim\limits_{n \rightarrow \infty} \frac{1}{b} \frac{n}{n+b} - \lim\limits_{n \rightarrow \infty} \frac{n^2}{b(n+b+r_*)(n+b)} \right) \\
& \leq \lim\limits_{n \rightarrow \infty} \frac{1}{b} \frac{n}{n+b} - \lim\limits_{n \rightarrow \infty} \frac{n^2}{b(n+b+r_*)(n+b)} \\
& = 0\,.
\end{split}
\end{equation*}
Accordingly, by the squeeze theorem, we can conclude, given $d$, $\mathcal{N}^b_{\widetilde{\bm X}}$ tends to zero when $n \rightarrow \infty$.

{\bf Polynomial decay:} $\lambda_i( \widetilde{\bm X}) \propto ni^{-2a}$ with $a > 1/2$, $i \in \{ 1, 2, \cdots, d \}$
\begin{equation*}
\begin{split}
\frac{1}{d} \mathcal{N}^b_{\widetilde{\bm X}} = \frac{1}{d}\sum_{i=1}^{r_*} \frac{ni^{-2 a}}{\left(b+ni^{-2 a} \right)^{2}} 
& \leqslant \frac{1}{2ab d} \left( \frac{n}{b} \right)^{\frac{1}{2a}} \int_{b/n}^{(r_*+1)^{2a}b/n} \frac{u^{\frac{1}{2a}}}{\left(1+u \right)^{2}} \mathrm{d} u \\
& \leqslant \frac{1}{2ab d} \left( \frac{n}{b} \right)^{\frac{1}{2a}} \int_{0}^{\infty} \frac{u^{\frac{1}{2a}}}{\left(1+u \right)^{2}} \mathrm{d} u  \\
& \leqslant \widetilde{C} \frac{1}{2ab d} \left( \frac{n}{b} \right)^{\frac{1}{2a}}  ~\quad \mbox{since the integral is finite due to $2a > 1$} 
\end{split}
\end{equation*}
Since the integral $\int \frac{u^{\frac{1}{2a}}}{\left(1+u \right)^{2}} \mathrm{d} u$ can behave rather differently for different choices of $a$, here we take $a=1$ as an example.
Taking the limit of $n \rightarrow \infty$, we have 
\begin{equation*}
\begin{split}
\lim\limits_{n \rightarrow \infty} \frac{1}{d} \mathcal{N}^b_{\widetilde{\bm X}} & = 	\lim\limits_{n \rightarrow \infty} \frac{1}{2b d} \left( \frac{n}{b} \right)^{\frac{1}{2}} \int_{b/n}^{(r_*+1)^{2}b/n} \frac{u^{\frac{1}{2}}}{\left(1+u \right)^{2}} \mathrm{d} u \\
& = \frac{1}{2b d} \lim\limits_{n \rightarrow \infty} \left( \frac{n}{b} \right)^{\frac{1}{2}} \left( \arctan(\sqrt{u}) - \frac{\sqrt{u}}{u+1} \right) \Bigg|_{b/n}^{(r_*+1)^{2}b/n} \\
& = \frac{1}{2b d} \lim\limits_{n \rightarrow \infty} \sqrt{\frac{n}{b}} \left( (r_*+1)\sqrt{b/n} - \frac{(r_*+1)\sqrt{b/n}}{(r_*+1)^2b/n} - \sqrt{b/n} + \frac{\sqrt{b/n}}{b/n+1}  \right)~\mbox{using $\lim\limits_{x \rightarrow 0} \frac{\arctan x}{x} = 1$.}\\
& = 0\,.
\end{split}
\end{equation*}

{\bf Exponential decay:} $\lambda_i( \widetilde{\bm X}) \propto ne^{-ai}$ with $a > 0$, $i \in \{ 1, 2, \cdots, d \}$
\begin{equation*}
\begin{split}
\frac{1}{d} \mathcal{N}^b_{\widetilde{\bm X}} = \frac{1}{d}\sum_{i=1}^{r_*} \frac{ne^{-ai}}{\left(b+ne^{-ai} \right)^{2}} & = \frac{1}{d}\sum_{i=1}^{q} \frac{ne^{-ai}}{\left(b+ne^{-ai} \right)^{2}} + \frac{1}{d}\sum_{i=q+1}^{r_*} \frac{ne^{-ai}}{\left(b+ne^{-ai} \right)^{2}} \\
& \leqslant \frac{1}{ad} \int_{ne^{-a(r_*+1)}}^{{ne^{-a}}} \frac{1}{\left(b+u \right)^{2}} \mathrm{d} u \\
& = \frac{1}{ad} \left( \frac{1}{b+ne^{-a(r_*+1)}} - \frac{1}{b+ne^{-a}} \right) .
\end{split}
\end{equation*}
Taking the limit of $n \rightarrow \infty$, we can directly have $\lim\limits_{n \rightarrow \infty} \frac{1}{d} \mathcal{N}^b_{\widetilde{\bm X}} = 0$.

\section{Additional Experiments}
\label{sec:appexp}

In this section, we present additional experiments including the following parts:
\begin{itemize}
	\item In Section~\ref{sec:appexpeig}, we add the \emph{MNIST} dataset \cite{L1998Gradient} to verify the eigenvalue decay equivalence, and evaluate the effect by different orders in polynomial kernel.
	\item In Section~\ref{sec:appexpsyn}, our model works in a polynomial kernel setting under the \emph{polynomial decay} and \emph{exponential decay} of $\widetilde{\bm X}$ on the synthetic dataset.
\end{itemize}
\subsection{Eigenvalue decay equivalence}
\label{sec:appexpeig}
Apart from the \emph{YearPredictionMSD} dataset in the main text, we add the \emph{MNIST} dataset \cite{L1998Gradient} to verify the eigenvalue decay equivalence.
We also compute eigenvalues of $\widetilde{\bm X} := \beta \bm X \bm X^{\!\top}/d + \alpha \bm 1 \bm 1^{\!\top}$ for validation.
Here the parameters $\alpha$ depends on the covariate $\bm \Sigma_d$, which can be empirically estimated by the sample covariance $ \frac1n \sum_{i=1}^n (\bm x_i - \frac1n \sum_{j=1}^n \bm x_j) (\bm x_i - \frac1n \sum_{j=1}^n \bm x_j)^{\!\top}$.

Results on the polynomial kernel with order 3 and the Gaussian kernel are presented in Figure~\ref{mnistpoly} and~\ref{mnistgauss}, respectively. 
It can be observed that, the nonlinear kernel matrix $\bm K$ admits almost the same eigenvalue as $\widetilde{\bm X} := \beta \bm X \bm X^{\!\top}/d + \alpha \bm 1 \bm 1^{\!\top}$ with a constant shift $\gamma$, and accordingly exhibits the same eigenvalue decay with $\widetilde{\bm X}$ and $\bm X \bm X^{\!\top}/d$.

\begin{figure*}[!htb]
	\centering
	\subfigure[digit 1]{
		\includegraphics[width=0.18\textwidth]{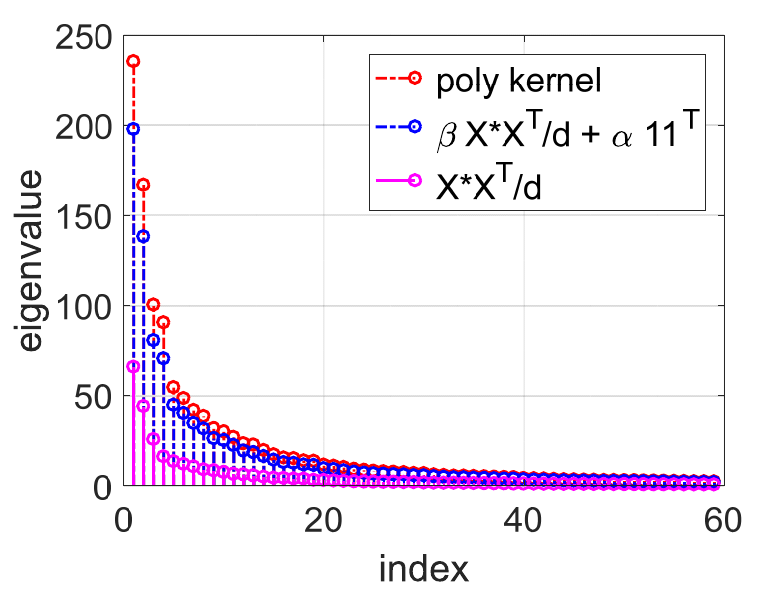}}
	\subfigure[digit 3]{
		\includegraphics[width=0.18\textwidth]{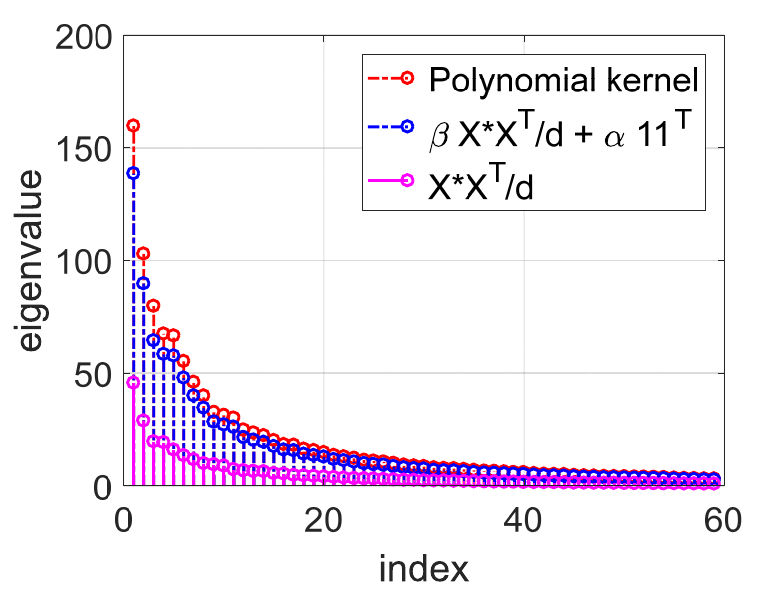}}
	\subfigure[digit 5]{
		\includegraphics[width=0.18\textwidth]{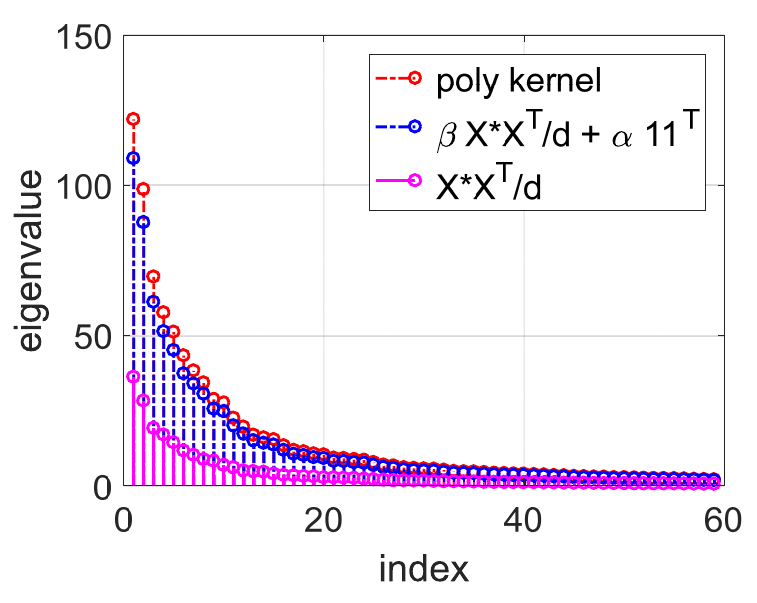}}
	\subfigure[digit 7]{
		\includegraphics[width=0.18\textwidth]{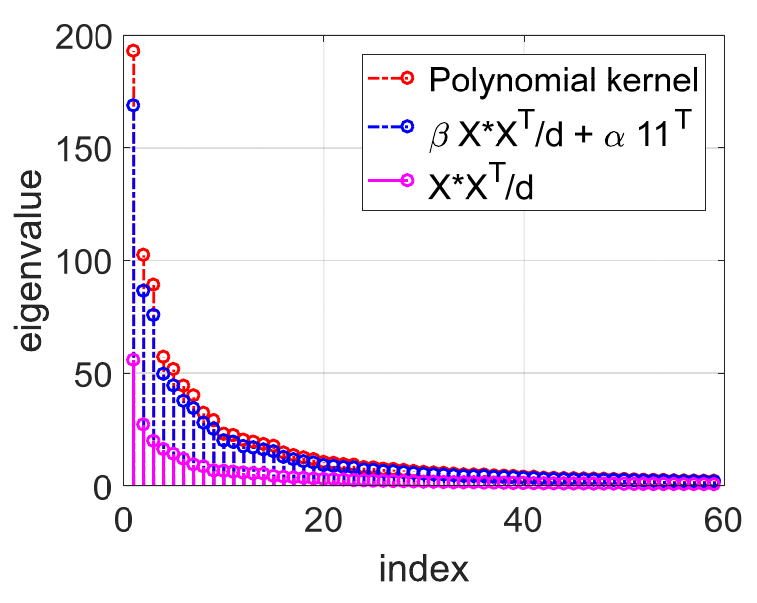}}
	\subfigure[digit 10]{
		\includegraphics[width=0.18\textwidth]{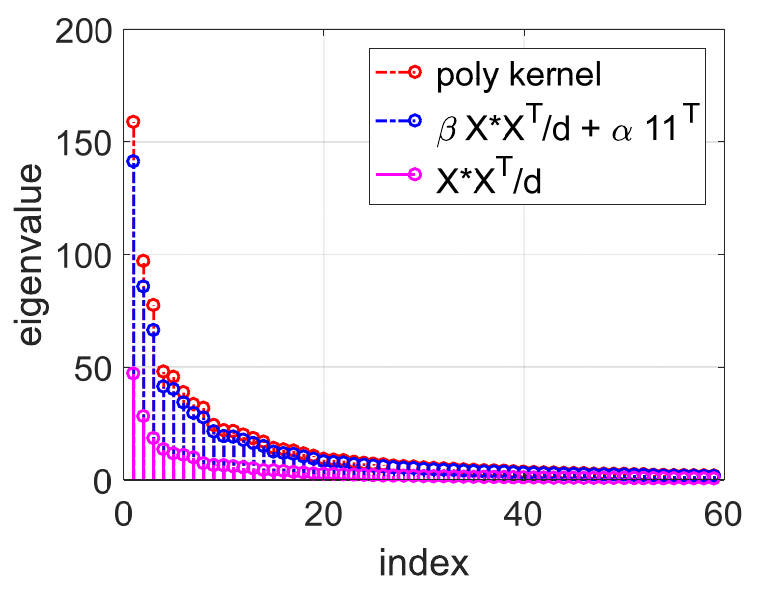}}
	\caption{Top 60 eigenvalues of Polynomial kernel with order 3 and its linearization on the MNIST dataset. Note that the largest eigenvalue $\lambda_1$ is not plotted for better display.}\label{mnistpoly}
\end{figure*}
\begin{figure*}[!htb]
	\centering
	\subfigure[digit 1]{
		\includegraphics[width=0.18\textwidth]{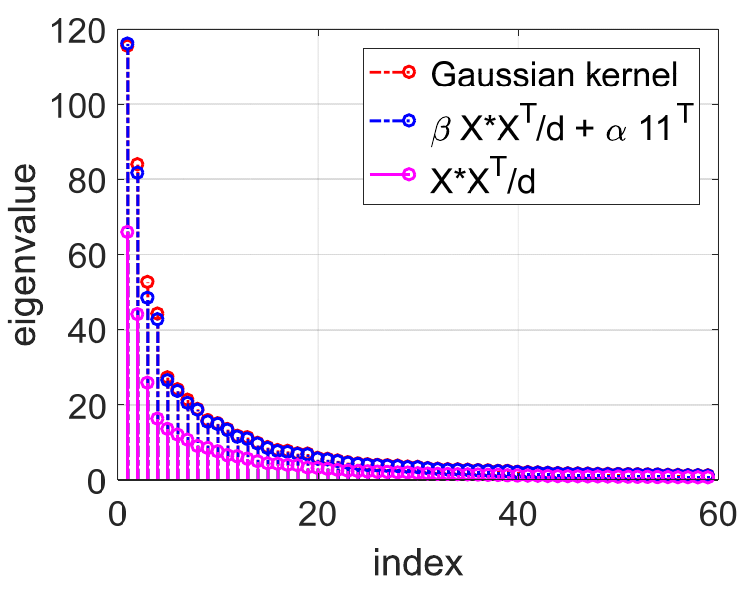}}
	\subfigure[digit 3]{
		\includegraphics[width=0.18\textwidth]{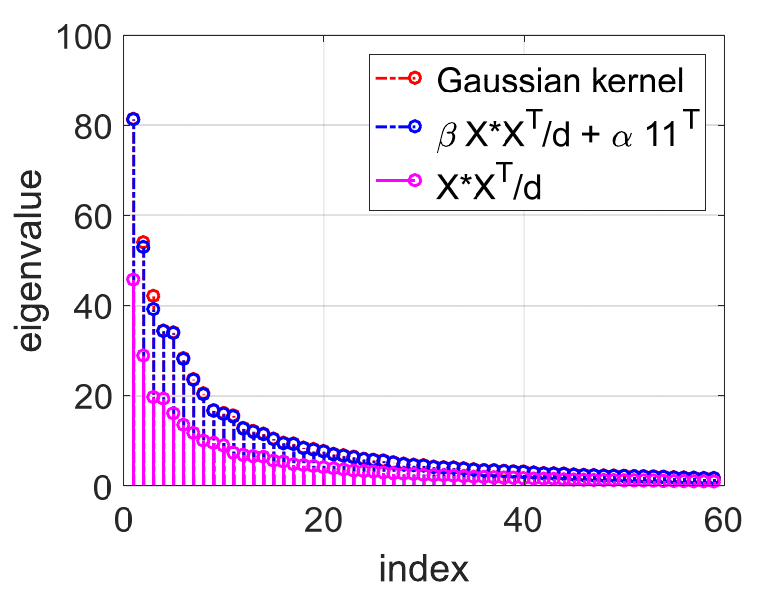}}
	\subfigure[digit 5]{
		\includegraphics[width=0.18\textwidth]{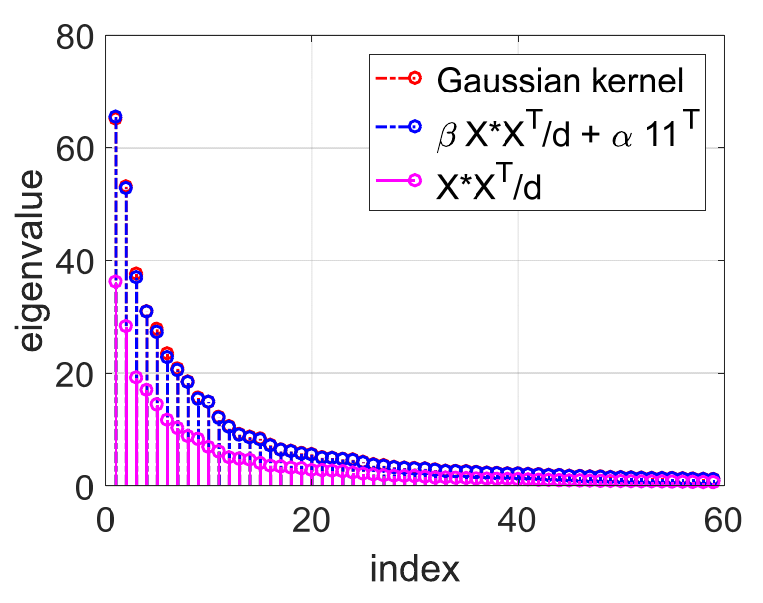}}
	\subfigure[digit 7]{
		\includegraphics[width=0.18\textwidth]{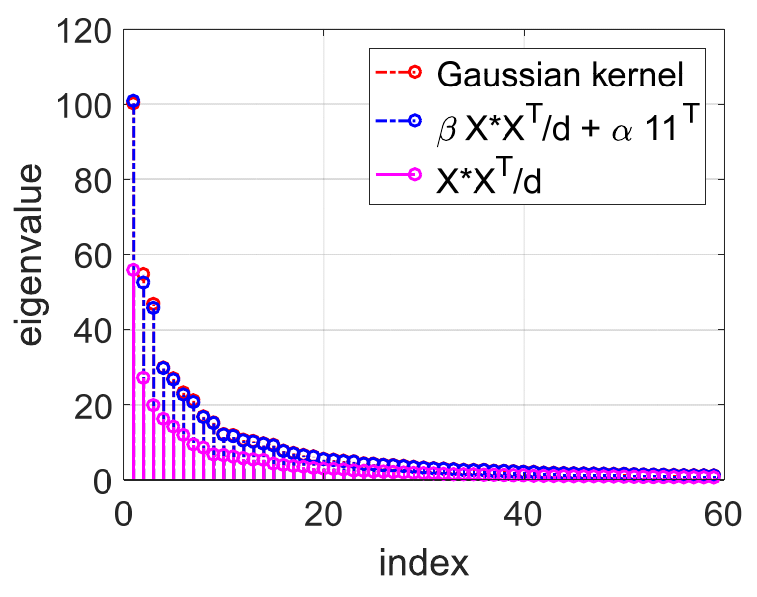}}
	\subfigure[digit 10]{
		\includegraphics[width=0.18\textwidth]{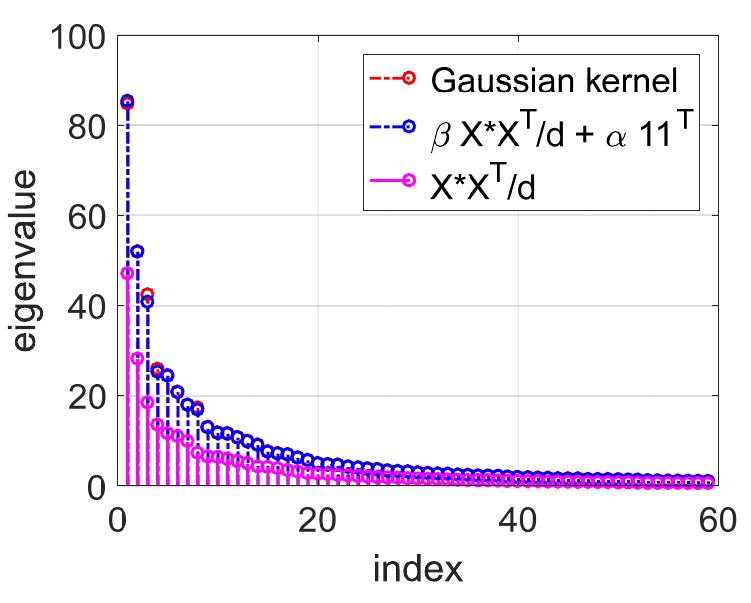}}
	\caption{Top 60 eigenvalues of Gaussian kernel and its linearization on the MNIST dataset. Note that the largest eigenvalue $\lambda_1$ is not plotted for better display.}\label{mnistgauss}
\end{figure*}

Besides, to study eigenvalue decay effected by the order in polynomial kernels, we present results of the order $p=5$ and $p=10$ in Figure~\ref{polymnistorder}.
Experimental results show that, there is some gap between the original kernel and its linearization in higher orders.
This is because, nonlinear kernel approximated by linear model here is based on Taylor expansion, which would incur in some residual errors as higher order in polynomial kernels brings in stronger non-linearity.

\begin{figure}[!htb]
	\centering
	\subfigure[\emph{order 5}]{
		\includegraphics[width=0.33\textwidth]{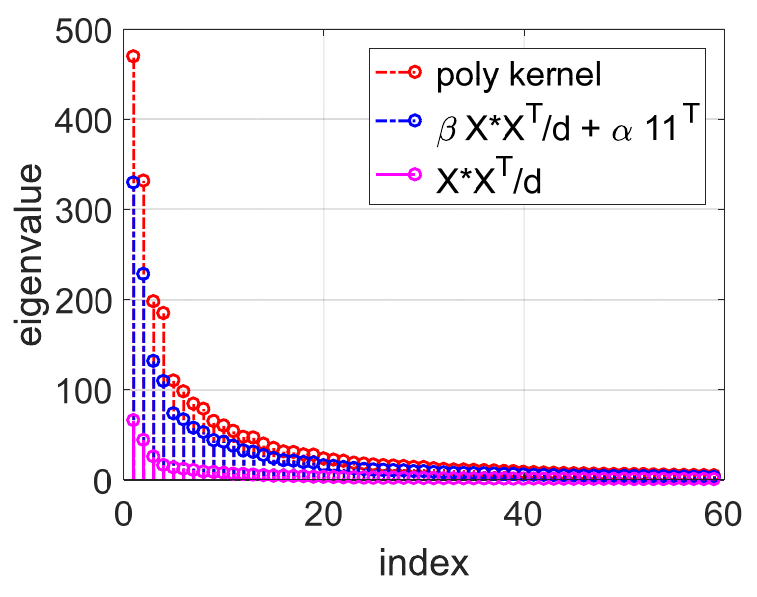}}
	\subfigure[\emph{order 10}]{
		\includegraphics[width=0.33\textwidth]{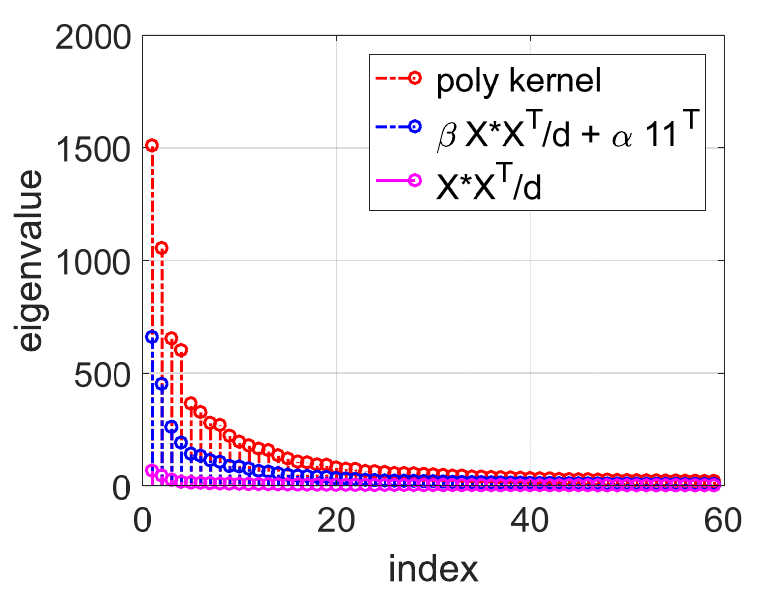}}
	\caption{Top 60 eigenvalues of polynomial kernel matrices and their linearizations on the MNIST dataset (digit 1). Note that the largest eigenvalue $\lambda_1$ is not plotted for better display.}\label{polymnistorder}
\end{figure}

\begin{figure*}[!htb]
	\centering
	\subfigure[$\vartheta=2/3$]{
		\includegraphics[width=0.23\textwidth]{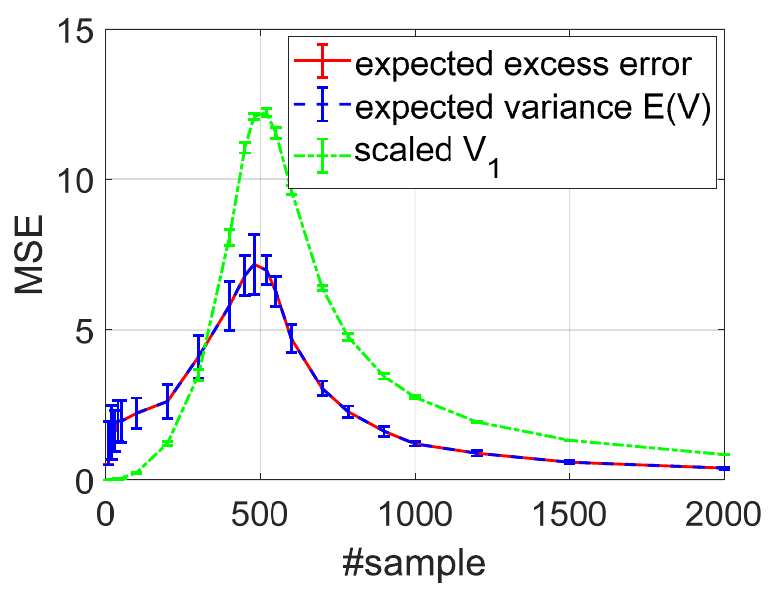}}
	\subfigure[$\vartheta=2/3$]{
		\includegraphics[width=0.23\textwidth]{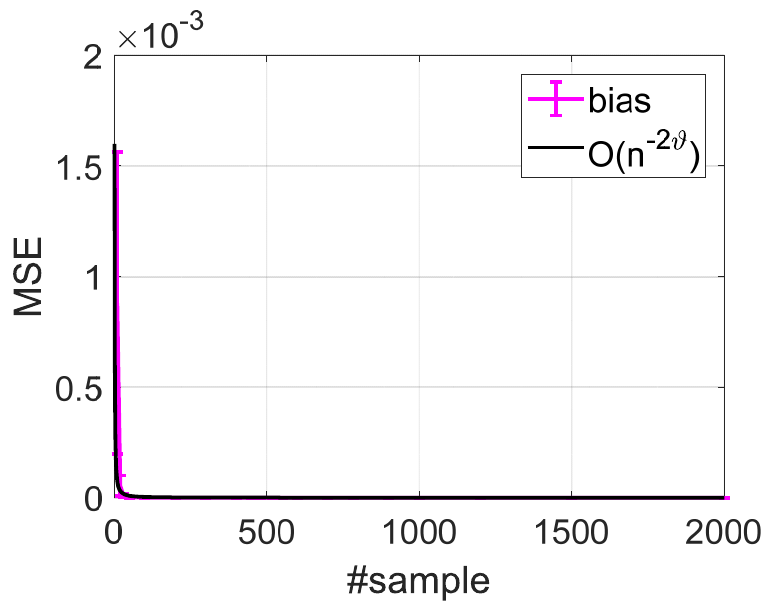}}
	\subfigure[$\vartheta=1/3$]{
		\includegraphics[width=0.23\textwidth]{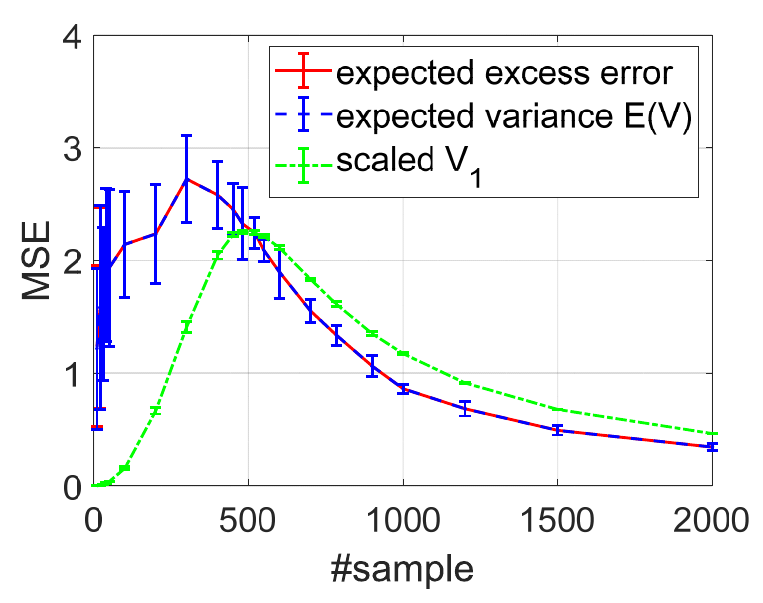}}
	\subfigure[$\vartheta=1/3$]{
		\includegraphics[width=0.23\textwidth]{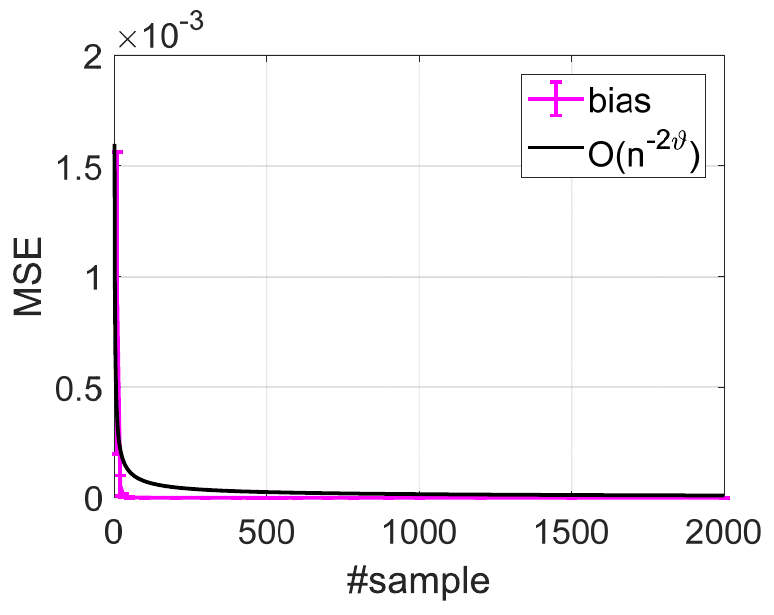}}
	\caption{Polynomial decay of $\widetilde{\bm X}$ in the polynomial kernel case: MSE of the expected excess risk, the variance in Eq.~\eqref{defvariance}, our derived ${\tt V_1}$, the bias in Eq.~\eqref{biaslemma1}, and our derived convergence rate $\mathcal{O}(n^{-2\vartheta r})$ with $r=1$ in Theorem~\ref{promain} under different $\vartheta$.}\label{fig-polykernelpolydec}
\end{figure*}

\begin{figure*}[!htb]
	\centering
	\subfigure[$\vartheta=2/3$]{
		\includegraphics[width=0.23\textwidth]{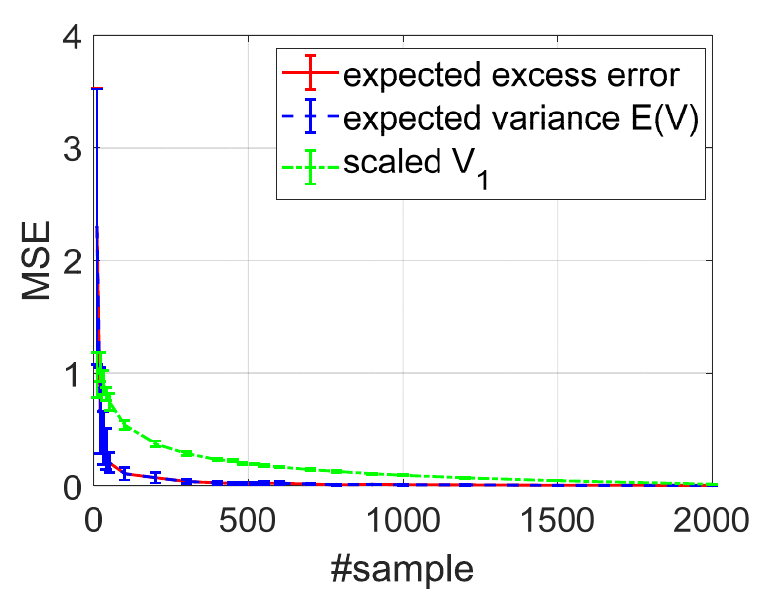}}
	\subfigure[$\vartheta=2/3$]{
		\includegraphics[width=0.23\textwidth]{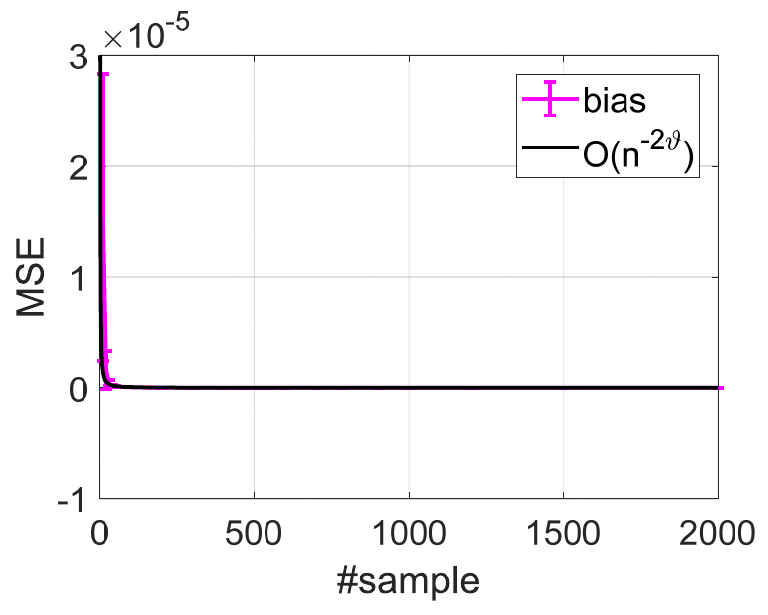}}
	\subfigure[$\vartheta=1/3$]{
		\includegraphics[width=0.23\textwidth]{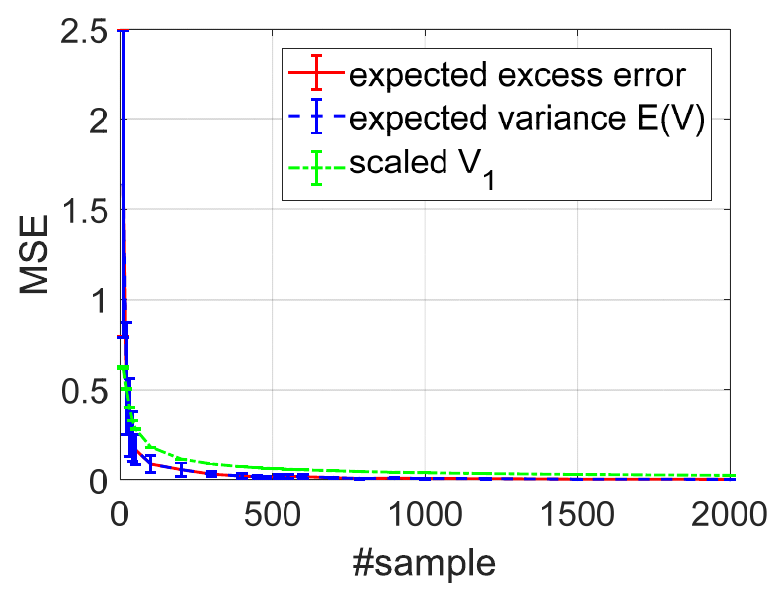}}
	\subfigure[$\vartheta=1/3$]{
		\includegraphics[width=0.23\textwidth]{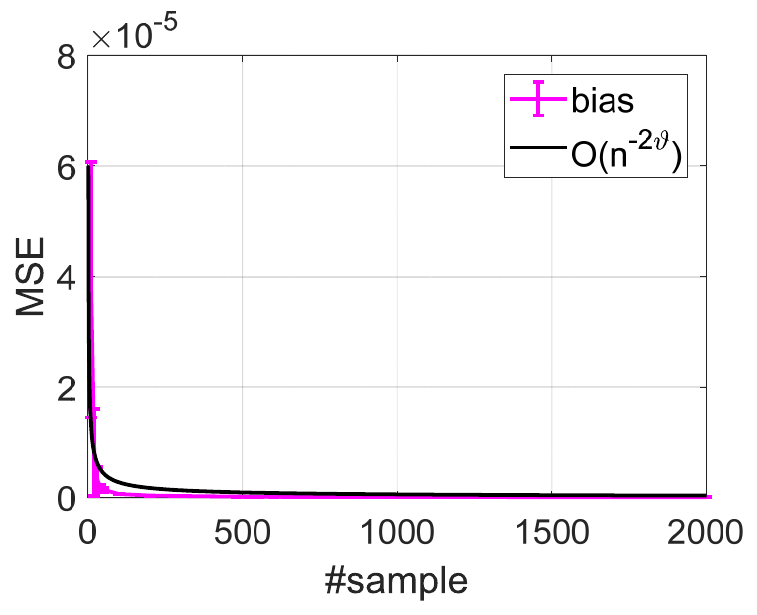}}
	\caption{Exponential decay of $\widetilde{\bm X}$ in the polynomial kernel case: MSE of the expected excess risk, the variance in Eq.~\eqref{defvariance}, our derived ${\tt V_1}$, the bias in Eq.~\eqref{biaslemma1}, and our derived convergence rate $\mathcal{O}(n^{-2\vartheta r})$ with $r=1$ in Theorem~\ref{promain} under different $\vartheta$.}\label{fig-polykernelexpdec}
\end{figure*}

\subsection{Results on the synthetic dataset}
\label{sec:appexpsyn}

Here we evaluate our model with the polynomial kernel on the synthetic dataset under the polynomial/exponential decay of $\bm \Sigma_d$.
The data generation process follows with our experiments part in the main text such that $\widetilde{\bm X}$ admits the polynomial/exponential decay.

Results on the \emph{polynomial decay} and the \emph{exponential decay} are shown in Figure~\ref{fig-polykernelpolydec} and Figure~\ref{fig-polykernelexpdec}, respectively.
We find that, the bias achieves the certain $\mathcal{O}(n^{-2 \vartheta r})$ convergence rate on both decays; while the variance shows different configurations on these two decays.
To be specific, the tend of ${\tt{V}}_1$ on the \emph{polynomial decay} is unimodal, and thus the risk curve is  bell-shaped.
However, in Figure~\ref{fig-polykernelexpdec}, ${\tt{V}}_1$ on the \emph{exponential decay} monotonically decreases with $n$ even if we set $\bar{c}$ to $10^{-5}$, $10^{-8}$ for a small regularization scheme.

Here we attempt to explain this phenomenon.
In our setting, $\gamma$ is set to zero. The condition in Eq.~\eqref{dericondition} can be reformulated as
\begin{equation*}
(2\vartheta - 1) \bar{c} \leq e^{-a}(1-\vartheta)\,.
\end{equation*}
Clearly, if we choose $0< \vartheta < 1/2$, the condition in Eq.~\eqref{dericondition} always holds.
Hence, ${\tt{V}}_1$ will monotonically decreases with $n$.
If $1/2< \vartheta < 1$, we examine our result with $a=1$ and $\vartheta=2/3$.
We conclude that the used $\bar{c}=0.01 < e^{-1}$, so the tend of ${\tt{V}}_1$ is monotonically decreasing with $n$.


\newcommand{\etalchar}[1]{$^{#1}$}
\providecommand{\bysame}{\leavevmode\hbox to3em{\hrulefill}\thinspace}
\providecommand{\MR}{\relax\ifhmode\unskip\space\fi MR }
\providecommand{\MRhref}[2]{%
  \href{http://www.ams.org/mathscinet-getitem?mr=#1}{#2}
}
\providecommand{\href}[2]{#2}
\begin{thebibliography}{SVGDB{\etalchar{+}}02}

\bibitem[AKM{\etalchar{+}}17]{avron2017random}
Haim Avron, Michael Kapralov, Cameron Musco, Christopher Musco, Ameya
  Velingker, and Amir Zandieh, \emph{Random \protect{F}ourier features for
  kernel ridge regression: Approximation bounds and statistical guarantees},
  the 34th International Conference on Machine Learning, 2017, pp.~253--262.

\bibitem[AKT19]{ali2019continuous}
Alnur Ali, J~Zico Kolter, and Ryan~J Tibshirani, \emph{A continuous-time view
  of early stopping for least squares regression}, International Conference on
  Artificial Intelligence and Statistics, 2019, pp.~1370--1378.

\bibitem[AP20]{adlam2020neural}
Ben Adlam and Jeffrey Pennington, \emph{The neural tangent kernel in high
  dimensions: Triple descent and a multi-scale theory of generalization},
  International Conference on Machine Learning, PMLR, 2020, pp.~74--84.

\bibitem[Bac13]{bach2013sharp}
Francis Bach, \emph{Sharp analysis of low-rank kernel matrix approximations},
  Conference on Learning Theory, 2013, pp.~185--209.

\bibitem[BCP20]{bordelon2020spectrum}
Blake Bordelon, Abdulkadir Canatar, and Cengiz Pehlevan, \emph{Spectrum
  dependent learning curves in kernel regression and wide neural networks},
  International Conference on Machine Learning, 2020, pp.~1--11.

\bibitem[BHMM19]{belkin2019reconciling}
Mikhail Belkin, Daniel Hsu, Siyuan Ma, and Soumik Mandal, \emph{Reconciling
  modern machine-learning practice and the classical bias--variance trade-off},
  the National Academy of Sciences \textbf{116} (2019), no.~32, 15849--15854.

\bibitem[BK10]{blanchard2010optimal}
Gilles Blanchard and Nicole Kr{\"a}mer, \emph{Optimal learning rates for kernel
  conjugate gradient regression}, Advances in Neural Information Processing
  Systems, 2010, pp.~226--234.

\bibitem[BLLT20]{bartlett2020benign}
Peter~L. Bartlett, Philip~M. Long, G{\'a}bor Lugosi, and Alexander Tsigler,
  \emph{Benign overfitting in linear regression}, the National Academy of
  Sciences (2020).

\bibitem[BRT19]{belkin2019does}
Mikhail Belkin, Alexander Rakhlin, and Alexandre~B Tsybakov, \emph{Does data
  interpolation contradict statistical optimality?}, International Conference
  on Artificial Intelligence and Statistics, 2019, pp.~1611--1619.

\bibitem[CBP20]{canatar2020statistical}
Abdulkadir Canatar, Blake Bordelon, and Cengiz Pehlevan, \emph{Statistical
  mechanics of generalization in kernel regression}, arXiv preprint
  arXiv:2006.13198 (2020).

\bibitem[CC20]{caron2020finite}
Emmanuel Caron and Stephane Chretien, \emph{A finite sample analysis of the
  double descent phenomenon for ridge function estimation}, arXiv preprint
  arXiv:2007.12882 (2020).

\bibitem[Cha08]{chang2008libsvm}
Chih-Chung Chang, \emph{\protect{LibSVM} data: Classification, regression, and
  multi-label}, http://www. csie. ntu. edu. tw/\~{}cjlin/libsvmtools/datasets/
  (2008).

\bibitem[CL20]{chinot2020benign}
Geoffrey Chinot and Matthieu Lerasle, \emph{Benign overfitting in the large
  deviation regime}, arXiv preprint arXiv:2003.05838 (2020).

\bibitem[CMBK20]{chen2020multiple}
Lin Chen, Yifei Min, Mikhail Belkin, and Amin Karbasi, \emph{Multiple descent:
  Design your own generalization curve}, arXiv preprint arXiv:2008.01036
  (2020).

\bibitem[CS02]{cucker2002mathematical}
Felipe Cucker and Steve Smale, \emph{On the mathematical foundations of
  learning}, Bulletin of the American mathematical society \textbf{39} (2002),
  no.~1, 1--49.

\bibitem[CZ07]{cucker2007learning}
Felipe Cucker and Dingxuan Zhou, \emph{Learning theory: an approximation theory
  viewpoint}, vol.~24, Cambridge University Press, 2007.

\bibitem[DLM19]{derezinski2019exact}
Micha{\l} Derezi{\'n}ski, Feynman Liang, and Michael~W Mahoney, \emph{Exact
  expressions for double descent and implicit regularization via surrogate
  random design}, arXiv preprint arXiv:1912.04533 (2019).

\bibitem[DMDVR09]{de2009elastic}
Christine De~Mol, Ernesto De~Vito, and Lorenzo Rosasco, \emph{Elastic-net
  regularization in learning theory}, Journal of Complexity \textbf{25} (2009),
  no.~2, 201--230.

\bibitem[DW18]{dobriban2018high}
Edgar Dobriban and Stefan Wager, \emph{High-dimensional asymptotics of
  prediction: Ridge regression and classification}, Annals of Statistics
  \textbf{46} (2018), no.~1, 247--279.

\bibitem[EK10]{el2010spectrum}
Noureddine El~Karoui, \emph{The spectrum of kernel random matrices}, Annals of
  Statistics \textbf{38} (2010), no.~1, 1--50.

\bibitem[EKZ{\etalchar{+}}20]{elkhalil2020risk}
Khalil Elkhalil, Abla Kammoun, Xiangliang Zhang, Mohamed-Slim Alouini, and
  Tareq Al-Naffouri, \emph{Risk convergence of centered kernel ridge regression
  with large dimensional data}, IEEE Transactions on Signal Processing
  \textbf{68} (2020), 1574--1588.

\bibitem[FS17]{fischer2017sobolev}
Simon Fischer and Ingo Steinwart, \emph{Sobolev norm learning rates for
  regularized least-squares algorithm}, arXiv preprint arXiv:1702.07254 (2017).

\bibitem[GLK{\etalchar{+}}20]{gerace2020generalisation}
Federica Gerace, Bruno Loureiro, Florent Krzakala, Marc M{\'e}zard, and Lenka
  Zdeborov{\'a}, \emph{Generalisation error in learning with random features
  and the hidden manifold model}, International Conference on Machine Learning,
  2020, pp.~3452--3462.

\bibitem[GMMM19]{ghorbani2019linearized}
Behrooz Ghorbani, Song Mei, Theodor Misiakiewicz, and Andrea Montanari,
  \emph{Linearized two-layers neural networks in high dimension}, Annals of
  Statistics (2019).

\bibitem[GSW17]{guo2017learning}
Zheng-Chu Guo, Lei Shi, and Qiang Wu, \emph{Learning theory of distributed
  regression with bias corrected regularization kernel network}, Journal of
  Machine Learning Research \textbf{18} (2017), no.~1, 4237--4261.

\bibitem[HMRT19]{hastie2019surprises}
Trevor Hastie, Andrea Montanari, Saharon Rosset, and Ryan~J. Tibshirani,
  \emph{Surprises in high-dimensional ridgeless least squares interpolation},
  arXiv preprint arXiv:1903.08560 (2019).

\bibitem[HRS16]{hardt2016train}
Moritz Hardt, Ben Recht, and Yoram Singer, \emph{Train faster, generalize
  better: Stability of stochastic gradient descent}, International Conference
  on Machine Learning, 2016, pp.~1225--1234.

\bibitem[J{\c{S}}S{\etalchar{+}}20a]{jacot2020implicit}
Arthur Jacot, Berfin {\c{S}}im{\c{s}}ek, Francesco Spadaro, Cl{\'e}ment
  Hongler, and Franck Gabriel, \emph{Implicit regularization of random feature
  models}, International Conference on Machine Learning, 2020, pp.~4631--4640.

\bibitem[J{\c{S}}S{\etalchar{+}}20b]{jacot2020kernel}
\bysame, \emph{Kernel alignment risk estimator: Risk prediction from training
  data}, Advances in Neural Information Processing Systems, 2020, pp.~1--9.

\bibitem[KLS20]{kobak2020optimal}
Dmitry Kobak, Jonathan Lomond, and Benoit Sanchez, \emph{The optimal ridge
  penalty for real-world high-dimensional data can be zero or negative due to
  the implicit ridge regularization}, Journal of Machine Learning Research
  \textbf{21} (2020), no.~169, 1--16.

\bibitem[LBBH98]{L1998Gradient}
Yann Lecun, Leon Bottou, Yoshua Bengio, and Patrick Haffner,
  \emph{Gradient-based learning applied to document recognition}, the IEEE
  \textbf{86} (1998), no.~11, 2278--2324.

\bibitem[LC18]{liao2018spectrum}
Zhenyu Liao and Romain Couillet, \emph{On the spectrum of random features maps
  of high dimensional data}, the International Conference on Machine Learning,
  2018, pp.~3063--3071.

\bibitem[LC19]{liao2019lssvm}
\bysame, \emph{A large dimensional analysis of least squares support vector
  machines}, IEEE Transactions on Signal Processing \textbf{67} (2019), no.~4,
  1065--1074.

\bibitem[LCM20]{liao2020random}
Zhenyu Liao, Romain Couillet, and Michael Mahoney, \emph{A random matrix
  analysis of random fourier features: beyond the gaussian kernel, a precise
  phase transition, and the corresponding double descent}, Neural Information
  Processing Systems, 2020.

\bibitem[LD20]{liu2020ridge}
Sifan Liu and Edgar Dobriban, \emph{Ridge regression: Structure,
  cross-validation, and sketching}, International Conference on Learning
  Representations, 2020.

\bibitem[LGZ17]{lin2017distributed}
Shao-Bo Lin, Xin Guo, and Ding-Xuan Zhou, \emph{Distributed learning with
  regularized least squares}, Journal of Machine Learning Research \textbf{18}
  (2017), no.~1, 3202--3232.

\bibitem[LHCS20]{liu2020survey}
Fanghui Liu, Xiaolin Huang, Yudong Chen, and Johan~A.K. Suykens, \emph{Random
  features for kernel approximation: A survey in algorithms, theory, and
  beyond}, arXiv preprint arXiv:2004.11154 (2020).

\bibitem[Li20]{li2020generalization}
Weilin Li, \emph{Generalization error of minimum weighted norm and kernel
  interpolation}, arXiv preprint arXiv:2008.03365 (2020).

\bibitem[LJB20]{lejeune2020implicit}
Daniel LeJeune, Hamid Javadi, and Richard Baraniuk, \emph{The implicit
  regularization of ordinary least squares ensembles}, International Conference
  on Artificial Intelligence and Statistics, 2020, pp.~3525--3535.

\bibitem[LR20]{liang2020just}
Tengyuan Liang and Alexander Rakhlin, \emph{Just interpolate: Kernel
  “ridgeless” regression can generalize}, Annals of Statistics \textbf{48}
  (2020), no.~3, 1329--1347.

\bibitem[LRZ19]{liang2020multiple}
Tengyuan Liang, Alexander Rakhlin, and Xiyu Zhai, \emph{On the multiple descent
  of minimum-norm interpolants and restricted lower isometry of kernels},
  Annual Conference on Learning Theory, 2019, pp.~1--32.

\bibitem[LSH{\etalchar{+}}21]{liu2020analysis}
Fanghui Liu, Lei Shi, Xiaolin Huang, Jie Yang, and Johan~A.K. Suykens,
  \emph{Analysis of regularized least squares in reproducing kernel kre\u{\i}n
  spaces}, Machine Learning (2021), 1--20.

\bibitem[LTOS19]{li2019towards}
Zhu Li, Jean-Francois Ton, Dino Oglic, and Dino Sejdinovic, \emph{Towards a
  unified analysis of random \protect{F}ourier features}, the 36th
  International Conference on Machine Learning, 2019, pp.~3905--3914.

\bibitem[MM19]{mei2019generalization}
Song Mei and Andrea Montanari, \emph{The generalization error of random
  features regression: Precise asymptotics and double descent curve}, arXiv
  preprint arXiv:1908.05355 (2019).

\bibitem[MP67]{marvcenko1967distribution}
Vladimir~A Mar{\v{c}}enko and Leonid~Andreevich Pastur, \emph{Distribution of
  eigenvalues for some sets of random matrices}, Mathematics of the
  USSR-Sbornik \textbf{1} (1967), no.~4, 457.

\bibitem[NVKM20]{nakkiran2020optimal}
Preetum Nakkiran, Prayaag Venkat, Sham Kakade, and Tengyu Ma, \emph{Optimal
  regularization can mitigate double descent}, arXiv preprint arXiv:2003.01897
  (2020).

\bibitem[PRDVR20]{pagliana2020interpolation}
Nicol{\`o} Pagliana, Alessandro Rudi, Ernesto De~Vito, and Lorenzo Rosasco,
  \emph{Interpolation and learning with scale dependent kernels}, arXiv
  preprint arXiv:2006.09984 (2020).

\bibitem[RCR13]{rudi2013sample}
Alessandro Rudi, Guillermo~D Canas, and Lorenzo Rosasco, \emph{On the sample
  complexity of subspace learning}, Advances in Neural Information Processing
  Systems, 2013, pp.~2067--2075.

\bibitem[RMR20]{richards2020asymptotics}
Dominic Richards, Jaouad Mourtada, and Lorenzo Rosasco, \emph{Asymptotics of
  ridge (less) regression under general source condition}, arXiv preprint
  arXiv:2006.06386 (2020).

\bibitem[RR17]{Rudi2017Generalization}
Alessandro Rudi and Lorenzo Rosasco, \emph{Generalization properties of
  learning with random features}, Advances in Neural Information Processing
  Systems, 2017, pp.~3215--3225.

\bibitem[RR19]{richards2019optimal}
Dominic Richards and Patrick Rebeschini, \emph{Optimal statistical rates for
  decentralised non-parametric regression with linear speed-up}, Advances in
  Neural Information Processing Systems, 2019, pp.~1216--1227.

\bibitem[RZ19]{rakhlin2019consistency}
Alexander Rakhlin and Xiyu Zhai, \emph{Consistency of interpolation with
  laplace kernels is a high-dimensional phenomenon}, Conference on Learning
  Theory, 2019, pp.~2595--2623.

\bibitem[SA08]{Steinwart2008SVM}
Ingo Steinwart and Christmann Andreas, \emph{Support vector machines}, Springer
  Science and Business Media, 2008.

\bibitem[SHK{\etalchar{+}}14]{srivastava2014dropout}
Nitish Srivastava, Geoffrey Hinton, Alex Krizhevsky, Ilya Sutskever, and Ruslan
  Salakhutdinov, \emph{Dropout: a simple way to prevent neural networks from
  overfitting}, Journal of Machine Learning Research \textbf{15} (2014), no.~1,
  1929--1958.

\bibitem[SS07]{steinwart2007fast}
Ingo Steinwart and Clint Scovel, \emph{Fast rates for support vector machines
  using \protect{G}aussian kernels}, Annals of Statistics \textbf{35} (2007),
  no.~2, 575--607.

\bibitem[SVGDB{\etalchar{+}}02]{suykens2002least}
Johan~A.K. Suykens, Tony Van~Gestel, Jos De~Brabanter, Bart De~Moor, and Joos
  Vandewalle, \emph{Least squares support vector machines}, World Scientific,
  2002.

\bibitem[SZ07]{smale2007learning}
Steve Smale and Ding-Xuan Zhou, \emph{Learning theory estimates via integral
  operators and their approximations}, Constructive Approximation \textbf{26}
  (2007), no.~2, 153--172.

\bibitem[Wen04]{wendland2004scattered}
Holger Wendland, \emph{Scattered data approximation}, vol.~17, Cambridge
  university press, 2004.

\bibitem[WX20]{wu2020optimal}
Denny Wu and Ji~Xu, \emph{On the optimal weighted $\ell_2$ regularization in
  overparameterized linear regression}, Advances in Neural Information
  Processing Systems, 2020, pp.~1--11.

\bibitem[WZ11]{Wang2011Optimal}
Cheng Wang and Ding-Xuan Zhou, \emph{Optimal learning rates for least squares
  regularized regression with unbounded sampling}, Journal of Complexity
  \textbf{27} (2011), no.~1, 55--67.

\bibitem[YSC{\etalchar{+}}16]{Yu2016Orthogonal}
Felix~Xinnan Yu, Ananda~Theertha Suresh, Krzysztof Choromanski, Daniel
  Holtmannrice, and Sanjiv Kumar, \emph{Orthogonal random features}, Advances
  in Neural Information Processing Systems, 2016, pp.~1975--1983.

\bibitem[YYY{\etalchar{+}}20]{yang2020rethinking}
Zitong Yang, Yaodong Yu, Chong You, Jacob Steinhardt, and Yi~Ma,
  \emph{Rethinking bias-variance trade-off for generalization of neural
  networks}, the International Conference on Machine Learning, 2020.

\bibitem[ZBH{\etalchar{+}}16]{zhang2016understanding}
Chiyuan Zhang, Samy Bengio, Moritz Hardt, Benjamin Recht, and Oriol Vinyals,
  \emph{Understanding deep learning requires rethinking generalization}, arXiv
  preprint arXiv:1611.03530 (2016).

\bibitem[ZDW13]{zhang2013divide}
Yuchen Zhang, John Duchi, and Martin Wainwright, \emph{Divide and conquer
  kernel ridge regression}, Conference on Learning Theory, 2013, pp.~592--617.

\end{thebibliography}
\end{document}